\documentclass[11pt,letterpaper,logo,twocolumn]{style}

\usepackage{amsmath,amssymb,amsfonts,bm,mathtools}
\usepackage{graphicx}
\usepackage{booktabs}
\usepackage{multirow}
\usepackage{makecell}
\usepackage{array}
\usepackage{pifont}
\usepackage{tikz}
\usetikzlibrary{arrows.meta,positioning,fit,calc,shapes.misc,backgrounds}
\usepackage{tabularray}
\UseTblrLibrary{booktabs}
\usepackage{algorithm}
\usepackage{algpseudocode}
\usepackage{textcomp}
\usepackage{url}
\usepackage{adjustbox}
\usepackage{threeparttable}

\graphicspath{{./}{image/}}
\definecolor{ASRGray}{RGB}{232,232,232}
\definecolor{ASRBlue}{RGB}{220,247,250}
\definecolor{ASRGreen}{RGB}{0,130,60}
\definecolor{TopOne}{RGB}{198,219,239}
\definecolor{TopTwo}{RGB}{222,235,247}
\definecolor{TopThree}{RGB}{239,243,255}
\definecolor{ASRHeader}{RGB}{228,228,228}
\definecolor{ASRGroup}{RGB}{238,238,238}
\definecolor{ASROurs}{RGB}{219,246,248}
\definecolor{VenueGray}{RGB}{80,80,80}

\newcommand{\venue}[1]{\textcolor{VenueGray}{\scriptsize #1}}

\providecommand{\best}[1]{\textbf{#1}}
\providecommand{\second}[1]{\underline{#1}}
\providecommand{\third}[1]{\textit{#1}}

\newcommand{\mb}[2]{#1~{\scriptsize(#2)}}
\newcommand{\gain}[1]{\,{\scriptsize\textcolor{ASRGreen}{#1}}}

\newtheorem{assumption}{Assumption}
\newtheorem{theorem}{Theorem}
\newtheorem{lemma}{Lemma}
\newtheorem{proposition}{Proposition}
\newtheorem{corollary}{Corollary}

\newcommand{\cS}{\mathcal{S}}
\newcommand{\cP}{\mathcal{P}}
\newcommand{\cG}{\mathcal{G}}

\newcommand{\R}{\mathbb{R}}
\newcommand{\E}{\mathbb{E}}

\newcommand{\law}{\mathsf{Law}}
\newcommand{\W}{\mathsf{W}}

\newcommand{\Lip}{\mathrm{Lip}}
\newcommand{\Tr}{\mathrm{Tr}}
\newcommand{\Cov}{\mathrm{Cov}}

\newcommand{\Fgt}{\mathfrak{F}}
\newcommand{\C}{\mathbb{C}}

\newcommand{\cF}{\mathcal{F}}
\newcommand{\cT}{\mathcal{T}}

\title{Attention-Spectrum Regularization for Replay-Free Continual Multimodal LLMs}
\runningtitle{Attention-Spectrum Regularization for Replay-Free Continual Multimodal LLMs}
\PublicDate{2026-06-20}

\author{%
  {\Authfont \small
    \textbf{Chuangxin Zhao}\equal\textsuperscript{1} \quad
    \textbf{Canran Xiao}\equal\textsuperscript{2} \quad
    \textbf{Siyuan Ma}\textsuperscript{3} \quad
    \textbf{Mengyao Lyu}\textsuperscript{3} \quad
    \textbf{Yanbiao Ma}\textsuperscript{4} \quad
    \textbf{Jun Xia}\advisor\textsuperscript{1} \quad
    \textbf{Guiguang Ding}\advisor\textsuperscript{5} \quad
    \textbf{Yang Liu}\advisor\textsuperscript{3}
  }\\
  {\Affilfont
    \textsuperscript{1} The Hong Kong University of Science and Technology (Guangzhou) \quad
    \textsuperscript{2} Sun Yat-sen University \quad
    \textsuperscript{3} Nanyang Technological University \quad
    \textsuperscript{4} Renmin University of China \quad
    \textsuperscript{5} Tsinghua University \\
    \equal\ Equal contribution \quad
    \advisor\ Corresponding authors: \texttt{junxia@hkust-gz.edu.cn}, \texttt{dinggg@tsinghua.edu.cn}, \texttt{yangliu@ntu.edu.sg}
  }
}


\begin{document}

\begin{abstract}
Multimodal large language models (MLLMs) are increasingly required to adapt to non-stationary streams of visual domains, question types, and user instructions, yet continual fine-tuning often causes severe forgetting of previously acquired multimodal skills. Existing continual vision--language methods mainly preserve outputs, replay data or pseudo-data, regularize embedding geometry, or allocate task-specific parameters, but they provide limited control over how the internal cross-modal attention patterns supporting old skills drift during adaptation. In this paper, we propose \emph{Attention-Spectrum Regularization} (ASR), a replay-free continual learning framework that preserves skill-conditioned structures of cross-modal attention. ASR treats cross-attention maps as two-dimensional signals, summarizes their scale and directional properties into compact spectral statistics, and stores only skill-wise prototype distributions instead of replaying past image-question pairs, generated pseudo-examples, or old-stage teacher snapshots. During later stages, a phase-invariant spectral regularizer constrains harmful drift of these prototypes while allowing instance-level attention to adapt to new tasks. We further provide theoretical analysis showing that skill-conditioned spectral drift controls forgetting under a spectral sufficiency assumption, and that Fourier power spectra are stable to spatial translations and bounded perturbations. Experiments on continual VQA and multimodal instruction-tuning benchmarks, including VQA v2, VQACL, CLT-VQA, CoIN, and UCIT, show that ASR consistently improves final performance and reduces forgetting over strong replay-, regularization-, and adapter-based baselines. These results suggest that preserving skill-level attention structure is an effective and lightweight mechanism for continual MLLMs. Code is available at \href{https://github.com/Creative-zcx/attention-spectrum-replay}{github.com/Creative-zcx/attention-spectrum-replay}.
\end{abstract}

\maketitle

\section{Introduction}
\label{sec:intro}

Multimodal large language models (MLLMs) are increasingly deployed in
dynamic environments where new domains, question types, and multimodal
instructions arrive over time~\cite{liu2023visual,xu2024lvlm,chu2025improving}. This continual adaptation is essential for
long-lived VQA systems, multimodal assistants, retrieval applications, and
real-world agents~\cite{constructvl23,COIN,chu2026ct,chi2026impromptu}. However, adapting MLLMs in
non-stationary streams remains difficult: standard fine-tuning often causes
catastrophic forgetting, degrading both previous task performance and
zero-shot generalization~\cite{VQA_CL,modx23,zscl23,chu2026medtri,wang2021semantic}.

Existing continual vision--language learning methods have made important
progress from several perspectives. Continual VQA studies design skill- or
question-type streams and mitigate forgetting through replay, pseudo-data,
symbolic prompts, or expert routing~\cite{VQA_CL,sgp23,vqacl23,quad25,clmoe25}.
Continual CLIP/VLM methods preserve image--text alignment by regularizing
features, similarities, or representation topology~\cite{incclip22,modx23,zscl23,ctp23,cclip25}.
Recent continual multimodal instruction-tuning methods further improve
parameter-efficient adaptation through LoRA, MoE adapters, routing, style
normalization, or data selection~\cite{Hidellava25,SEFE25,BranchLoRA25,DMoLE25,Adapt25}.
These methods substantially improve the stability--plasticity trade-off, but
they mainly operate on external outputs, stored or generated samples,
embedding geometry, or trainable parameter allocation.

A key aspect remains underexplored: the internal cross-modal attention
structure that supports multimodal reasoning. Question-conditioned visual
attention has long been central to VQA models~\cite{yang2016stacked,anderson2018bottomup},
and modern VL transformers explicitly use co-attentional or cross-modality
layers to align language and visual regions~\cite{lu2019vilbert,tan2019lxmert}.
For skills such as counting, localization, relation reasoning, and text
reading, successful prediction depends not only on the final answer or global
image--text similarity, but also on where the model looks, how fine-grained
its focus is, and which spatial structures are emphasized~\cite{selvaraju2017gradcam,xiao2026reversible,zhang2026multivariate,jiao2026large}.
Thus, a model may preserve part of its representation geometry while still
drifting in the visual evidence it uses for old skills; this concern is
consistent with VQA studies showing that answer accuracy can be inflated by
language priors rather than grounded visual understanding~\cite{goyal2017making}.
Although attention transfer and attention-prior studies suggest that attention
maps can serve as useful supervisory signals,
existing continual VL methods provide limited mechanisms for explicitly
diagnosing or preserving such skill-dependent patterns of visual focus.

\begin{figure}[ht]
	\centering
	\includegraphics[width=\linewidth]{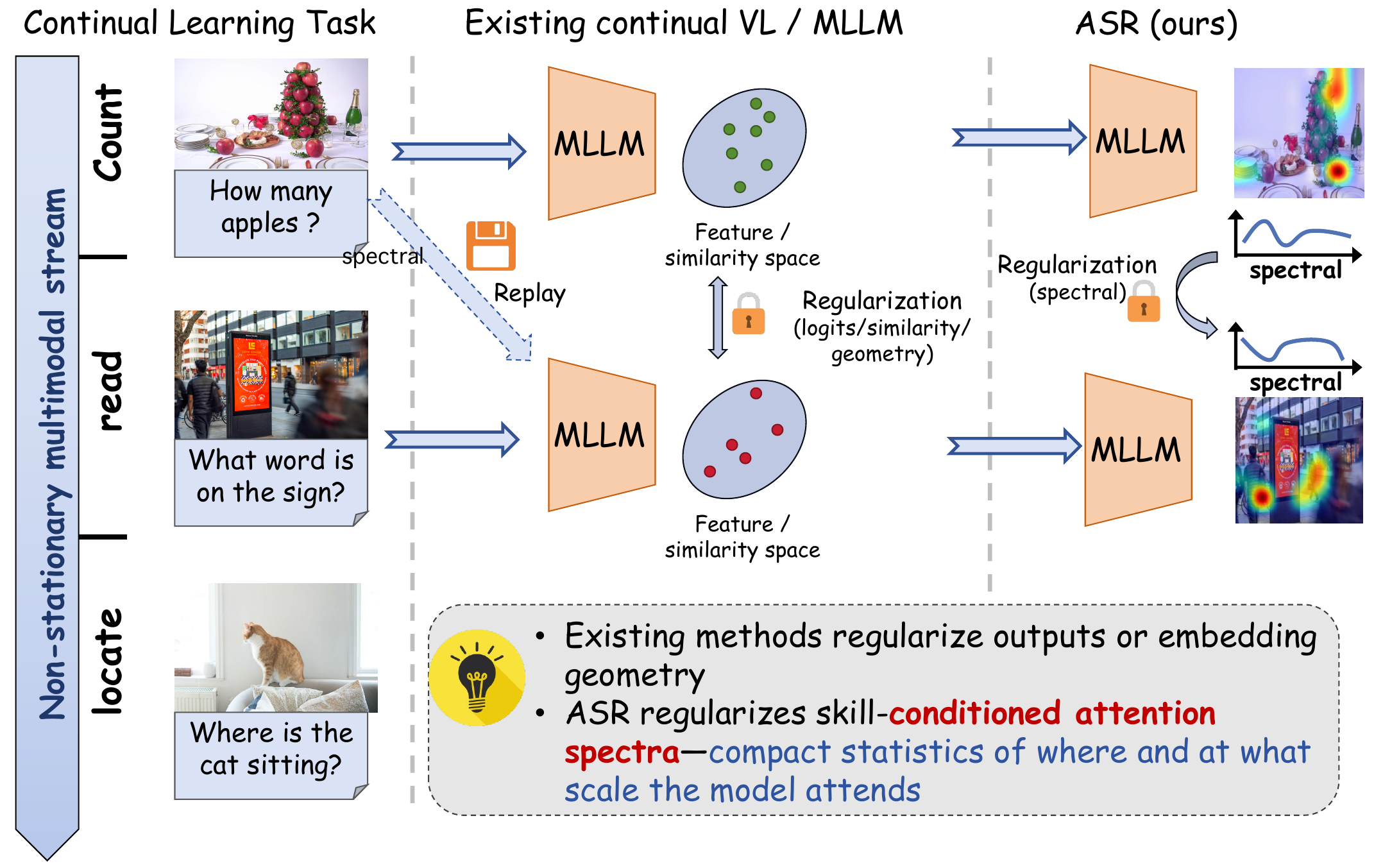}
	\caption{
Existing continual learning (CL) methods adapt a multimodal backbone mainly by replaying past data or pseudo-examples and regularizing logits or feature / similarity geometry.
ASR instead treats cross-attention as a 2D signal, and applies a phase-invariant spectral regularization to keep attention spectra stable across stages.}
\label{fig:teaser}
\end{figure}

This paper asks: \emph{can we mitigate forgetting in MLLMs by preserving the
skill-conditioned structure of cross-modal attention?} We hypothesize that
different multimodal skills induce relatively stable attention patterns, and
that harmful drift of these patterns is an important source of forgetting.
Based on this view, we propose \emph{Attention-Spectrum Regularization} (ASR), which
treats cross-modal attention as a structured two-dimensional signal and
preserves compact skill-level summaries of its scale and directional
properties. This allows ASR to regularize internal visual focus behavior
without storing past images, questions, or teacher models.

Our contributions are summarized as follows:
\begin{itemize}
	\item \textbf{Insight into multimodal forgetting.}
	We show that different reasoning skills exhibit distinct cross-attention
	spectral patterns, and that drift of these patterns is strongly associated
	with skill-wise forgetting.

    \item \textbf{Replay-free structural retention.}
    We propose \emph{Attention-Spectrum Regularization} (ASR), a plug-in continual
    learning framework that preserves compact skill-conditioned attention
    statistics without storing or replaying past samples, pseudo-samples, or old-stage teacher snapshots.

    \item \textbf{Consistent gains across benchmarks.}
    ASR improves final performance and reduces forgetting on continual VQA
    and multimodal instruction-tuning benchmarks, including VQA v2, VQACL,
    CLT-VQA, CoIN, and UCIT, with consistent gains across multiple MLLM
    backbones.
\end{itemize}

\section{Related Work}
\label{sec:related}

\subsection{Continual vision--language and VQA learning.}
Early continual learning methods mainly mitigated forgetting by constraining parameters or predictions, as in EWC~\cite{EWC} and LwF~\cite{LwF}, or by replaying a small memory of past samples~\cite{ER,DER}. With the rise of vision--language models, this problem shifted from closed-set recognition to preserving cross-modal alignment and compositional reasoning. Representative studies explored continual CLIP/VLM pretraining and retrieval~\cite{incclip22,modx23,ctp23,ticclip24,cclip25}, data-free structured concept learning~\cite{constructvl23}, and continual VQA streams where models must retain question-specific reasoning skills~\cite{VQA_CL,vqacl23}. VQACL introduced a skill--concept benchmark with novel compositions~\cite{vqacl23}; TRIPLET decoupled unimodal prompt learning before cross-modal interaction~\cite{triplet23}; SGP replayed symbolic scene-graph prompts~\cite{sgp23}; QUAD used questions-only replay with attention consistency distillation~\cite{quad25}; and recent SoTA methods such as CL-MoE and BCP-MFA address continual VQA through expert routing or drift correction~\cite{clmoe25,BCPMFA25}. These methods substantially improve stability, but they often depend on stored questions, symbolic surrogates, generated pseudo-data, teacher signals, or additional routing capacity. In contrast, ASR keeps neither past images nor past questions, and stores only compact skill-conditioned spectral prototypes of cross-modal attention.

\subsection{Continual multimodal instruction tuning of MLLMs.}
Continual instruction tuning extends the above setting from VQA-style tasks to heterogeneous instruction streams. CoIN established a benchmark for sequential multimodal instruction tuning and showed that MLLMs suffer severe forgetting across diverse task families~\cite{COIN}. Subsequent work improved parameter-efficient adaptation through LoRA-style subspaces and expert modules~\cite{lora,Olora}. HiDe-LLaVA separates task-specific and task-general components according to representation similarity~\cite{Hidellava25}; SEFE distinguishes superficial answer-style forgetting from essential ability forgetting~\cite{SEFE25}; BranchLoRA introduces asymmetric LoRA branches and routing~\cite{BranchLoRA25}; D-MoLE dynamically composes curriculum LoRA experts~\cite{DMoLE25}; and Adapt-$\infty$ selects dynamic data to scale continual multimodal instruction tuning~\cite{Adapt25}. These methods mainly operate on adapter allocation, routing, answer style, or data selection. ASR is complementary: rather than allocating more task-specific capacity, it regularizes the spatial scale and directional structure of cross-modal attention, providing an internal behavioral constraint that can be combined with parameter-efficient tuning.

\subsection{Attention, representation, and spectral retention.}
Knowledge retention has also been studied through intermediate representations and attention maps. Feature or representation distillation has been used in continual learning and dense prediction~\cite{LwF,Plop}, while attention transfer shows that attention maps can serve as compact supervisory signals~\cite{attentiontransfer17,tian2022vibus,tian2022vibus}. In multimodal continual learning, MAFED observes that different modalities may evolve at different rates and uses modality-aware feature distillation~\cite{mafed24}; GaB studies data-free continual VQA through generated pseudo-rehearsal~\cite{gab25}; and QUAD shows that attention drift is an important source of forgetting in VQA~\cite{quad25}. However, directly matching high-dimensional features or raw attention maps can be teacher-dependent, memory-inefficient, and overly restrictive when new tasks require different visual grounding. ASR instead treats cross-modal attention as a two-dimensional signal and preserves its phase-insensitive Fourier statistics. By modeling radial energy, angular anisotropy, and dominant spectral modes with skill-conditioned Gaussian prototypes, ASR retains stable reasoning-specific visual focus patterns while allowing instance-level attention to adapt to new data.

\begin{figure*}[t]
  \centering
  \includegraphics[width=0.98\linewidth]{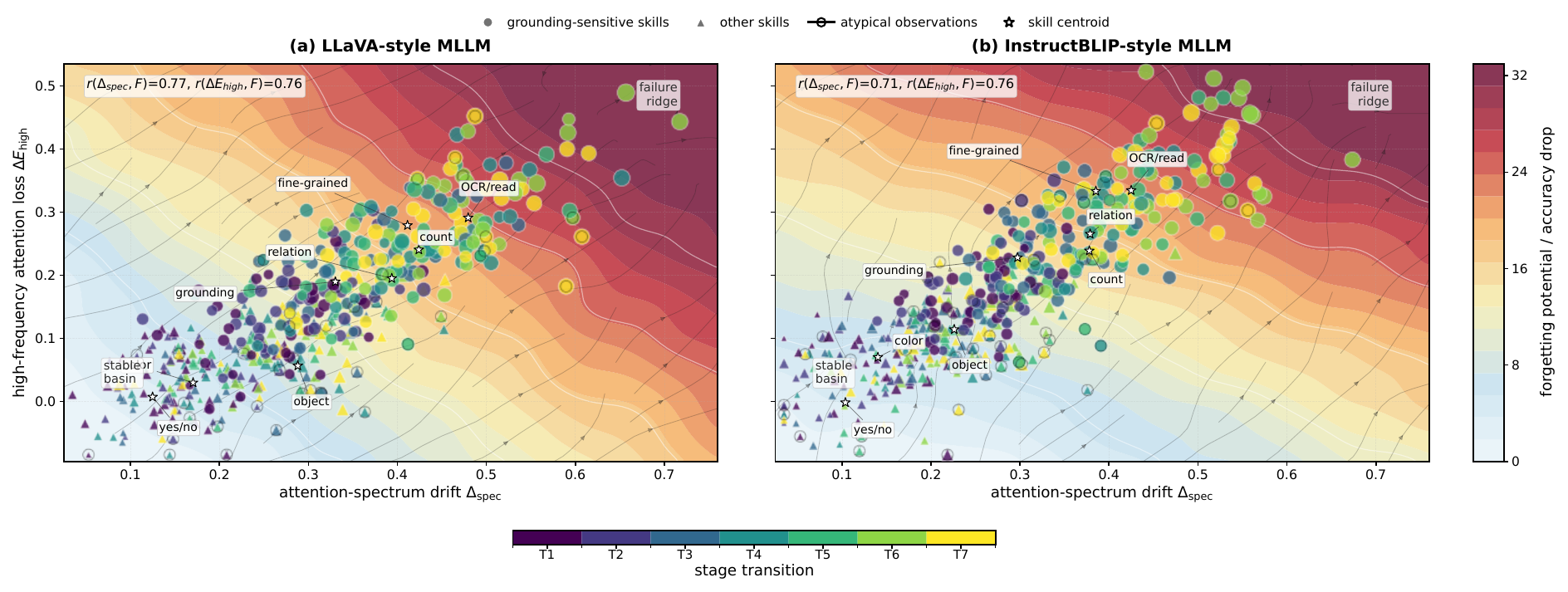}
  \vspace{-2mm}
  \caption{
  Exploratory potential-field diagnosis of vanilla sequential fine-tuning on
  two different MLLM backbones.
  Each marker corresponds to one skill--transition--seed observation. The
  horizontal axis measures attention-spectrum drift, and the vertical axis
  measures high-frequency attention loss. The background contours show the
  smoothed forgetting potential, and the streamlines indicate the direction of
  increasing forgetting. Marker area denotes the increase of attention entropy.
  Rings mark atypical observations caused by noisy stage transitions or
  answer-distribution shifts. T1--T7 denote consecutive stage transitions.
  Both models exhibit a similar stable basin at low spectral drift and low
  high-frequency loss, as well as a failure ridge where large spectrum drift
  coincides with weakened fine-scale attention.
  }
  \label{fig:prelim_potential}
  \vspace{-2mm}
\end{figure*}

\section{Preliminary Study}
\label{sec:prelim}

Before introducing our method, we first perform an exploratory study to examine
\emph{how multimodal forgetting manifests inside cross-modal attention}.

\paragraph{Diagnostic setup}
We consider two representative multimodal backbones: a LLaVA-style MLLM and an
InstructBLIP-style MLLM. Each model is continually fine-tuned on the same
non-stationary VQA stream. After each stage, we evaluate old validation subsets
grouped by reasoning skills, such as \texttt{count}, \texttt{color},
\texttt{spatial}, \texttt{relation}, \texttt{OCR/read}, and
\texttt{fine-grained}. For model $m$, skill $s$, and transition
$t{-}1 \rightarrow t$, the skill-wise forgetting is defined as
\begin{equation}
  F_{m,s}^{(t)}
  =
  \max_{\tau < t}
  \operatorname{Acc}
  \big(
    \boldsymbol{\theta}_{m}^{(\tau)};\mathcal{V}_s
  \big)
  -
  \operatorname{Acc}
  \big(
    \boldsymbol{\theta}_{m}^{(t)};\mathcal{V}_s
  \big),
\end{equation}
where $\mathcal{V}_s$ denotes the held-out validation samples associated with
skill $s$.

For the same validation samples, we extract cross-modal attention maps from the
saved checkpoints and summarize their frequency content. Let
$\bar{\mathbf{r}}_{m,s}^{(t)}$ and $\bar{\mathbf{d}}_{m,s}^{(t)}$ denote the
mean radial and angular spectra of skill $s$ at checkpoint
$\boldsymbol{\theta}_{m}^{(t)}$. We measure the spectrum drift across two
adjacent checkpoints by
\begin{equation}
  \Delta_{\mathrm{spec},m,s}^{(t)}
  =
  \operatorname{JSD}
  \big(
    \bar{\mathbf{r}}_{m,s}^{(t-1)}
    \Vert
    \bar{\mathbf{r}}_{m,s}^{(t)}
  \big)
  +
  \eta
  \operatorname{JSD}
  \big(
    \bar{\mathbf{d}}_{m,s}^{(t-1)}
    \Vert
    \bar{\mathbf{d}}_{m,s}^{(t)}
  \big),
\end{equation}
where $\operatorname{JSD}(\cdot\Vert\cdot)$ is the Jensen--Shannon divergence.
We further measure the loss of high-frequency attention energy:
\begin{equation}
  \Delta E_{\mathrm{high},m,s}^{(t)}
  =
  \sum_{\rho_k > \rho_0}
  \left[
    \bar{r}_{m,s,k}^{(t-1)}
    -
    \bar{r}_{m,s,k}^{(t)}
  \right],
\end{equation}
where $\rho_0$ separates low/mid-frequency and high-frequency bands. A positive
$\Delta E_{\mathrm{high}}$ indicates that fine-scale attention components
become weaker after continual tuning.

\paragraph{Forgetting potential field}
To visualize the failure landscape, we construct a smooth diagnostic potential
field for each model:
\begin{equation}
  \mathcal{U}_{m}(x,y)
  \approx
  \mathbb{E}
  \left[
    F_{m,s}^{(t)}
    \mid
    x=\Delta_{\mathrm{spec},m,s}^{(t)},
    \,
    y=\Delta E_{\mathrm{high},m,s}^{(t)}
  \right].
\end{equation}
It is estimated
from the observed skill--transition--seed measurements by smooth interpolation.
The contours indicate the severity of expected forgetting, while the streamlines
show the direction of increasing forgetting potential.

\paragraph{Observations}
Fig.~\ref{fig:prelim_potential} shows that, despite architectural differences, the two backbones exhibit similar failure geometry. Skills with small spectral
drift and limited high-frequency loss remain in the stable basin and suffer only minor forgetting. 
In contrast, grounding-sensitive skills such as
\texttt{count}, \texttt{relation}, \texttt{OCR/read}, \texttt{grounding}, and
\texttt{fine-grained} tend to move toward the high-potential region. These skills do not merely lose output accuracy; their cross-modal attention also undergoes a structured change in frequency space.

Some points show moderate forgetting without severe spectral drift, which may be caused by answer-space shift or language-side bias. Conversely, a few samples show large attention drift but limited accuracy degradation, suggesting that not every attention change is harmful. Nevertheless, the dominant trend is consistent across both models: severe forgetting is most likely when the attention spectrum drifts and fine-scale visual evidence is weakened simultaneously.

This preliminary study suggests that multimodal forgetting has a measurable attention-level signature. More importantly, this signature is skill-dependent:
different reasoning skills occupy different regions of the potential field and
degrade along different trajectories. These observations motivate our method
design. Instead of storing old data or only constraining output logits, we seek
to preserve compact, skill-conditioned summaries of cross-modal attention
structure. The next section instantiates this idea as
\emph{Attention-Spectrum Regularization}.

\begin{figure*}[htb]
	\centering
	\includegraphics[width=1\linewidth]{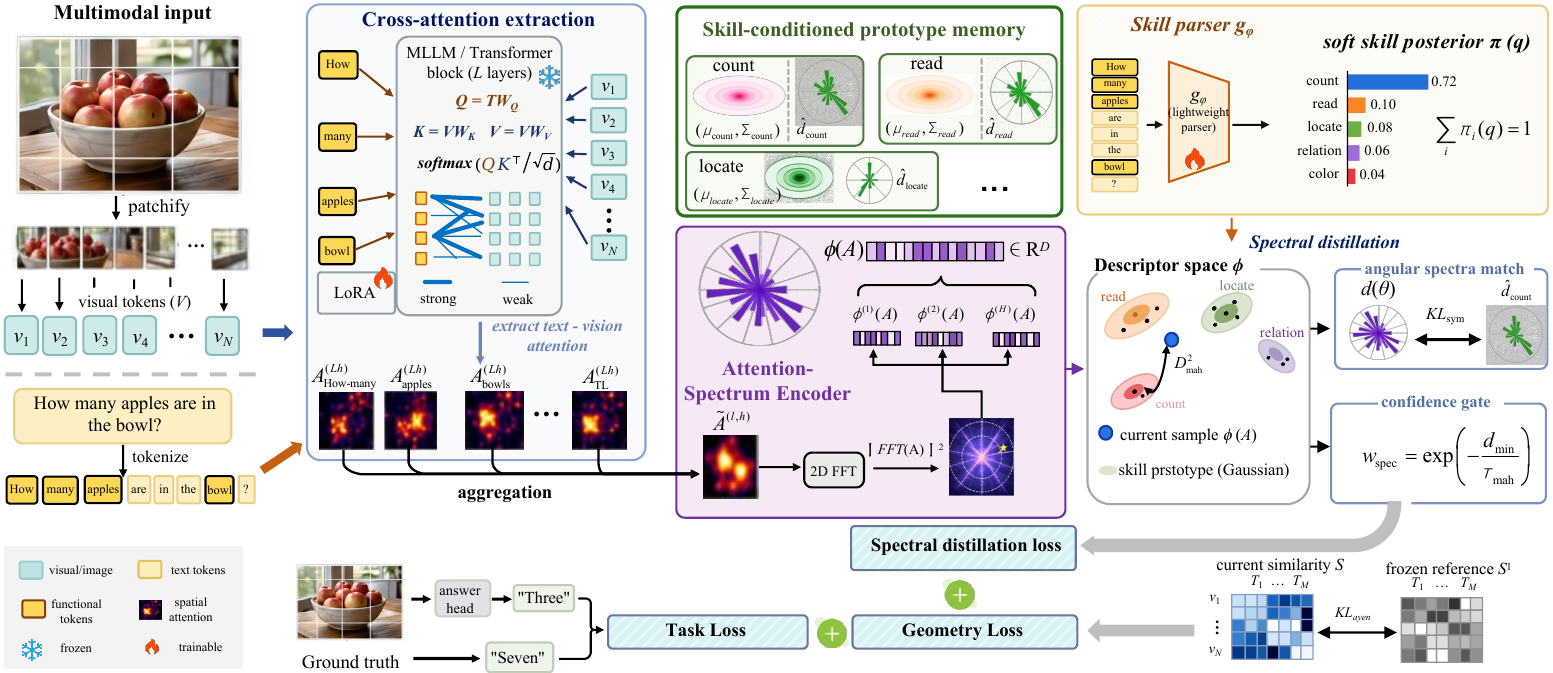}
	\caption{Overview of Attention-Spectrum Regularization (ASR). We extract cross-modal attention maps from a multimodal LLM, encode them into low-dimensional spectral descriptors, aggregate them into skill-conditioned prototypes stored in a compact memory, and apply a spectral distillation loss during training on later stages. ASR only keeps spectrum statistics for each reasoning skill.}
	\label{fig:asr_pipline}
\end{figure*}

\section{Method}
\label{sec:method}

We propose \emph{Attention-Spectrum Regularization} (ASR), a replay-free continual learning framework for MLLMs. 
ASR treats these cross-attention maps as two-dimensional signals and operates on them through four components: 
(1) a \emph{cross-attention extractor} that selects layers/heads and aggregates functional tokens;
(2) a \emph{spectral encoder} $\boldsymbol{\phi}(\cdot)$ that maps attention maps to low-dimensional spectral statistics summarizing \emph{where} and \emph{at what scale} the model attends;
(3) a \emph{skill parser} $g_{\boldsymbol{\psi}}$ producing a distribution over reasoning skills $\boldsymbol{\pi}(q)$;
(4) a \emph{prototype memory} $\mathcal{M}$ storing skill-conditioned spectral prototypes.
The overall pipeline is illustrated in Fig.~\ref{fig:asr_pipline}.

\subsection{Preliminaries}
\label{subsec:problem}

At stage $t \in \{1,\dots,T\}$, the learner observes a dataset
$
  \mathcal{D}^{(t)} = \big\{(I_i^{(t)}, q_i^{(t)}, y_i^{(t)}) \big\}_{i=1}^{N_t},
$
where $I_i^{(t)}$ is an image, $q_i^{(t)}$ a question or instruction, and $y_i^{(t)}$ a task-specific target (answer tokens, relevance labels, etc.). Once stage $t$ is finished, $\mathcal{D}^{(t)}$ is no longer accessible.

A multimodal LLM $f_{\boldsymbol{\theta}}$ processes $(I,q)$ through a vision encoder, a language model, and architecture-specific multimodal fusion modules~\cite{zhang2026adaptive}. Since different MLLM families expose visual grounding signals in different forms, we do not assume that every backbone has explicit encoder--decoder cross-attention layers. Instead, ASR operates on a unified \emph{cross-modal attention map}, defined as a text-to-vision attention map. For models with explicit cross-attention, this map is the native cross-attention matrix. For decoder-only MLLMs such as LLaVA-style, Qwen-VL-style, and InternVL-style models, it is extracted from the text-to-vision submatrix of decoder self-attention after visual tokens have been inserted into the language-model context. Detailed backbone-specific extraction rules are provided in Appendix~\ref{app:attention_extraction}.

Let $\mathcal{L}_\mathrm{attn}$ be the set of candidate multimodal attention layers or decoder blocks, and let $\mathcal{H}_\mathrm{attn}$ be the set of attention heads. For layer or block $l \in \mathcal{L}_\mathrm{attn}$ and head $h \in \mathcal{H}_\mathrm{attn}$, ASR obtains
$
  A^{(l,h)} \in \mathbb{R}^{T_q \times H \times W},
$
where $T_q$ is the number of question/instruction tokens and $H \times W$ is the reconstructed visual grid. The slice $A^{(l,h)}_{\tau} \in \mathbb{R}^{H \times W}$ encodes how token $\tau$ attends to visual evidence.

We assume a fixed set of reasoning skills $\mathcal{S}$ (e.g., $\{\texttt{count}, \texttt{color}, \texttt{read}, \dots\}$). A lightweight parser $g_{\boldsymbol{\psi}}$ maps $q$ to a probability vector~\cite{chen2022cerberus}:
\begin{equation}
  \boldsymbol{\pi}(q) = \big(\pi_s(q)\big)_{s \in \mathcal{S}} \in \Delta^{|\mathcal{S}|},
  \qquad
  \sum_{s \in \mathcal{S}} \pi_s(q) = 1,
\end{equation}
where $\pi_s(q)$ is the posterior that skill $s$ is required by $q$.

\subsection{Attention-Spectrum Encoding}
\label{subsec:spectrum}

Cross-modal attention maps capture where the model looks when solving a multimodal problem. Empirically, different skills induce characteristic spatial scales and directional patterns. ASR encodes these properties via frequency-domain statistics.

\vspace{1mm}
\noindent\textbf{Token aggregation.}
Given $A^{(l,h)} \in \mathbb{R}^{T_q \times H \times W}$, we select question-side functional tokens $U(q)\subseteq \{1,\dots,T_q\}$, such as wh-words, verbs, relation words, OCR-related tokens, and head nouns. Answer tokens are excluded when teacher forcing is used. We average the corresponding text-to-vision maps:
\begin{equation}
  \tilde{A}^{(l,h)}(x,y)
  = \frac{1}{|U(q)|} \sum_{\tau \in U(q)} A^{(l,h)}_{\tau}(x,y),
\end{equation}
then normalize the map to sum to one.
For brevity we write $A(x,y)$ for a normalized map.

\vspace{1mm}
\noindent\textbf{2D Fourier spectrum.}
We compute the 2D discrete Fourier transform
\begin{equation}
  \hat{A}(u,v)
  = \sum_{x=0}^{H-1} \sum_{y=0}^{W-1}
    A(x,y) \exp\!\bigg(- 2\pi i \Big(\frac{ux}{H} + \frac{vy}{W}\Big) \bigg),
\end{equation}
for $u\in\{0,\dots,H{-}1\}$ and $v\in\{0,\dots,W{-}1\}$, and the power spectrum
$
  P(u,v) = |\hat{A}(u,v)|^2.
$
We map $(u,v)$ to polar coordinates
\begin{align}
  \rho(u,v) &= \sqrt{\Big(\frac{u}{H}\Big)^2 + \Big(\frac{v}{W}\Big)^2}, \\
  \theta(u,v) &= \operatorname{atan2}\Big(\frac{v}{W}, \frac{u}{H}\Big),
\end{align}
and renormalize radii to $[0,1]$ by dividing by $\rho_{\max} = \max_{u,v} \rho(u,v)$.

\vspace{1mm}
\noindent\textbf{Radial and angular spectra.}
We discretize radii into $K$ bins $\{\mathcal{R}_k\}$ on $[0,1]$ and angles into $M_\theta$ bins $\{\Theta_m\}$ on $[-\pi,\pi)$.
For each head, we compute a radial spectrum $\mathbf{r}\in\mathbb{R}^K$ and an angular spectrum $\mathbf{d}\in\mathbb{R}^{M_\theta}$ by averaging the power $P(u,v)$ over frequencies whose radius or angle falls into the corresponding bin:
\begin{equation}
  r_k \propto \sum_{\rho(u,v)\in\mathcal{R}_k} P(u,v),
  \qquad
  d_m \propto \sum_{\theta(u,v)\in\Theta_m} P(u,v),
\end{equation}
followed by normalization.
We further define an anisotropy vector $\mathbf{a}\in\mathbb{R}^{M_\theta}$ as
\begin{equation}
  a_m = \frac{d_m}{\bar d + \varepsilon},
  \qquad
  \bar d = \frac{1}{M_\theta}\sum_m d_m,
\end{equation}
so that $a_m>1$ indicates directions with above-average energy. We use $M_\theta$ for angular bins to avoid overloading $M$, which denotes the number of continual tasks in the evaluation metrics.

\vspace{1mm}
\noindent\textbf{Dominant peak and per-head descriptor.}
We locate the dominant frequency peak $(u^\star,v^\star)=\arg\max_{u,v}P(u,v)$ with radius $\rho^\star$ and angle $\theta^\star$, and record its power $p^\star=P(u^\star,v^\star)$.
For each layer--head $(l,h)$ we then form a spectral descriptor
\begin{equation}
\boldsymbol{\phi}^{(l,h)}(A)
= \big[\mathbf{r}^\top,\;\mathbf{d}^\top,\;\mathbf{a}^\top,\;\rho^\star,\;\theta^\star,\;p^\star\big]^\top
\in \mathbb{R}^{D_0},
\end{equation}
where $D_0 = K + 2M_\theta + 3$.

\vspace{1mm}
\noindent\textbf{Layer--head aggregation.}
Given a set of selected attention layers or blocks $\mathcal{L}_\mathrm{sel}\subseteq\mathcal{L}_\mathrm{attn}$ and heads $\mathcal{H}_\mathrm{sel}\subseteq\mathcal{H}_\mathrm{attn}$, we denote $\mathcal{J}=\mathcal{L}_\mathrm{sel}\times\mathcal{H}_\mathrm{sel}$ and $J=|\mathcal{J}|$.
For each sample we aggregate head-level descriptors by their mean and element-wise variance,
\begin{equation}
\begin{aligned}
\boldsymbol{\mu}_\phi &= \frac{1}{J}\sum_{(l,h)\in\mathcal{J}} \boldsymbol{\phi}^{(l,h)}(A), \\
\boldsymbol{\sigma}_\phi^2 &= \frac{1}{J}\sum_{(l,h)\in\mathcal{J}}\big(\boldsymbol{\phi}^{(l,h)}(A)-\boldsymbol{\mu}_\phi\big)^{\odot 2},
\end{aligned}
\end{equation}
and define the final per-sample spectral descriptor as
\begin{equation}
\boldsymbol{\phi}(A) =
[\boldsymbol{\mu}_\phi^\top,\;\boldsymbol{\sigma}_\phi^\top]^\top
\in \mathbb{R}^D,
\qquad
D=2D_0.
\end{equation}
We similarly average head-level angular spectra $\mathbf{d}^{(l,h)}(\theta)$ over $(l,h)\in\mathcal{J}$ to obtain a per-sample angular profile $\mathbf{d}(\theta)$, which we will later match to skill-specific angular prototypes.

\subsection{Skill-Conditioned Spectrum Prototypes}
\label{subsec:prototypes}

Different skills induce distinct but stable attention spectra across tasks. ASR therefore maintains, for each skill $s\in\mathcal{S}$, a Gaussian distribution over spectral descriptors and a mean angular spectrum.

\vspace{1mm}
\noindent\textbf{Per-stage moments.}
At the end of stage $t$, we collect descriptors attributed to skill $s$ as
\begin{equation}
  \mathcal{F}_s^{(t)} 
  = \big\{ \boldsymbol{\phi}_i^{(t)} \;\big|\; (I_i^{(t)}, q_i^{(t)}, y_i^{(t)}) \in \mathcal{D}^{(t)},\; \pi_s(q_i^{(t)}) > \tau_\mathrm{skill} \big\}
\end{equation}
where $s\in\mathcal{S}$, $\boldsymbol{\phi}_i^{(t)} = \boldsymbol{\phi}(A_i^{(t)})$ and $\tau_\mathrm{skill}\in[0,1)$ is a threshold. Let $n_s^{(t)} = |\mathcal{F}_s^{(t)}|$. We estimate per-stage mean and diagonal covariance
\begin{equation} \small
\begin{aligned}
  \widehat{\boldsymbol{\mu}}_s^{(t)}
  &= \frac{1}{n_s^{(t)}}
     \sum_{\boldsymbol{\phi} \in \mathcal{F}_s^{(t)}} \boldsymbol{\phi}, \\
  \widehat{\boldsymbol{\Sigma}}_s^{(t)}
  &= \operatorname{diag}\!\Big(
       \frac{1}{n_s^{(t)} - 1}
       \sum_{\boldsymbol{\phi} \in \mathcal{F}_s^{(t)}}
       \big(\boldsymbol{\phi} - \widehat{\boldsymbol{\mu}}_s^{(t)}\big)^{\odot 2}
     \Big).
\end{aligned}
\end{equation}
Similarly, with angular spectra $\mathbf{d}_i^{(t)} \in \mathbb{R}^M$ for those samples,
\begin{equation}
  \widehat{\mathbf{d}}_s^{(t)} 
  = \frac{1}{n_s^{(t)}}
    \sum_{i: \boldsymbol{\phi}_i^{(t)} \in \mathcal{F}_s^{(t)}}
    \mathbf{d}_i^{(t)}.
\end{equation}

\vspace{1mm}
\noindent\textbf{Exponential moving average.}
We maintain a global prototype memory
$
  \mathcal{M} = \big\{ (\boldsymbol{\mu}_s, \boldsymbol{\Sigma}_s, \hat{\mathbf{d}}_s) \big\}_{s \in \mathcal{S}},
$
initialized at $t=1$. After stage $t$ we update
\begin{align}
  \boldsymbol{\mu}_s 
  &\leftarrow \alpha \, \boldsymbol{\mu}_s + (1-\alpha) \, \widehat{\boldsymbol{\mu}}_s^{(t)}, \\
  \boldsymbol{\Sigma}_s 
  &\leftarrow \alpha \, \boldsymbol{\Sigma}_s + (1-\alpha) \, \widehat{\boldsymbol{\Sigma}}_s^{(t)}, \\
  \hat{\mathbf{d}}_s
  &\leftarrow \alpha \, \hat{\mathbf{d}}_s + (1-\alpha) \, \widehat{\mathbf{d}}_s^{(t)},
\end{align}
with $\alpha\in[0,1)$ controlling the memory horizon. Thus, for each $s$ we approximate
$
  \boldsymbol{\phi}(A)\,|\,s \sim \mathcal{N}(\boldsymbol{\mu}_s, \boldsymbol{\Sigma}_s),
$
and interpret $\hat{\mathbf{d}}_s$ (after normalization) as the mean angular spectrum for skill $s$.

\subsection{Spectral Distillation}
\label{subsec:distillation}

Given prototypes, ASR constrains the current model to use attention spectra consistent with previously learned skills.

\vspace{1mm}
\noindent\textbf{Skill-weighted Mahalanobis penalty.}
For a sample $(I,q)$ at stage $t{+}1$, let $\boldsymbol{\phi}=\boldsymbol{\phi}(A)$ be its descriptor, $\mathbf{d}(\theta)$ its angular spectrum, and $\boldsymbol{\pi}(q)$ its skill posterior. For $s\in\mathcal{S}$, define the squared Mahalanobis distance
\begin{equation}
  D^2_{\mathrm{mah}}(\boldsymbol{\phi} \,\Vert\, s)
  = \big(\boldsymbol{\phi} - \boldsymbol{\mu}_s\big)^\top
    \boldsymbol{\Sigma}_s^{-1}
    \big(\boldsymbol{\phi} - \boldsymbol{\mu}_s\big),
\end{equation}
where $\boldsymbol{\Sigma}_s^{-1}$ is the inverse diagonal covariance. The skill-weighted penalty is
\begin{equation}
  \ell_{\mathrm{spec}}^{\mathrm{mah}}(\boldsymbol{\phi}, q)
  = \sum_{s \in \mathcal{S}} 
    \pi_s(q) \, D^2_{\mathrm{mah}}(\boldsymbol{\phi} \,\Vert\, s).
\end{equation}

\vspace{1mm}
\noindent\textbf{Angular KL divergence.}
We normalize $\mathbf{d}(\theta)$ and $\hat{\mathbf{d}}_s$ to discrete distributions
\begin{equation}
  p_m = \frac{d_m}{\sum_{m'} d_{m'} + \varepsilon}, 
  \qquad
  \hat{p}_{s,m} = \frac{\hat{d}_{s,m}}{\sum_{m'} \hat{d}_{s,m'} + \varepsilon},
\end{equation}
and use a symmetric KL divergence.
The angular term is
\begin{equation}
  \ell_{\mathrm{spec}}^{\mathrm{ang}}(\mathbf{d}, q)
  = \sum_{s \in \mathcal{S}} 
    \pi_s(q) \, \mathrm{KL}_\mathrm{sym}\big(\mathbf{d} \,\Vert\, \hat{\mathbf{d}}_s\big).
\end{equation}

\vspace{1mm}
\noindent\textbf{Confidence-adaptive weight.}
We reduce regularization for spectra that are far from all prototypes. Define
\begin{equation}
\begin{split}
  d_{\min}(\boldsymbol{\phi}) 
  &= \min_{s \in \mathcal{S}} D^2_{\mathrm{mah}}(\boldsymbol{\phi} \,\Vert\, s), \\
  w_{\mathrm{spec}}(\boldsymbol{\phi}) 
  &= \exp\!\Big(-\frac{d_{\min}(\boldsymbol{\phi})}{\tau_{\mathrm{mah}}}\Big),
\end{split}
\end{equation}
with temperature $\tau_{\mathrm{mah}} > 0$. When $\boldsymbol{\phi}$ is close to some prototype, $w_{\mathrm{spec}}\approx1$; otherwise it decays exponentially.

\vspace{1mm}
\noindent\textbf{Spectral distillation loss.}
The spectral distillation loss for sample $(I,q)$ is
\begin{equation}
\begin{split}
  \mathcal{L}_{\mathrm{spec}}(I,q)
  = w_{\mathrm{spec}}\big(\boldsymbol{\phi}(A)\big)
    \bigg(& \ell_{\mathrm{spec}}^{\mathrm{mah}}\big(\boldsymbol{\phi}(A), q\big) \\
          & + \lambda_{\mathrm{ang}}
            \,\ell_{\mathrm{spec}}^{\mathrm{ang}}\big(\mathbf{d}(\theta), q\big)
    \bigg),
\end{split}
\end{equation}
with $\lambda_{\mathrm{ang}}\ge0$. 

\subsection{Objective and Training Schedule}
\label{subsec:objective}

\begin{algorithm}[t]
\small
\caption{Attention-Spectrum Regularization (ASR)}
\label{alg:asr}
\begin{algorithmic}[1]
\Require Stage datasets $\{\mathcal{D}^{(t)}\}_{t=1}^{T}$, MLLM $f_{\boldsymbol{\theta}}$, skill parser $g_{\boldsymbol{\psi}}$, decay $\alpha$, weights $\beta,\gamma$
\Ensure Final model parameters $\boldsymbol{\theta}$ and spectrum memory $\mathcal{M}$
\State Initialize prototype memory $\mathcal{M}\leftarrow \emptyset$
\For{$t=1,\ldots,T$}
    \For{mini-batch $\mathcal{B}\subset \mathcal{D}^{(t)}$}
        \State Run $f_{\boldsymbol{\theta}}$ on $\mathcal{B}$ to obtain predictions and cross-attention maps $\{A^{(l,h)}\}$
        \State Aggregate functional-token maps and encode spectra $(\boldsymbol{\phi}_i,\mathbf{d}_i)$ by 2D FFT
        \State Compute skill posterior $\boldsymbol{\pi}_i=g_{\boldsymbol{\psi}}(q_i)$ for each sample
        \State Compute $\mathcal{L}_{\mathrm{spec}}$ using skill-weighted Mahalanobis distance, angular KL divergence, and confidence-adaptive weight
        \State Compute $\mathcal{L}^{(t)}=\frac{1}{|\mathcal{B}|}\sum_i\big[\mathcal{L}_{\mathrm{task}}+\beta\mathcal{L}_{\mathrm{spec}}\big]+\gamma\mathcal{L}_{\mathrm{geo}}$
        \State Update $\boldsymbol{\theta}\leftarrow \boldsymbol{\theta}-\eta\nabla_{\boldsymbol{\theta}}\mathcal{L}^{(t)}$
    \EndFor
    \State Recompute $(\boldsymbol{\phi}_i,\mathbf{d}_i,\boldsymbol{\pi}_i)$ on $\mathcal{D}^{(t)}$ or a subset before discarding data
    \For{skill $s\in\mathcal{S}$}
        \State Collect $\mathcal{F}_s^{(t)}=\{\boldsymbol{\phi}_i:\pi_{i,s}>\tau_{\mathrm{skill}}\}$ and estimate $(\widehat{\boldsymbol{\mu}}_s^{(t)},\widehat{\boldsymbol{\Sigma}}_s^{(t)},\widehat{\mathbf{d}}_s^{(t)})$
        \State Initialize or update $(\boldsymbol{\mu}_s,\boldsymbol{\Sigma}_s,\widehat{\mathbf{d}}_s)\in\mathcal{M}$ by EMA with decay $\alpha$
    \EndFor
\EndFor
\State \Return $\boldsymbol{\theta},\mathcal{M}$
\end{algorithmic}
\end{algorithm}

Given $\hat{y} = f_{\boldsymbol{\theta}}(I,q)$ at stage $t$, we denote the task loss by $\mathcal{L}_{\mathrm{task}}(I,q,y)$.

\vspace{1mm}
\noindent\textbf{Geometry regularizer.}
To avoid over-regularization in the spectral space, we add a light geometric regularizer that anchors cross-modal similarities to a frozen reference encoder $f_{\boldsymbol{\theta}^0}$ (e.g., the original pretrained model). For a mini-batch of $B$ pairs $\{(I_i,q_i)\}_{i=1}^B$, let $z_{I,i}, z_{q,i}$ be embeddings from $f_{\boldsymbol{\theta}}$ and $z_{I,i}^0, z_{q,i}^0$ from $f_{\boldsymbol{\theta}^0}$. Define cosine similarity matrices
\begin{align}
  S_{ij} &= \frac{\langle z_{I,i}, z_{q,j} \rangle}{\|z_{I,i}\| \,\|z_{q,j}\|}, &
  S^0_{ij} &= \frac{\langle z_{I,i}^0, z_{q,j}^0 \rangle}{\|z_{I,i}^0\| \,\|z_{q,j}^0\|},
\end{align}
and a batch-level loss
\begin{equation}
\begin{split}
  \mathcal{L}_{\mathrm{geo}} 
  = \frac{1}{2B} \sum_{i=1}^B 
    \Big( & \mathrm{KL}\big(\sigma(S_{i,:}) \,\Vert\, \sigma(S^0_{i,:})\big) \\
          & + \mathrm{KL}\big(\sigma(S^0_{i,:}) \,\Vert\, \sigma(S_{i,:})\big) \Big),
\end{split}
\end{equation}
where $\sigma$ is row-wise softmax.

\vspace{1mm}
\noindent\textbf{Total loss.}
For a mini-batch $\mathcal{B}^{(t)} \subset \mathcal{D}^{(t)}$ at stage $t$, the total loss is
\begin{equation}
\begin{split}
  \mathcal{L}^{(t)} 
  = \frac{1}{|\mathcal{B}^{(t)}|}
    \sum_{(I,q,y) \in \mathcal{B}^{(t)}}
    \Big[ & \mathcal{L}_{\mathrm{task}}(I,q,y) \\
          & + \beta \,\mathcal{L}_{\mathrm{spec}}(I,q) \Big]
    + \gamma \,\mathcal{L}_{\mathrm{geo}},
\end{split}
\end{equation}
with hyperparameters $\beta,\gamma\ge0$.

\noindent\textbf{Training schedule.}
The overall training procedure is summarized in Algorithm~\ref{alg:asr}. At each stage, ASR adapts the current MLLM using only the current data, where the spectral loss is inactive at the first stage and is enabled once prototype memory is available. After stage-wise optimization, ASR extracts attention-spectrum descriptors from the current stage, groups them according to skill posteriors, and updates the skill-conditioned prototype memory by exponential moving average before discarding the data. In this way, ASR preserves compact skill-level attention statistics without storing past samples or maintaining a teacher model.

\section{Theory}
\label{sec:theory}

\subsection{Forgetting Bound via Skill-Conditioned Spectral Drift}
\label{subsec:theory_forgetting}

We formalize ASR as a memory mechanism that preserves the
skill-conditioned distribution of attention spectra.  Let
$x=(I,q)$ denote an image--question pair and let
$z_{\theta}(x)=\phi(A_{\theta}(x))\in\R^D$ be the ASR spectral
descriptor extracted from the cross-modal attention of model
$f_{\theta}$.  For an old stage $m$ and a later stage $t>m$, let
$\theta_m$ and $\theta_t$ be the parameters after training stages
$m$ and $t$, respectively.  We write $S\in\cS$ for the reasoning
skill and define the skill-conditioned descriptor law
\begin{equation}
  Q_{m,s}^{\theta}
  :=
  \law\!\left(
    z_{\theta}(X)
    \,\middle|\,
    (X,Y,S)\sim \cP_m,\ S=s
  \right),
  \qquad s\in\cS .
  \label{eq:main_descriptor_law}
\end{equation}
Let $\omega_{m,s}:=\Pr_{\cP_m}(S=s)$ be the skill mass at stage
$m$.  The old-stage risk and the forgetting from stage $m$ to stage
$t$ are
\begin{equation}
\begin{aligned}
  R_m(\theta)
  & :=
  \E_{(X,Y)\sim \cP_m}
  \left[
    \ell(f_{\theta}(X),Y)
  \right],
  \\
  \Fgt_{m\to t}
  & :=
  \left[
    R_m(\theta_t)-R_m(\theta_m)
  \right]_+ .
\end{aligned}
\label{eq:main_risk_forgetting_def}
\end{equation}

For each old skill $s$, ASR stores a spectral prototype
$(\mu_{m,s},\Sigma_{m,s})$, where $\Sigma_{m,s}\succ 0$ is the
skill-wise covariance used to whiten spectral deviations.  It
induces the Mahalanobis ground metric
\begin{equation}
  d_{m,s}(z,z')
  :=
  \left\|
    \Sigma_{m,s}^{-1/2}(z-z')
  \right\|_2 .
  \label{eq:main_mahalanobis_metric}
\end{equation}
The following assumption states that, on an old skill, the loss
increase caused by changing the model can be read from the drift of
its spectral descriptors up to a small residual.

\begin{assumption}[Skill-wise spectral sufficiency]
\label{ass:main_spectral_sufficiency}
For every old stage $m$ and skill $s\in\cS$, there exist a measurable
surrogate $g_{m,s}:\R^D\to\R$, a constant $L_{m,s}>0$, and a residual
$\epsilon_{m,s}\ge 0$ such that, for
$\theta\in\{\theta_m,\theta_t\}$,
\begin{equation}
\begin{aligned}
  \left|
    R_{m,s}(\theta)
    -
    \E_{Z\sim Q_{m,s}^{\theta}}
    \left[
      g_{m,s}(Z)
    \right]
  \right|
  &\le
  \epsilon_{m,s},
  \\
  \left|
    g_{m,s}(z)-g_{m,s}(z')
  \right|
  &\le
  L_{m,s} d_{m,s}(z,z') ,
\end{aligned}
\label{eq:main_spectral_sufficiency}
\end{equation}
where
$R_{m,s}(\theta)
 :=
 \E[\ell(f_{\theta}(X),Y)\mid S=s]$.
\end{assumption}

Let $\W_{p,m,s}$ be the $p$-Wasserstein distance under the ground
metric $d_{m,s}$.  Let
$\cG_{m,s}^{\theta}$ denote the Gaussian with the same spectral mean
and covariance as $Q_{m,s}^{\theta}$ under the representation stored
by ASR.  Define the non-Gaussianity and the Gaussian spectral drift
as
\begin{equation}
\begin{aligned}
  \Gamma_{m,s}^{\theta}
  &:=
  \W_{2,m,s}
  \!\left(
    Q_{m,s}^{\theta},
    \cG_{m,s}^{\theta}
  \right),
  \\
  \Delta_{m,t,s}
  &:=
  \W_{2,m,s}
  \!\left(
    \cG_{m,s}^{\theta_t},
    \cG_{m,s}^{\theta_m}
  \right).
\end{aligned}
\label{eq:main_gamma_delta_def}
\end{equation}
When $\cG_{m,s}^{\theta_m}=\mathcal{N}(\mu_{m,s},\Sigma_{m,s})$ and
$\cG_{m,s}^{\theta_t}=\mathcal{N}(\mu_{t|m,s},\Sigma_{t|m,s})$,
$\Delta_{m,t,s}$ has the closed form
\begin{equation}
\begin{aligned}
  \Delta_{m,t,s}^{2}
  &=
  \left\|
    \Sigma_{m,s}^{-1/2}
    (\mu_{t|m,s}-\mu_{m,s})
  \right\|_2^2
  \\
  &\quad+
  \left\|
    \left(
      \Sigma_{m,s}^{-1/2}
      \Sigma_{t|m,s}
      \Sigma_{m,s}^{-1/2}
    \right)^{1/2}
    -I_D
  \right\|_F^2 .
\end{aligned}
\label{eq:main_closed_form_delta}
\end{equation}

\begin{theorem}[Skill-conditioned spectral drift controls forgetting]
\label{thm:spectral_drift_forgetting}
Under Assumption~\ref{ass:main_spectral_sufficiency}, for every old
stage $m<t$, the forgetting of stage $m$ after training until stage
$t$ satisfies the sharp distributional bound
\begin{equation}
\begin{aligned}
  \Fgt_{m\to t}
  &\le
  \sum_{s\in\cS}
  \omega_{m,s} L_{m,s}
  \W_{1,m,s}
  \!\left(
    Q_{m,s}^{\theta_t},
    Q_{m,s}^{\theta_m}
  \right)
  \\
  &\quad+
  2
  \sum_{s\in\cS}
  \omega_{m,s}\epsilon_{m,s}.
\end{aligned}
\label{eq:main_sharp_w1_bound}
\end{equation}
Moreover, the distributional drift is controlled by the ASR
skill-wise Gaussian spectral drift:
\begin{equation}
\begin{aligned}
  \Fgt_{m\to t}
  &\le
  \sum_{s\in\cS}
  \omega_{m,s} L_{m,s}
  \left(
    \Gamma_{m,s}^{\theta_t}
    +
    \Delta_{m,t,s}
    +
    \Gamma_{m,s}^{\theta_m}
  \right)
  \\
  &\quad+
  2
  \sum_{s\in\cS}
  \omega_{m,s}\epsilon_{m,s}.
\end{aligned}
\label{eq:main_gaussian_drift_bound}
\end{equation}
In particular, if the skill-conditioned descriptors are Gaussian in
the ASR spectral space, then
$\Gamma_{m,s}^{\theta_t}=\Gamma_{m,s}^{\theta_m}=0$ and forgetting is
controlled only by the closed-form drift
$\Delta_{m,t,s}$ in Eq.~\eqref{eq:main_closed_form_delta}.
\end{theorem}

\noindent\textbf{Proof sketch.}
For each old skill, Assumption~\ref{ass:main_spectral_sufficiency}
reduces the old-task risk difference to the difference of two
expectations of an $L_{m,s}$-Lipschitz function over the old and
current descriptor laws.  Therefore,
\begin{equation}
\begin{aligned}
  \left[
    R_{m,s}(\theta_t)-R_{m,s}(\theta_m)
  \right]_+
  &\le
  L_{m,s}
  \W_{1,m,s}
  \!\left(
    Q_{m,s}^{\theta_t},
    Q_{m,s}^{\theta_m}
  \right)
  \\
  &\quad+
  2\epsilon_{m,s}.
\end{aligned}
\label{eq:main_sketch_skill_bound}
\end{equation}
Averaging Eq.~\eqref{eq:main_sketch_skill_bound} over the skill
weights $\omega_{m,s}$ gives Eq.~\eqref{eq:main_sharp_w1_bound}.
The second bound follows by inserting Gaussian prototypes between
the two descriptor laws:
\begin{equation}
\begin{aligned}
  \W_{1,m,s}
  \!\left(
    Q_{m,s}^{\theta_t},
    Q_{m,s}^{\theta_m}
  \right)
  &\le
  \W_{2,m,s}
  \!\left(
    Q_{m,s}^{\theta_t},
    Q_{m,s}^{\theta_m}
  \right)
  \\
  &\le
  \Gamma_{m,s}^{\theta_t}
  +
  \Delta_{m,t,s}
  +
  \Gamma_{m,s}^{\theta_m}.
\end{aligned}
\label{eq:main_sketch_gaussian_insert}
\end{equation}
Finally, Eq.~\eqref{eq:main_closed_form_delta} is the exact
$2$-Wasserstein distance between two Gaussians after whitening by
$\Sigma_{m,s}^{-1/2}$.  The full proof is given in
Appendix~\ref{app:spectral_forgetting}.

\subsection{Phase-Invariant Stability of Spectral Attention}
\label{subsec:phase_invariant_theory}

ASR regularizes the Fourier power spectrum of cross-attention maps
rather than the raw attention maps.  This choice induces a useful
invariance: spatial translations of the attention map only change
the phase of its Fourier coefficients and therefore leave the power
spectrum unchanged.  We formalize this property and show that the
resulting spectral descriptor is stable under noisy or boundary
corrupted translations.

Let $A\in\R^{H\times W}$ be a nonzero attention map and let
$\widehat A=\cF A$ be the unitary two-dimensional discrete Fourier
transform, defined by
\begin{equation}
  \widehat A(u,v)
  =
  \frac{1}{\sqrt{HW}}
  \sum_{x=0}^{H-1}
  \sum_{y=0}^{W-1}
  A(x,y)
  \exp\!\left(
    -2\pi i
    \left(
      \frac{ux}{H}
      +
      \frac{vy}{W}
    \right)
  \right).
  \label{eq:phase_main_dft_def}
\end{equation}
The normalized Fourier power spectrum is
\begin{equation}
  p(A)_{u,v}
  :=
  \frac{
    |\widehat A(u,v)|^2
  }{
    \|A\|_F^2
  },
  \qquad
  p(A)\in\Delta^{HW-1}.
  \label{eq:phase_main_power_def}
\end{equation}
Let $C$ be the deterministic ASR spectral coarsening operator that
maps the frequency distribution to radial/angular statistics, and
define
\begin{equation}
  \Psi(A)
  :=
  C p(A),
  \qquad
  \kappa_C
  :=
  \sup_{r\ne 0}
  \frac{
    \|Cr\|_2
  }{
    \|r\|_1
  }.
  \label{eq:phase_main_descriptor_def}
\end{equation}
For a cyclic shift $\delta=(\delta_x,\delta_y)$, define
\begin{equation}
  (\cT_{\delta}A)(x,y)
  :=
  A\big((x-\delta_x)\!\!\mod H,\,
        (y-\delta_y)\!\!\mod W\big).
  \label{eq:phase_main_translation_def}
\end{equation}

We consider a skill-conditioned source-to-target attention shift
where a target-domain attention map is generated as
\begin{equation}
  A^{\mathrm{tar}}
  =
  \cT_{\Delta} A^{\mathrm{src}} + E,
  \qquad
  \rho
  :=
  \frac{
    \|E\|_F
  }{
    \|A^{\mathrm{src}}\|_F
  },
  \label{eq:phase_main_shift_model}
\end{equation}
where $\Delta$ may be random and may depend on the sample, and $E$
captures non-translational perturbations such as boundary effects,
occlusion, or attention noise.

\begin{assumption}[Spectral Lipschitz risk]
\label{ass:phase_main_lipschitz_risk}
For each skill $s\in\cS$, the old-skill loss can be represented by
a spectral surrogate $h_s$ satisfying
\begin{equation}
  |h_s(z)-h_s(z')|
  \le
  L_s\|z-z'\|_2,
  \qquad
  z,z'\in\mathrm{range}(\Psi).
  \label{eq:phase_main_lipschitz_risk}
\end{equation}
\end{assumption}

\begin{theorem}[Phase-invariant spectral stability]
\label{thm:phase_invariant_stability}
For every nonzero attention map $A$ and every cyclic translation
$\delta$, the Fourier power spectrum and the ASR spectral descriptor
are exactly invariant:
\begin{equation}
  p(\cT_{\delta}A)=p(A),
  \qquad
  \Psi(\cT_{\delta}A)=\Psi(A).
  \label{eq:phase_main_exact_invariance}
\end{equation}
Moreover, under the perturbed shift model in
Eq.~\eqref{eq:phase_main_shift_model}, if $\rho\le 1$, then
\begin{equation}
  \left\|
    p(A^{\mathrm{tar}})
    -
    p(A^{\mathrm{src}})
  \right\|_1
  \le
  2\rho,
  \left\|
    \Psi(A^{\mathrm{tar}})
    -
    \Psi(A^{\mathrm{src}})
  \right\|_2
  \le
  2\kappa_C\rho .
  \label{eq:phase_main_pointwise_stability}
\end{equation}
The constant $2$ in the power-spectrum bound is sharp.

Consequently, for any skill $s$ and any coupling between source and
target attention maps satisfying
Eq.~\eqref{eq:phase_main_shift_model},
\begin{equation}
\begin{aligned}
  &\W_1
  \left(
    \law(\Psi(A^{\mathrm{tar}})\mid s),
    \law(\Psi(A^{\mathrm{src}})\mid s)
  \right)
  \\
  &\qquad\le
  2\kappa_C
  \E
  \left[
    \rho
    \mid s
  \right].
\end{aligned}
\label{eq:phase_main_distribution_stability}
\end{equation}
Under Assumption~\ref{ass:phase_main_lipschitz_risk}, the
skill-conditioned spectral risk gap obeys
\begin{equation}
\begin{aligned}
  &\left|
    \E
    \left[
      h_s(\Psi(A^{\mathrm{tar}}))
      \mid s
    \right]
    -
    \E
    \left[
      h_s(\Psi(A^{\mathrm{src}}))
      \mid s
    \right]
  \right|
  \\
  &\qquad\le
  2L_s\kappa_C
  \E
  \left[
    \rho
    \mid s
  \right].
\end{aligned}
\label{eq:phase_main_risk_stability}
\end{equation}
In particular, when $E=0$, the spectral distribution and the
spectral risk are invariant to arbitrary skill-conditioned random
translations.
\end{theorem}

\noindent\textbf{Proof sketch.}
The DFT diagonalizes cyclic translations.  Specifically,
\begin{equation}
  \widehat{\cT_{\delta}A}(u,v)
  =
  \exp\!\left(
    -2\pi i
    \left(
      \frac{u\delta_x}{H}
      +
      \frac{v\delta_y}{W}
    \right)
  \right)
  \widehat A(u,v).
  \label{eq:phase_main_sketch_phase_ramp}
\end{equation}
Thus the translation only multiplies each Fourier coefficient by a
unit-modulus phase factor, proving
Eq.~\eqref{eq:phase_main_exact_invariance}.  For the perturbed case,
let $a=\cF(\cT_{\Delta}A^{\mathrm{src}})$ and
$b=\cF(A^{\mathrm{tar}})=a+\cF E$.  Parseval gives
$\|\cF E\|_2=\|E\|_F$ and $\|a\|_2=\|A^{\mathrm{src}}\|_F$.
The key step is the sharp projective stability inequality
\begin{equation}
  \left\|
    \frac{|b|^{\odot 2}}{\|b\|_2^2}
    -
    \frac{|a|^{\odot 2}}{\|a\|_2^2}
  \right\|_1
  \le
  2
  \frac{\|b-a\|_2}{\|a\|_2}.
  \label{eq:phase_main_sketch_projective_bound}
\end{equation}
This inequality follows by bounding the classical total variation
between the coordinate measurements of two unit vectors by the trace
distance between the corresponding rank-one projectors.  Applying
the coarsening operator $C$ yields the descriptor bound in
Eq.~\eqref{eq:phase_main_pointwise_stability}.  Finally,
Eq.~\eqref{eq:phase_main_distribution_stability} follows by using
the source--target coupling as a valid transport plan, and
Eq.~\eqref{eq:phase_main_risk_stability} follows from the Lipschitz
property of $h_s$.  Full proofs are in
Appendix~\ref{app:phase_invariant_stability}.

\begin{table*}[t]
\centering
\small
\caption{
Continual VQA results.
AP is higher better and AF is lower better.
Joint denotes the multitask upper bound.
Green numbers after ASR results indicate improvement over the strongest baseline for each metric.
}
\label{tab:vqa_results}
\resizebox{\textwidth}{!}{%
\begin{tblr}{
  colspec = {
    Q[l,wd=4.15cm]
    Q[c,wd=1.15cm]
    |
    Q[c,wd=1.18cm]
    Q[c,wd=1.18cm]
    |
    Q[c,wd=1.36cm]
    Q[c,wd=1.36cm]
    Q[c,wd=1.36cm]
    Q[c,wd=1.36cm]
    |
    Q[c,wd=1.52cm]
    Q[c,wd=1.52cm]
  },
  row{1-2} = {bg=ASRHeader, font=\bfseries},
  cell{1}{1} = {r=2}{l},
  cell{1}{2} = {r=2}{c},
  cell{1}{3} = {c=2}{c},
  cell{1}{5} = {c=4}{c},
  cell{1}{9} = {c=2}{c},
  vline{3,5,9} = {0.5pt},
  rowsep = 1.45pt,
  colsep = 2.3pt,
}
\toprule
\textbf{Method} & \textbf{Venue}
& \textbf{VQA v2 (10 tasks)} &
& \textbf{VQACL (VQAv2+NExT-QA)} & & &
& \textbf{CLT-VQA} & \\
&
& AP$\uparrow$ & AF$\downarrow$
& AP$_\text{std}\uparrow$ & AF$_\text{std}\downarrow$
& AP$_\text{nov}\uparrow$ & AF$_\text{nov}\downarrow$
& AP$\uparrow$ & AF$\downarrow$ \\
\midrule

\SetRow{bg=ASRGroup,font=\itshape}
\SetCell[c=10]{l} Reference
& & & & & & & & & \\
\mb{Vanilla}{LLaVA-1.5-7B}
& \venue{--}
& 44.1 & 18.3
& 40.2 & 19.5
& 28.7 & 17.3
& 36.8 & 22.5 \\

\midrule
\SetRow{bg=ASRGroup,font=\itshape}
\SetCell[c=10]{l} Regularization-based
& & & & & & & & & \\
\mb{EWC~\cite{EWC}}{LLaVA-1.5-7B}
& \venue{PNAS'17}
& 45.0 & 15.7
& 40.9 & 18.1
& 29.2 & 16.8
& 37.6 & 20.4 \\
\mb{LwF~\cite{LwF}}{LLaVA-1.5-7B}
& \venue{ECCV'16}
& 45.8 & 14.2
& 41.3 & 17.2
& 29.8 & 15.9
& 38.3 & 19.1 \\

\midrule
\SetRow{bg=ASRGroup,font=\itshape}
\SetCell[c=10]{l} Replay-based
& & & & & & & & & \\
\mb{ER~\cite{ER}}{LLaVA-1.5-7B}
& \venue{--}
& 47.9 & 11.1
& 42.7 & 14.0
& 31.0 & 13.1
& 40.8 & 15.7 \\

\midrule
\SetRow{bg=ASRGroup,font=\itshape}
\SetCell[c=10]{l} Continual VQA / VLM methods
& & & & & & & & & \\
\mb{VQACL~\cite{vqacl23}}{T5-style}
& \venue{CVPR'23}
& 48.5 & 9.2
& 44.1 & 11.6
& 33.4 & 10.4
& 41.5 & 13.9 \\
\mb{QUAD~\cite{quad25}}{T5-style}
& \venue{ICCV'25}
& 50.3 & 6.1
& \third{46.8} & 7.8
& 36.5 & 7.1
& 43.6 & 10.1 \\
\mb{CL-MoE~\cite{clmoe25}}{LLaVA-7B}
& \venue{CVPR'25}
& \second{51.2} & \second{3.8}
& \second{47.9} & \second{6.2}
& \second{37.8} & \second{6.0}
& 44.4 & 8.4 \\
\mb{BCP-MFA~\cite{BCPMFA25}}{VQA Enc-Dec}
& \venue{ICCV'25}
& 49.4 & 7.5
& 45.9 & 9.1
& 35.0 & 8.6
& \second{45.4} & \second{6.4} \\
\mb{KeepLoRA~\cite{keeplora26}}{LLaVA-1.5-7B}
& \venue{ICLR'26}
& \third{50.6} & \third{4.5}
& 46.5 & \third{6.9}
& \third{37.0} & \third{6.6}
& \third{44.7} & \third{7.2} \\

\midrule
\SetRow{bg=ASRGroup,font=\itshape}
\SetCell[c=10]{l} Upper bound
& & & & & & & & & \\
\mb{Joint}{LLaVA-1.5-7B}
& \venue{--}
& 55.0 & 0.0
& 52.4 & 0.0
& 43.0 & 0.0
& 49.8 & 0.0 \\

\midrule
\SetRow{bg=ASRGroup,font=\itshape}
\SetCell[c=10]{l} Ours
& & & & & & & & & \\

\SetRow{bg=ASROurs}
\mb{ASR}{LLaVA-1.5-7B}
& \venue{Ours}
& \best{52.0}\gain{$\uparrow$0.8}
& \best{2.4}\gain{$\downarrow$1.4}
& \best{48.6}\gain{$\uparrow$0.7}
& \best{5.1}\gain{$\downarrow$1.1}
& \best{39.1}\gain{$\uparrow$1.3}
& \best{4.9}\gain{$\downarrow$1.1}
& \best{46.2}\gain{$\uparrow$0.8}
& \best{5.5}\gain{$\downarrow$0.9} \\

\SetRow{bg=ASROurs}
\mb{ASR}{Qwen2.5-VL-7B}
& \venue{Ours}
& 54.2\gain{$\uparrow$3.0}
& 2.0\gain{$\downarrow$1.8}
& 51.4\gain{$\uparrow$3.5}
& 4.6\gain{$\downarrow$1.6}
& 42.0\gain{$\uparrow$4.2}
& 4.4\gain{$\downarrow$1.6}
& 48.0\gain{$\uparrow$2.6}
& 5.1\gain{$\downarrow$1.3} \\

\SetRow{bg=ASROurs}
\mb{ASR}{InternVL3-8B}
& \venue{Ours}
& 53.6\gain{$\uparrow$2.4}
& 2.1\gain{$\downarrow$1.7}
& 50.8\gain{$\uparrow$2.9}
& 4.8\gain{$\downarrow$1.4}
& 41.4\gain{$\uparrow$3.6}
& 4.5\gain{$\downarrow$1.5}
& 47.4\gain{$\uparrow$2.0}
& 5.3\gain{$\downarrow$1.1} \\

\bottomrule
\end{tblr}
}
\end{table*}

\begin{table}[t]
\centering
\small
\caption{
Continual multimodal instruction tuning on CoIN and UCIT.
We report Last and Avg, where higher is better.
Joint denotes the multitask upper bound.
Green numbers after ASR results indicate improvement over the strongest non-ours baseline for each metric.
}
\label{tab:cit_results}
\resizebox{\linewidth}{!}{%
\begin{tblr}{
  colspec = {
    Q[l,wd=4.25cm]
    Q[c,wd=1.20cm]
    |
    Q[c,wd=1.18cm]
    Q[c,wd=1.18cm]
    |
    Q[c,wd=1.18cm]
    Q[c,wd=1.18cm]
  },
  row{1-2} = {bg=ASRHeader, font=\bfseries},
  cell{1}{1} = {r=2}{l},
  cell{1}{2} = {r=2}{c},
  cell{1}{3} = {c=2}{c},
  cell{1}{5} = {c=2}{c},
  vline{3,5} = {0.5pt},
  rowsep = 1.45pt,
  colsep = 2.8pt,
}
\toprule
\textbf{Method} & \textbf{Venue}
& \textbf{CoIN} &
& \textbf{UCIT} & \\
&
& Last$\uparrow$ & Avg$\uparrow$
& Last$\uparrow$ & Avg$\uparrow$ \\
\midrule

\SetRow{bg=ASRGroup,font=\itshape}
\SetCell[c=6]{l} Reference
& & & & & \\
\mb{LoRA-FT}{LLaVA-1.5-7B}
& \venue{ICLR'22}
& 49.3 & 47.2
& 41.5 & 39.8 \\

\midrule
\SetRow{bg=ASRGroup,font=\itshape}
\SetCell[c=6]{l} LoRA / MoE-based continual tuning
& & & & & \\
\mb{O-LoRA~\cite{Olora}}{LLaVA-1.5-7B}
& \venue{--}
& 52.1 & 49.8
& 43.2 & 42.1 \\
\mb{MoELoRA~\cite{moelora}}{LLaVA-1.5-7B}
& \venue{--}
& 54.7 & 52.4
& 45.9 & 44.3 \\
\mb{HiDe-LLaVA~\cite{Hidellava25}}{LLaVA-1.5-7B}
& \venue{ACL'25}
& 57.2 & 55.1
& 47.8 & 46.0 \\
\mb{SEFE~\cite{SEFE25}}{LLaVA-1.5-7B}
& \venue{ICML'25}
& 58.4 & 56.2
& 49.0 & 47.3 \\
\mb{BranchLoRA~\cite{BranchLoRA25}}{LLaVA-1.5-7B}
& \venue{ACL'25}
& 59.1 & 56.9
& 49.7 & 47.9 \\
\mb{D-MoLE~\cite{DMoLE25}}{InternVL2-2B}
& \venue{ICML'25}
& \second{60.3} & 58.0
& 51.2 & 49.4 \\
\mb{Adapt-$\infty$~\cite{Adapt25}}{LLaVA-1.5-7B}
& \venue{ICLR'25}
& \third{59.7} & \second{58.8}
& \second{51.7} & \second{50.2} \\
\mb{PCLR~\cite{PCLR26}}{LLaVA-7B}
& \venue{ICLR'26}
& 59.5 & \third{58.2}
& \third{51.3} & \third{49.8} \\

\midrule
\SetRow{bg=ASRGroup,font=\itshape}
\SetCell[c=6]{l} Upper bound
& & & & & \\
\mb{Joint}{LLaVA-1.5-7B}
& \venue{--}
& 64.5 & 63.1
& 56.2 & 54.8 \\

\midrule
\SetRow{bg=ASRGroup,font=\itshape}
\SetCell[c=6]{l} Ours
& & & & & \\

\SetRow{bg=ASROurs}
\mb{ASR}{LLaVA-1.5-7B}
& \venue{Ours}
& \best{61.4}\gain{$\uparrow$1.1}
& \best{60.0}\gain{$\uparrow$1.2}
& \best{53.1}\gain{$\uparrow$1.4}
& \best{51.5}\gain{$\uparrow$1.3} \\

\SetRow{bg=ASROurs}
\mb{ASR}{Qwen2.5-VL-7B}
& \venue{Ours}
& 63.2\gain{$\uparrow$2.9}
& 61.8\gain{$\uparrow$3.0}
& 55.0\gain{$\uparrow$3.3}
& 53.3\gain{$\uparrow$3.1} \\

\SetRow{bg=ASROurs}
\mb{ASR}{InternVL3-8B}
& \venue{Ours}
& 62.7\gain{$\uparrow$2.4}
& 61.2\gain{$\uparrow$2.4}
& 54.5\gain{$\uparrow$2.8}
& 52.8\gain{$\uparrow$2.6} \\

\bottomrule
\end{tblr}
}
\end{table}

\section{Experimental Results}

\subsection{Experimental Setup}
\label{subsec:exp_setup}

\vspace{1mm}
\noindent\textbf{Datasets.}
ASR is evaluated on five continual multimodal settings:
\emph{(i) VQA v2 (Question-Type Incremental).}
It contains $\sim$200k images and 1.1M questions sourced mainly from MS-COCO. Following the VQACL/CL-MoE protocol, we split VQA v2 into $M{=}10$ tasks by question type as \citet{quad25}.
\emph{(ii) VQAv2 + NExT-QA (VQACL skill--concept setting).}
To evaluate compositional generalization and cross-dataset transfer, we additionally follow the VQACL~\cite{vqacl23} setting, where VQAv2 and NExT-QA are organized as a two-level sequence over reasoning skills and visual concepts.
\emph{(iii) CLT-VQA setting.}
To further examine robustness under imbalanced continual learning, we include the continual long-tailed VQA setting following \citet{BCPMFA25}. This setting introduces long-tailed answer and concept distributions across sequential VQA tasks, stressing whether a method can preserve previously learned visual reasoning patterns while adapting to new and under-represented concepts.
\emph{(iv) CoIN Benchmark for CIT.}
For continual multimodal instruction tuning, we follow \citet{SEFE25} to use the CoIN benchmark~\cite{COIN}.
\emph{(v) UCIT Benchmark for Unseen CIT.}
To stress generalization to tasks unseen during supervised fine-tuning of the base MLLM, we follow HiDe-LLaVA~\cite{Hidellava25} to evaluate on UCIT.

\vspace{1mm}
\noindent\textbf{Evaluation metrics.}
We follow standard continual learning metrics for multimodal VQA and CIT. For each task $b$ and training stage $a$, we record per-task performance $m_{a,b}$ using the task's official metric, including VQA accuracy for VQAv2 and CLT-VQA, WUPS for NExT-QA, and dataset-specific scores for CoIN/UCIT. Overall performance is summarized by \emph{Final Average Performance}
$
\mathrm{AP} = \frac{1}{M} \sum_{t=1}^{M} m_{M,t}
$
and \emph{Average Forgetting} $\mathrm{AF}$.
In the VQACL/QUAD skill--concept setting, we additionally report AP/AF separately on standard, i.e., seen skill--concept, and novel-composition splits. For the CLT-VQA setting, we report the overall AP/AF. For continual instruction tuning on CoIN and UCIT, we follow prior work and report \emph{Last}, i.e., final performance averaged over all $T$ stages, and \emph{Avg}, i.e., time-averaged performance across the training trajectory.

\vspace{1mm}
\noindent\textbf{Implementation details.}
Unless otherwise specified, main experiments are conducted with LLaVA-1.5-7B as the base MLLM. The backbone is kept frozen, and only lightweight adapters and task-specific heads are trainable. To verify that ASR is not tied to a specific LLaVA-style architecture, we further conduct backbone robustness experiments on Qwen2.5-VL-7B~\cite{qwen25vl} and InternVL3-8B~\cite{internvl3}. For LLaVA-1.5-7B, Qwen2.5-VL-7B, and InternVL3-8B, ASR extracts text-to-vision cross-modal attention maps from selected decoder self-attention blocks after visual tokens have been inserted into the language-model context. For backbones that expose explicit image--text cross-attention, ASR directly uses the native cross-attention matrices. In all cases, the extracted maps are reconstructed onto a visual grid and normalized before spectral encoding. The architecture-specific extraction details are given in Appendix~\ref{app:attention_extraction}.

For the VQA v2 10-task split, we train each task for 1 epoch with batch size $16$, using AdamW with learning rate $2\mathrm{e}{-4}$ and cosine decay. For the VQACL/QUAD setting on VQAv2 and NExT-QA, we train for 3 epochs per task with batch size $80$ and learning rate $1\mathrm{e}{-4}$. For the CLT-VQA setting, we follow the task sequence, long-tailed split, and optimization protocol of \citet{BCPMFA25}. For CoIN and UCIT, we follow \citet{Hidellava25} schedules: 1 epoch per task for the majority of datasets and up to 5 epochs for smaller OCR-style datasets. For Qwen2.5-VL-7B and InternVL3-8B, we use the same ASR hyperparameters as LLaVA-1.5-7B, and adjust only the per-device batch size with gradient accumulation to keep the effective batch size unchanged.

Please refer to the supplementary materials for more details.

\subsection{Main results}
\label{sec:main-results}

\noindent\textbf{Continual VQA.}
Table~\ref{tab:vqa_results} reports results on VQA v2, VQACL, and CLT-VQA. On the default LLaVA-1.5-7B backbone, ASR achieves the best results among all non-joint methods. On VQA v2, ASR improves over CL-MoE~\cite{clmoe25} by +0.8 AP and reduces AF by 1.4 points. On VQACL, ASR consistently improves both standard and novel-composition splits, with gains of +0.7 AP$_\text{std}$, +1.3 AP$_\text{nov}$, and lower forgetting. On CLT-VQA, ASR further improves over BCP-MFA~\cite{BCPMFA25} by +0.8 AP and 0.9 AF, showing its robustness under long-tailed continual VQA.
ASR also generalizes to stronger backbones. With Qwen2.5-VL-7B and InternVL3-8B, ASR obtains higher absolute AP and lower AF across all three VQA settings.

\vspace{1mm}
\noindent\textbf{Continual instruction tuning.}
Table~\ref{tab:cit_results} reports results on CoIN and UCIT. On LLaVA-1.5-7B, ASR achieves 61.4 Last and 60.0 Avg on CoIN, improving over the strongest baselines by +1.1 and +1.2 points, respectively. On UCIT, ASR reaches 53.1 Last and 51.5 Avg, outperforming Adapt-$\infty$~\cite{Adapt25} by +1.4 and +1.3 points. With Qwen2.5-VL-7B and InternVL3-8B, ASR further improves the absolute performance on both benchmarks, demonstrating consistent effectiveness across different MLLM backbones.

\subsection{Ablation Results}
\label{sec:ablation}

\vspace{1mm}
\noindent\textbf{Single–factor ablations.}
We perform single–factor ablations of ASR on the VQA v2 10-task split and CoIN CIT benchmark:
\ding{182} \textbf{w/o Spectrum Distillation} (remove $\mathcal{L}_{\mathrm{spec}}$; only task loss + $\mathcal{L}_{\mathrm{geo}}$);\quad
\ding{183} \textbf{w/o Skill Conditioning} (replace skill-specific prototypes $\{(\boldsymbol{\mu}_s,\boldsymbol{\Sigma}_s,\hat{\mathbf{d}}_s)\}_{s\in\mathcal{S}}$ by a single global prototype);\quad
\ding{184} \textbf{w/o Angular term} (set $\lambda_{\mathrm{ang}}{=}0$; only radial / Mahalanobis spectrum matching);\quad
\ding{185} \textbf{w/o Confidence Weighting} (set $w_{\mathrm{spec}}(\boldsymbol{\phi}) \equiv 1$);\quad
\ding{186} \textbf{w/o Geometry Regularizer} (set $\gamma{=}0$; no $\mathcal{L}_{\mathrm{geo}}$ anchor).
From Table~\ref{tab:ablation-asr}, removing \textbf{Spectrum Distillation} causes by far the largest degradation, confirming it as the main driver of stability, while collapsing per-skill prototypes into a single global prototype still yields clear drops in AP/Last and higher AF, underscoring the value of skill-conditioned priors. The angular term, confidence weighting, and geometry regularizer each contribute smaller but consistent gains, indicating that directional cues, adaptive regularization, and a light geometric anchor jointly turn spectral statistics into a robust constraint across tasks.

\newcommand{\abl}[2]{#1\,{\scriptsize\textcolor{blue}{(#2)}}}

\begin{table}[t]
\centering
\small
\setlength{\tabcolsep}{6pt}
\renewcommand{\arraystretch}{1.15}
\caption{\textbf{Single–factor ablations of ASR.}}
\vspace{-3mm}
\label{tab:ablation-asr}
\resizebox{\linewidth}{!}{%
\begin{tabular}{lcccc}
\toprule
\multirow{2}{*}{Variant} &
\multicolumn{2}{c}{VQA v2 (10 tasks)} &
\multicolumn{2}{c}{CoIN} \\
\cmidrule(lr){2-3} \cmidrule(lr){4-5}
& AP$\uparrow$ & AF$\downarrow$ &
Last$\uparrow$ & Avg$\uparrow$ \\
\midrule
\textbf{Full ASR} 
& \textbf{52.0} & \textbf{2.4} & \textbf{61.4} & \textbf{60.0} \\
\midrule
\ding{182}
& \abl{50.1}{-1.9} & \abl{4.3}{+1.9} & \abl{59.0}{-2.4} & \abl{57.7}{-2.3} \\
\ding{183} 
& \abl{51.1}{-0.9} & \abl{3.5}{+1.1} & \abl{60.0}{-1.4} & \abl{58.6}{-1.4} \\
\ding{184} 
& \abl{51.4}{-0.6} & \abl{3.0}{+0.6} & \abl{60.6}{-0.8} & \abl{59.1}{-0.9} \\
\ding{185}
& \abl{51.2}{-0.8} & \abl{3.2}{+0.8} & \abl{60.2}{-1.2} & \abl{58.9}{-1.1} \\
\ding{186}
& \abl{51.6}{-0.4} & \abl{2.9}{+0.5} & \abl{60.7}{-0.7} & \abl{59.4}{-0.6} \\
\bottomrule
\end{tabular}%
}
\end{table}

\begin{figure*}[t]
\centering
\begin{minipage}[t]{0.32\linewidth}
\centering
\includegraphics[width=\linewidth]{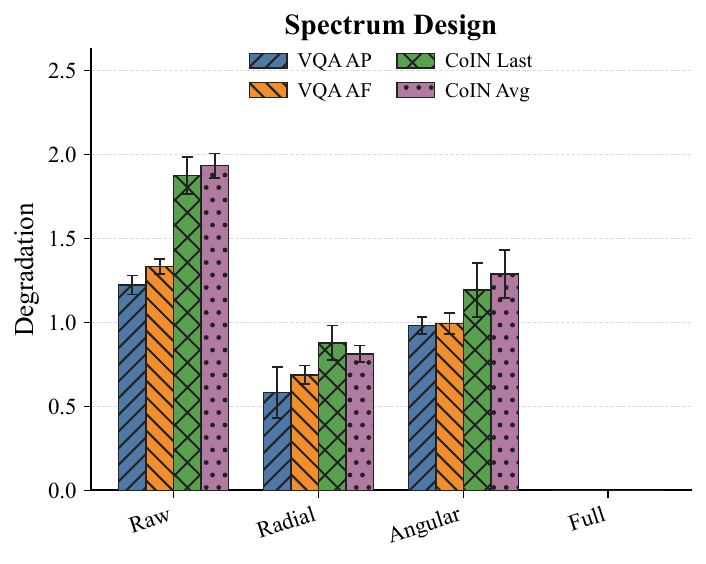}
\centerline{\small (a) Spectrum design}
\end{minipage}
\hfill
\begin{minipage}[t]{0.32\linewidth}
\centering
\includegraphics[width=\linewidth]{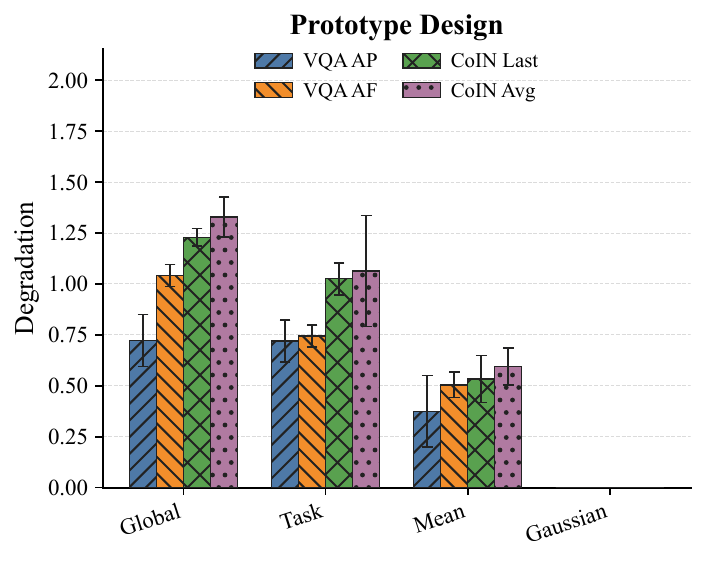}
\centerline{\small (b) Prototype design}
\end{minipage}
\hfill
\begin{minipage}[t]{0.32\linewidth}
\centering
\includegraphics[width=\linewidth]{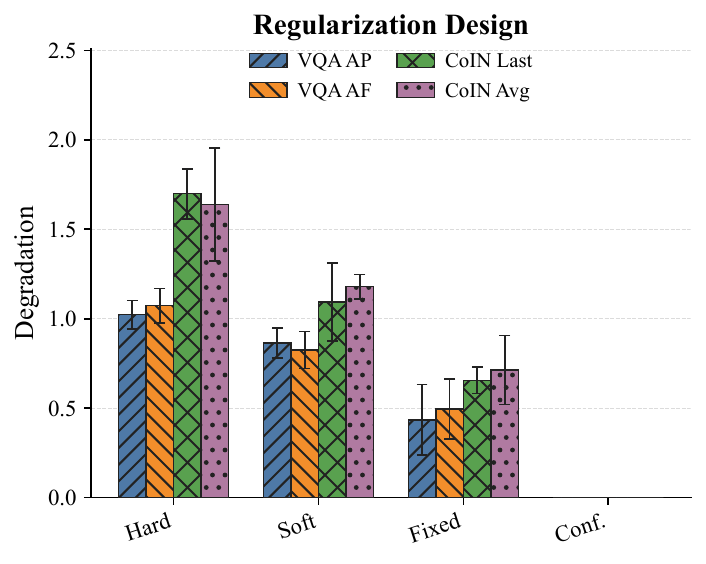}
\centerline{\small (c) Regularization design}
\end{minipage}
\caption{
\textbf{Design-level ablations of ASR.}
All plots show degradation relative to the full ASR variant.
For AP, Last, and Avg, lower bars indicate smaller performance drops.
For AF, lower bars indicate less forgetting increase.
}
\label{fig:design-ablation}
\end{figure*}

\noindent\textbf{Spectrum design variants.}
We study how different attention representations affect continual stability.
We compare four variants: directly distilling raw spatial attention maps
(\emph{Raw Attention}), using only the radial frequency profile
(\emph{Radial-only}), using only the angular profile
(\emph{Angular-only}), and the complete spectrum descriptor used by ASR
(\emph{Full Spectrum}).
As shown in Fig.~\ref{fig:design-ablation}(a), raw attention distillation is less effective, obtaining 50.8 AP / 3.7 AF on VQA v2 and 59.5 Last / 58.2 Avg on CoIN.
Radial-only matching performs better than angular-only matching, suggesting that spatial scale is more critical for preserving visual grounding.
However, the full spectrum achieves the best results, confirming that both scale and directional cues are useful for stabilizing cross-modal attention.

\vspace{1mm}
\noindent\textbf{Prototype design variants.}
We then examine how spectrum prototypes should be organized.
We compare a single global prototype (\emph{Global}), task-level prototypes
(\emph{Task-wise}), skill-wise prototypes with mean only
(\emph{Skill-wise Mean}), and the proposed skill-wise Gaussian prototypes
(\emph{Skill-wise Gaussian}).
Fig.~\ref{fig:design-ablation}(b) shows that global prototypes are insufficient, dropping to 51.1 AP and 60.0 Last.
Task-wise prototypes improve the result slightly, but remain weaker than skill-wise prototypes because tasks can contain mixed reasoning skills.
Using skill-wise mean prototypes further improves stability, while the full Gaussian modeling reaches the best result.
This supports our design choice of modeling forgetting at the skill level with both mean and variance statistics.

\vspace{1mm}
\noindent\textbf{Regularization variants.}
Finally, we compare different spectral regularization strategies.
\emph{Hard Skill} assigns each sample to the most likely skill, while
\emph{Soft Skill} uses the full posterior distribution $\boldsymbol{\pi}(q)$.
\emph{Fixed Weight} applies a non-adaptive spectral weight, and
\emph{Confidence Weight} uses the proposed prototype-distance-based adaptive weight.
As shown in Fig.~\ref{fig:design-ablation}(c), hard skill assignment is the weakest because ambiguous questions may involve multiple skills.
Soft skill weighting improves performance, while fixed weighting still over-regularizes samples whose spectra are far from existing prototypes.
The proposed confidence-adaptive weighting obtains the best results, reducing VQA AF to 2.4 and improving CoIN Avg to 60.0.

\subsection{Mechanism Analysis}

\subsubsection{Do different skills exhibit distinct attention spectra}
To test whether ASR learns skill-specific attention patterns, we analyze four representative skills (\texttt{count}, \texttt{read}, \texttt{locate}, \texttt{relation}) on the last stage of the VQA v2 stream.
For each skill we aggregate cross-attention maps into a 6D spectral summary (low/mid/high frequency energy and anisotropy along three orientation bands) for both the Baseline and ASR.
As shown in Fig.~\ref{fig:skill-spectra}, ASR yields more clearly separated skill-wise spectra, while the Baseline exhibits more entangled patterns, indicating that our method sharpens skill-conditioned attention priors instead of applying a uniform regularization.

\begin{figure}[ht]
\centering
\includegraphics[width=0.9\linewidth]{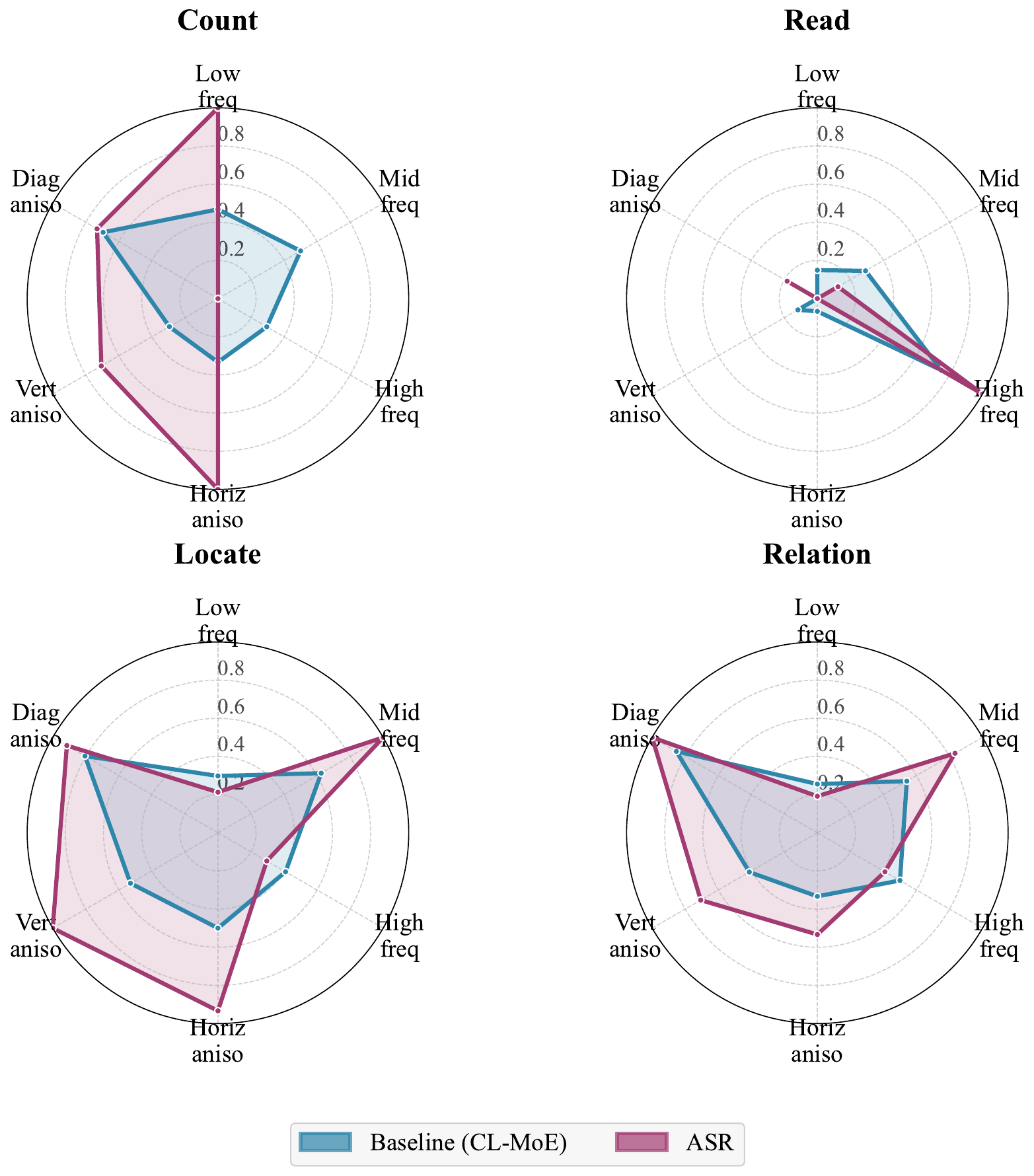}
\vspace{-2mm}
\caption{\textbf{Skill-wise attention spectra.}
ASR produces more distinct spectral signatures for different skills.}
\label{fig:skill-spectra}
\end{figure}

\subsubsection{Does stabilizing spectra correlate with reduced forgetting}
For each skill $s \in \{\texttt{count},\texttt{read},\texttt{locate},\texttt{relation}\}$ and stage $t>1$ in the VQA v2 stream, we measure a spectral prototype drift $\Delta_s^{(t)}=\|\boldsymbol{\mu}_s^{(t)}-\boldsymbol{\mu}_s^{(t-1)}\|_2$ and a skill-wise forgetting score $F_s^{(t)}$ (drop from the best historical accuracy on skill-$s$ questions to the accuracy after stage $t$).
From Fig.~\ref{fig:drift-forgetting}, the baseline exhibits a clear trend where larger prototype drift coincides with higher forgetting, whereas ASR both reduces the typical drift magnitudes and achieves consistently lower $F_s^{(t)}$ at comparable drift levels. 
It shows quantitative evidence that stabilizing skill-conditioned attention spectra is closely tied to mitigating catastrophic forgetting in continual multimodal learning.

\begin{figure}[ht]
\centering
\includegraphics[width=0.95\linewidth]{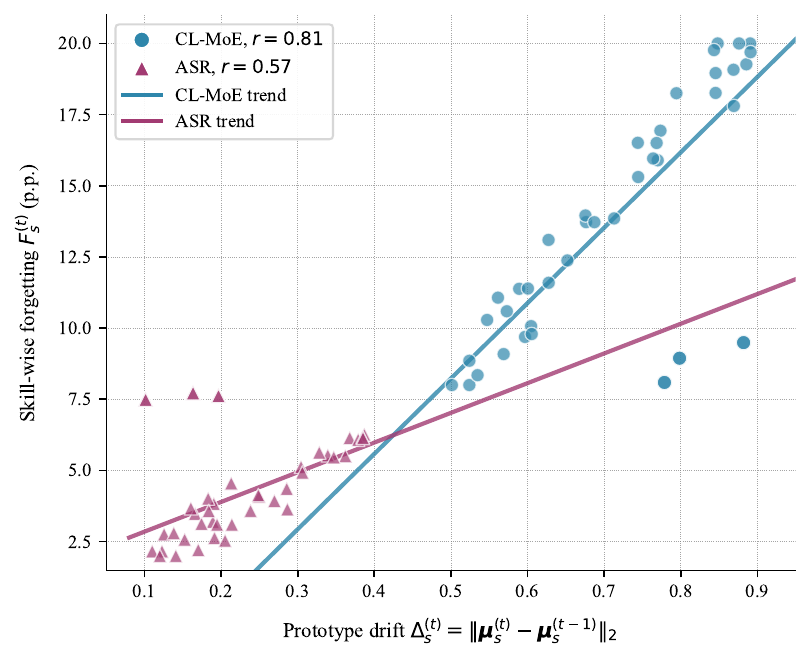}
\vspace{-2mm}
\caption{\textbf{Prototype spectral drift vs.\ forgetting.}
Each point is a skill–stage pair $(s,t)$.}
\label{fig:drift-forgetting}
\end{figure}

\subsubsection{Are correct predictions closer to spectral prototypes}
On the fifth task of the VQA v2 stream, we compute for each sample the squared Mahalanobis distance $D_{\mathrm{mah}}^2(\boldsymbol{\phi}(A),\boldsymbol{\mu}_s,\boldsymbol{\Sigma}_s)$ to its skill prototype and group instances by prediction outcome (correct vs.\ error) for Full ASR and w/o Spectrum Distillation.
From Fig.~\ref{fig:mahal-violin}, ASR produces markedly lower distances for correct than for error cases, with a clear separation between the two groups, whereas removing spectral distillation shifts all distributions toward larger, more overlapping distances.
This supports the view that ASR tightens skill-conditioned spectral manifolds and makes out-of-prototype spectra more indicative of mistakes.

\begin{figure}[ht]
\centering
\includegraphics[width=0.9\linewidth]{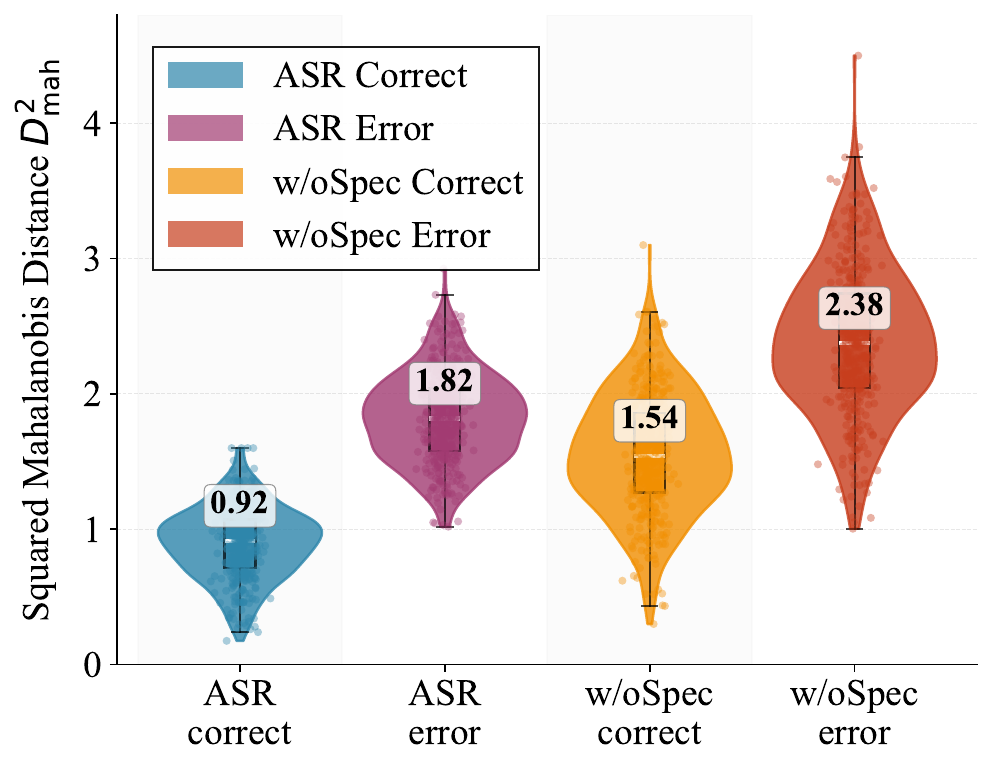}
\vspace{-2mm}
\caption{\textbf{Mahalanobis distance to skill prototypes for correct vs.\ error predictions.}
}
\label{fig:mahal-violin}
\end{figure}

\subsubsection{Task-Level Interference Structure}
We next examine how ASR reshapes task-level interference in the
VQA v2 10-task stream. We visualize the full task$\times$task structure to see which tasks interfere with
which others, and how this pattern changes under ASR.
Recall that $m_{a,b}$ denotes the performance on task $b$ after
training up to stage $a$ ($a\ge b$ for seen tasks). For each method
we construct a $10\times10$ matrix $M = [m_{a,b}]_{a,b=1}^{10}$ for
the VQA v2 stream. To isolate \emph{forgetting} on each task $b$, we
normalize column-wise by the best historical performance on that task:
\begin{equation}
  m_b^{\max}
  = \max_{a \ge b} m_{a,b},
  \qquad
  \Delta_{a,b}
  = m_{a,b} - m_b^{\max}.
\end{equation}
Thus, $\Delta_{a,b}=0$ for the stage(s) where task $b$ attains its
best performance, while $\Delta_{a,b}<0$ quantifies the amount of
forgetting at later stages. We restrict our analysis to the lower
triangular part ($a\ge b$), since tasks $b>a$ have not been seen at
stage $a$.

Fig.~\ref{fig:task-interference} shows the resulting task$\times$task heatmaps for CL-MoE and ASR.
Under CL-MoE, the interference map shows clear negative bands in the
lower-left region: early tasks (columns 1--3) gradually accumulate
2--3 points of forgetting as more tasks arrive, and several mid-stream
tasks also experience noticeable drops after later stages. In
contrast, ASR produces a much lighter map: most $\Delta_{a,b}$ values
remain within a narrow band around zero, and deep negative streaks
are largely absent. This confirms that ASR does not merely improve the
final average AP, but genuinely reshapes the task-level interference
structure, making forgetting more uniformly mild across the stream
and especially alleviating interference on early skills.

\begin{figure}[t]
\centering
\includegraphics[width=1\linewidth]{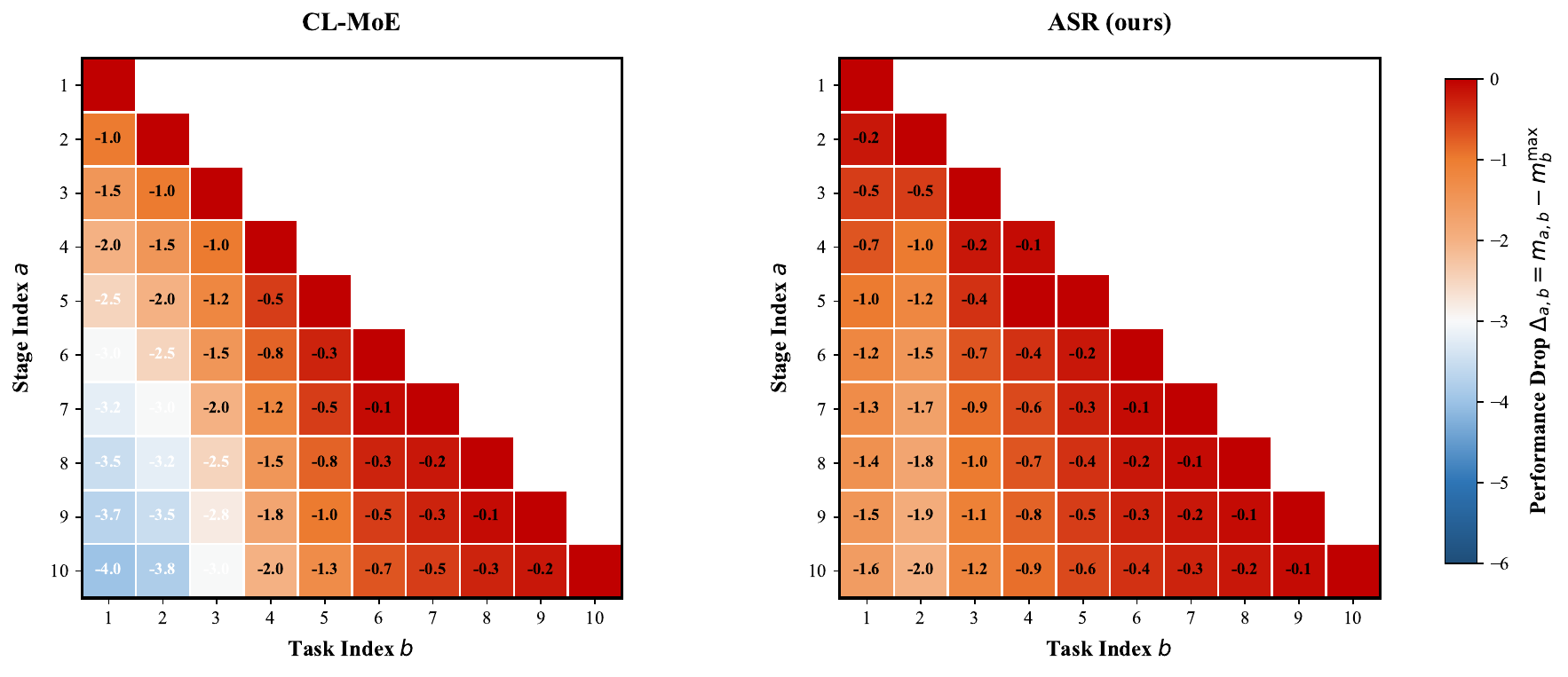}
\vspace{-2mm}
\caption{\textbf{Task$\times$task interference on VQA v2 (10-task
stream).} Each heatmap shows
$\Delta_{a,b} = m_{a,b} - m_b^{\max}$ for $a\ge b$, i.e., the drop
from best historical performance on task $b$ at stage $a$, the upper
triangle ($a<b$) is masked. Darker colors indicate stronger
forgetting.}
\label{fig:task-interference}
\end{figure}

\begin{figure}[t]
	\centering
	\includegraphics[width=\linewidth]{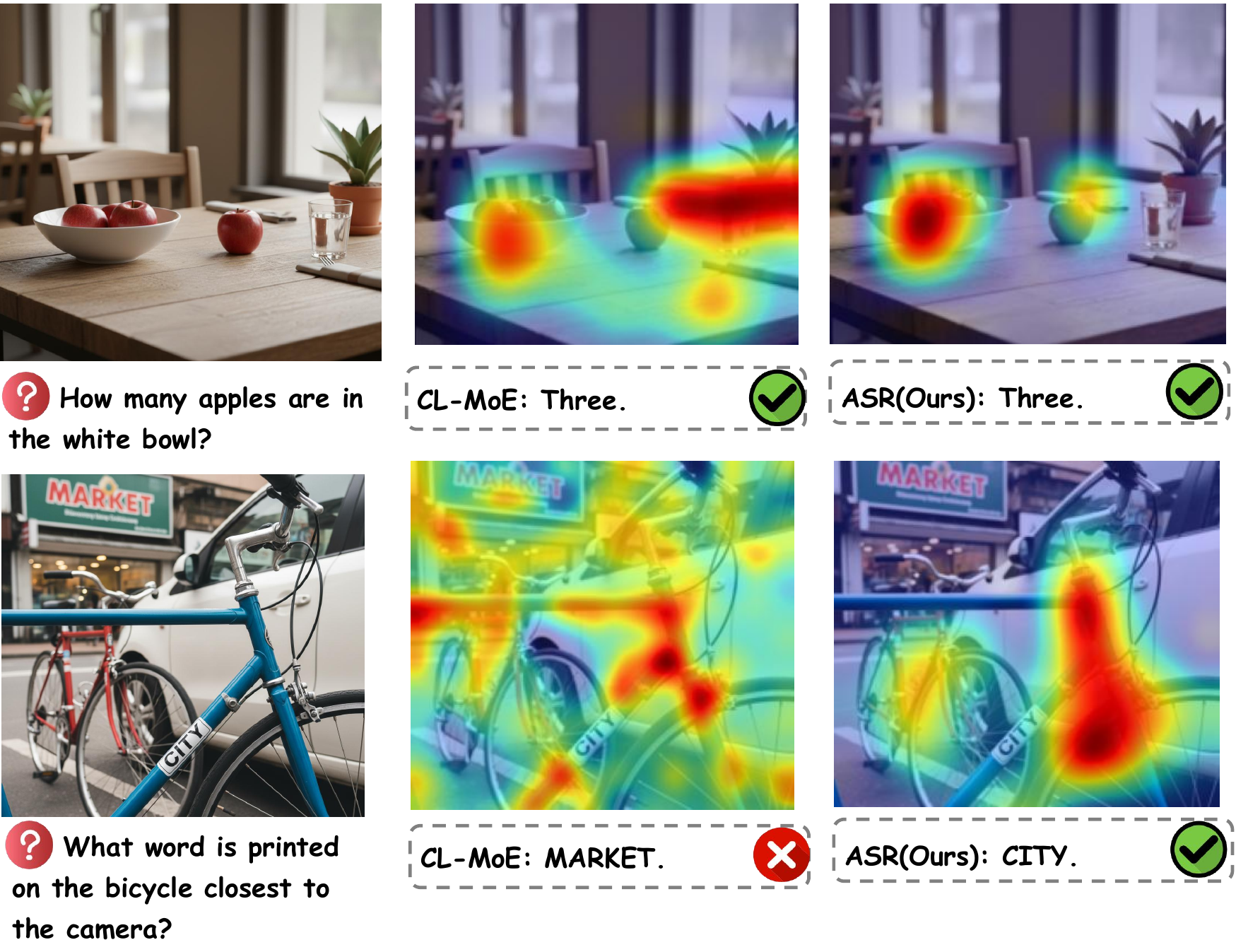}
  \vspace{-2mm}
	\caption{\textbf{Case study}.}
	\label{fig:vis}
\end{figure}

\subsubsection{Qualitative results} 
Figure~\ref{fig:vis} illustrates how ASR reshapes cross-attention.
In the counting case (top row), both methods predict the correct answer “Three”, but ASR focuses its attention almost exclusively on the apples in the white bowl, whereas CL-MoE spreads attention across the table.
In the reading case (bottom row), CL-MoE is distracted by the large “MARKET” sign in the background, while ASR concentrates on the small “CITY” text on the frame of the bicycle and answers correctly.

\section{Conclusion}
We addressed catastrophic forgetting in MLLMs by shifting the focus from preserving all features to preserving skill-conditioned patterns of visual attention, and proposed Attention-Spectrum Regularization (ASR) as a replay-free spectral regularizer that stabilizes cross-modal attention across tasks. Our experiments show that constraining attention spectra via compact, per-skill prototypes yields consistent gains over strong replay- and MoE-based baselines.
Future work may explore fully self-supervised skill discovery, and apply spectral attention regularization to broader continual settings such as open-world detection, grounding, and embodied agents.

\bibliographystyle{unsrtnat}
\bibliography{main}

\clearpage
\appendix

\section{Architecture-Agnostic Cross-Modal Attention Extraction}
\label{app:attention_extraction}

This appendix specifies how ASR extracts the unified cross-modal attention map
$A^{(l,h)}\in\mathbb{R}^{T_q\times H\times W}$ used in
Section~\ref{subsec:problem} and Section~\ref{subsec:spectrum}. The key point
is that ASR does not require every MLLM backbone to contain explicit
encoder--decoder cross-attention. For decoder-only MLLMs, ASR uses the
text-to-vision block of self-attention; for models with explicit
cross-attention, ASR uses the native cross-attention matrix.

\subsection{Token Index Sets}

After the multimodal input is packed by a backbone, let
$\mathcal{V}$ denote the set of visual-token indices and let
$\mathcal{T}$ denote the set of question/instruction-side text-token indices.
Special tokens and answer tokens are excluded from $\mathcal{T}$ when they are
not part of the input question or instruction. We also keep the visual-token
layout produced by the vision encoder or multimodal projector, so that each
visual token can be mapped back to a spatial position.

\subsection{Decoder-Only MLLMs}

For decoder-only MLLMs, such as LLaVA-style, Qwen-VL-style, and InternVL-style
models, visual tokens and text tokens are inserted into a single language-model
context. For a selected decoder block $l$ and head $h$, the self-attention
matrix is
\begin{equation}
\mathbf{P}^{(l,h)}
=
\operatorname{softmax}
\left(
\frac{\mathbf{Q}^{(l,h)}(\mathbf{K}^{(l,h)})^\top}{\sqrt{d_h}}
+
\mathbf{M}
\right),
\end{equation}
where $\mathbf{M}$ is the causal or backbone-specific attention mask. ASR
extracts the text-to-vision submatrix
\begin{equation}
\mathbf{B}^{(l,h)}
=
\mathbf{P}^{(l,h)}[\mathcal{T},\mathcal{V}]
\in
\mathbb{R}^{|\mathcal{T}|\times|\mathcal{V}|}.
\label{eq:app_self_attention_block}
\end{equation}
The row $\mathbf{B}^{(l,h)}_{\tau,:}$ represents how text token $\tau$ attends
to visual tokens in block $l$ and head $h$. This is the cross-modal attention
signal used by ASR for decoder-only MLLMs.

\subsection{Backbones with Explicit Cross-Attention}

For MLLMs that contain explicit cross-attention modules, such as Q-Former-style
or encoder--decoder fusion modules, ASR directly uses the image--text
cross-attention matrix:
\begin{equation}
\mathbf{B}^{(l,h)}
=
\mathbf{P}_{\mathrm{cross}}^{(l,h)}
\in
\mathbb{R}^{|\mathcal{T}|\times|\mathcal{V}|}.
\label{eq:app_explicit_cross_attention}
\end{equation}
Thus, Eq.~\eqref{eq:app_self_attention_block} and
Eq.~\eqref{eq:app_explicit_cross_attention} define the same object: a
token-level text-to-vision attention distribution.

\subsection{Visual-Grid Reconstruction}

To compute spectral statistics, ASR converts the token-level attention vector
into a two-dimensional visual map. Let $\Omega_{x,y}\subseteq\mathcal{V}$ be
the set of visual tokens assigned to grid cell $(x,y)$. For a text token
$\tau\in\mathcal{T}$, we define
\begin{equation}
A_{\tau}^{(l,h)}(x,y)
=
\frac{1}{|\Omega_{x,y}|}
\sum_{v\in\Omega_{x,y}}
\mathbf{B}^{(l,h)}_{\tau,v},
\qquad
(x,y)\in[H]\times[W].
\label{eq:app_grid_reconstruction}
\end{equation}
For standard patch-based visual encoders, $\Omega_{x,y}$ usually contains one
patch token. For backbones with patch merging, tiling, or dynamic-resolution
packing, $\Omega_{x,y}$ is determined by the backbone-specific visual-token
layout. If different samples produce different grid sizes, the attention maps
are resized to a common resolution before spectral encoding. Each map is then
renormalized as
\begin{equation}
A_{\tau}^{(l,h)}
\leftarrow
\frac{
A_{\tau}^{(l,h)}
}{
\sum_{x,y} A_{\tau}^{(l,h)}(x,y)+\varepsilon
}.
\label{eq:app_attention_normalization}
\end{equation}

\subsection{Functional-Token Selection}

ASR aggregates attention maps over a small set of question-side functional
tokens $U(q)$. In our implementation, $U(q)$ includes wh-words, verbs, head
nouns, relation words, numerals, and OCR-related tokens. If no token is
selected by this rule, we use all non-special question/instruction tokens. The
aggregated map is
\begin{equation}
\tilde{A}^{(l,h)}(x,y)
=
\frac{1}{|U(q)|}
\sum_{\tau\in U(q)}
A_{\tau}^{(l,h)}(x,y).
\label{eq:app_functional_token_aggregation}
\end{equation}
This aggregation is performed before the Fourier transform in
Section~\ref{subsec:spectrum}.

\subsection{Backbone-Specific Summary}

\begin{table}[t]
\centering
\caption{
Backbone-specific attention sources used by ASR. Decoder-only MLLMs do not
need explicit cross-attention layers; ASR extracts the text-to-vision block
from decoder self-attention after visual tokens are inserted into the
language-model context.
}
\label{tab:app_attention_sources}
\setlength{\tabcolsep}{3.5pt}
\renewcommand{\arraystretch}{1.10}
\begin{tabular}{p{2.25cm}p{2.85cm}p{2.75cm}}
\toprule
Backbone family & Attention source & ASR map \\
\midrule
LLaVA-style models
&
Decoder self-attention over inserted visual tokens and text tokens
&
Text-token rows to visual-token columns,
$\mathbf{P}^{(l,h)}[\mathcal{T},\mathcal{V}]$
\\
Qwen-VL-style models
&
Decoder self-attention over dynamically packed visual and text tokens
&
Text-to-vision submatrix reconstructed with the visual-token layout
\\
InternVL-style models
&
Decoder self-attention over visual-context tokens and text tokens
&
Text-to-vision submatrix reconstructed with the visual-token layout
\\
Q-Former / encoder--decoder models
&
Explicit image--text cross-attention
&
Native cross-attention matrix
$\mathbf{P}_{\mathrm{cross}}^{(l,h)}$
\\
\bottomrule
\end{tabular}
\end{table}

The selected layer--head set
$\mathcal{J}=\mathcal{L}_\mathrm{sel}\times\mathcal{H}_\mathrm{sel}$ is fixed
for each backbone across all continual-learning stages. For decoder-only
MLLMs, we select middle-to-late decoder blocks where visual and linguistic
tokens have undergone sufficient interaction. For backbones with explicit
cross-attention, we select the corresponding multimodal fusion blocks. After
this extraction step, all backbones provide the same object
$A^{(l,h)}\in\mathbb{R}^{T_q\times H\times W}$, so the subsequent spectral
encoder is architecture-agnostic.

\section{Full Proofs for the Skill-Wise Spectral Forgetting Bound}
\label{app:spectral_forgetting}

For stage $m$, let $\cP_m$ be the distribution of triples
$(X,Y,S)$, where $X=(I,q)$ and $S\in\cS$ is the reasoning skill.
For $s\in\cS$, define the conditional old-stage risk
\begin{equation}
  R_{m,s}(\theta)
  :=
  \E_{\cP_m}
  \left[
    \ell(f_{\theta}(X),Y)
    \,\middle|\,
    S=s
  \right].
  \label{eq:app_conditional_risk}
\end{equation}
The full old-stage risk decomposes as
\begin{equation}
  R_m(\theta)
  =
  \sum_{s\in\cS}
  \omega_{m,s} R_{m,s}(\theta),
  \qquad
  \omega_{m,s}:=\Pr_{\cP_m}(S=s).
  \label{eq:app_risk_decomposition}
\end{equation}
For any model parameter $\theta$, the skill-conditioned spectral
descriptor law is
\begin{equation}
  Q_{m,s}^{\theta}
  =
  \law
  \left(
    z_{\theta}(X)
    \,\middle|\,
    (X,Y,S)\sim \cP_m,\ S=s
  \right).
  \label{eq:app_descriptor_law}
\end{equation}
The ASR memory for skill $s$ at stage $m$ is
$(\mu_{m,s},\Sigma_{m,s})$ with $\Sigma_{m,s}\succ 0$.  The
corresponding ground metric is
\begin{equation}
  d_{m,s}(z,z')
  =
  \left\|
    \Sigma_{m,s}^{-1/2}(z-z')
  \right\|_2 .
  \label{eq:app_ground_metric}
\end{equation}
For $p\in\{1,2\}$ and probability measures $P,Q$ on $\R^D$, the
Wasserstein distance under $d_{m,s}$ is
\begin{equation}
  \W_{p,m,s}(P,Q)
  :=
  \inf_{\gamma\in\Pi(P,Q)}
  \left(
    \E_{(U,V)\sim\gamma}
    \left[
      d_{m,s}(U,V)^p
    \right]
  \right)^{1/p},
  \label{eq:app_wasserstein_def}
\end{equation}
where $\Pi(P,Q)$ is the set of all couplings of $P$ and $Q$.

For any descriptor law $Q$, let $\cG(Q)$ be the Gaussian distribution
with the same mean and covariance as represented in the ASR spectral
memory.  We use the shorthand
\begin{equation}
  \cG_{m,s}^{\theta}
  :=
  \cG(Q_{m,s}^{\theta}).
  \label{eq:app_gaussianization_def}
\end{equation}
The non-Gaussianity term and the Gaussian spectral drift are
\begin{equation}
\begin{aligned}
  \Gamma_{m,s}^{\theta}
  &:=
  \W_{2,m,s}
  \left(
    Q_{m,s}^{\theta},
    \cG_{m,s}^{\theta}
  \right),
  \\
  \Delta_{m,t,s}
  &:=
  \W_{2,m,s}
  \left(
    \cG_{m,s}^{\theta_t},
    \cG_{m,s}^{\theta_m}
  \right).
\end{aligned}
\label{eq:app_gamma_delta_def}
\end{equation}

\subsection{Auxiliary Lemmas}
\label{app:spectral_forgetting_lemmas}

\begin{lemma}[Risk reduction to spectral transport]
\label{lem:risk_to_transport}
Suppose Assumption~\ref{ass:main_spectral_sufficiency} holds for
stage $m$ and skill $s$.  Then
\begin{equation}
\begin{aligned}
  \left|
    R_{m,s}(\theta_t)-R_{m,s}(\theta_m)
  \right|
  &\le
  \left|
    \E_{Z\sim Q_{m,s}^{\theta_t}} \left[ g_{m,s}(Z) \right]
    \right. \\
  &\qquad \left.
    - \E_{Z\sim Q_{m,s}^{\theta_m}} \left[ g_{m,s}(Z) \right]
  \right| \\
  &\quad + 2\epsilon_{m,s}.
\end{aligned}
\label{eq:app_risk_reduction_statement}
\end{equation}
\end{lemma}

\begin{proof}
By adding and subtracting the spectral surrogate expectations, we get
\begin{equation}
\begin{aligned}
  &R_{m,s}(\theta_t)-R_{m,s}(\theta_m)
  \\
  &=
  \left[
    R_{m,s}(\theta_t)
    -
    \E_{Z\sim Q_{m,s}^{\theta_t}}
    g_{m,s}(Z)
  \right]
  \\
  &\quad+
  \left[
    \E_{Z\sim Q_{m,s}^{\theta_t}}
    g_{m,s}(Z)
    -
    \E_{Z\sim Q_{m,s}^{\theta_m}}
    g_{m,s}(Z)
  \right]
  \\
  &\quad+
  \left[
    \E_{Z\sim Q_{m,s}^{\theta_m}}
    g_{m,s}(Z)
    -
    R_{m,s}(\theta_m)
  \right].
\end{aligned}
\label{eq:app_risk_decomposition_chain}
\end{equation}
Taking absolute values and applying the triangle inequality yields
\begin{equation}
\begin{aligned}
  &\left|
    R_{m,s}(\theta_t)-R_{m,s}(\theta_m)
  \right|
  \\
  &\le
  \left|
    R_{m,s}(\theta_t)
    -
    \E_{Z\sim Q_{m,s}^{\theta_t}}
    g_{m,s}(Z)
  \right|
  \\
  &\quad+
  \left|
    \E_{Z\sim Q_{m,s}^{\theta_t}}
    g_{m,s}(Z)
    -
    \E_{Z\sim Q_{m,s}^{\theta_m}}
    g_{m,s}(Z)
  \right|
  \\
  &\quad+
  \left|
    \E_{Z\sim Q_{m,s}^{\theta_m}}
    g_{m,s}(Z)
    -
    R_{m,s}(\theta_m)
  \right|.
\end{aligned}
\label{eq:app_risk_triangle_chain}
\end{equation}
Assumption~\ref{ass:main_spectral_sufficiency} bounds the first and
third terms by $\epsilon_{m,s}$.  Therefore,
\begin{equation}
\begin{aligned}
  \left|
    R_{m,s}(\theta_t)-R_{m,s}(\theta_m)
  \right|
  &\le
  \epsilon_{m,s} \\
  &\quad +
  \left|
    \E_{Z\sim Q_{m,s}^{\theta_t}} g_{m,s}(Z)
    \right. \\
  &\qquad \left.
    - \E_{Z\sim Q_{m,s}^{\theta_m}} g_{m,s}(Z)
  \right| \\
  &\quad + \epsilon_{m,s},
\end{aligned}
\label{eq:app_risk_reduction_final}
\end{equation}
which proves Eq.~\eqref{eq:app_risk_reduction_statement}.
\end{proof}

\begin{lemma}[Lipschitz test functions are controlled by Wasserstein drift]
\label{lem:lipschitz_wasserstein}
Let $g:\R^D\to\R$ satisfy
$|g(z)-g(z')|\le Ld_{m,s}(z,z')$ for all $z,z'$.  Then, for any
probability measures $P,Q$ with finite first moments,
\begin{equation}
  \left|
    \E_{Z\sim P}g(Z)
    -
    \E_{Z\sim Q}g(Z)
  \right|
  \le
  L\W_{1,m,s}(P,Q).
  \label{eq:app_lip_wasserstein_statement}
\end{equation}
\end{lemma}

\begin{proof}
Let $\gamma\in\Pi(P,Q)$ be any coupling.  If $(U,V)\sim\gamma$, then
$U\sim P$ and $V\sim Q$.  Hence
\begin{equation}
\begin{aligned}
  &\left|
    \E_{Z\sim P}g(Z)
    -
    \E_{Z\sim Q}g(Z)
  \right|
  \\
  &=
  \left|
    \E_{(U,V)\sim\gamma}
    \left[
      g(U)-g(V)
    \right]
  \right|
  \\
  &\le
  \E_{(U,V)\sim\gamma}
  \left[
    |g(U)-g(V)|
  \right]
  \\
  &\le
  L
  \E_{(U,V)\sim\gamma}
  \left[
    d_{m,s}(U,V)
  \right].
\end{aligned}
\label{eq:app_lip_coupling_chain}
\end{equation}
Taking the infimum over all couplings $\gamma\in\Pi(P,Q)$ gives
\begin{equation}
\begin{aligned}
  \left|
    \E_{Z\sim P}g(Z)
    -
    \E_{Z\sim Q}g(Z)
  \right|
  &\le
  L
  \inf_{\gamma\in\Pi(P,Q)}
  \E_{(U,V)\sim\gamma}
  \left[
    d_{m,s}(U,V)
  \right]
  \\
  &=
  L\W_{1,m,s}(P,Q).
\end{aligned}
\label{eq:app_lip_wasserstein_final}
\end{equation}
\end{proof}

\begin{lemma}[$\W_1$--$\W_2$ comparison and Gaussian insertion]
\label{lem:w1_w2_gaussian_insert}
For any descriptor laws
$Q_{m,s}^{\theta_t}$ and $Q_{m,s}^{\theta_m}$,
\begin{equation}
\begin{aligned}
  \W_{1,m,s}
  \left(
    Q_{m,s}^{\theta_t},
    Q_{m,s}^{\theta_m}
  \right)
  &\le
  \W_{2,m,s}
  \left(
    Q_{m,s}^{\theta_t},
    Q_{m,s}^{\theta_m}
  \right)
  \\
  &\le
  \Gamma_{m,s}^{\theta_t}
  +
  \Delta_{m,t,s}
  +
  \Gamma_{m,s}^{\theta_m}.
\end{aligned}
\label{eq:app_w1_w2_gaussian_statement}
\end{equation}
\end{lemma}

\begin{proof}
We first prove $\W_1\le\W_2$.  For any coupling
$\gamma\in\Pi(P,Q)$, Jensen's inequality gives
\begin{equation}
\begin{aligned}
  \E_{\gamma}
  \left[
    d_{m,s}(U,V)
  \right]
  &\le
  \left(
    \E_{\gamma}
    \left[
      d_{m,s}(U,V)^2
    \right]
  \right)^{1/2}.
\end{aligned}
\label{eq:app_jensen_w1_w2}
\end{equation}
Taking the infimum on the left over all couplings and then using
that the same inequality holds for every coupling gives
\begin{equation}
\begin{aligned}
  \W_{1,m,s}(P,Q)
  &=
  \inf_{\gamma\in\Pi(P,Q)}
  \E_{\gamma}
  \left[
    d_{m,s}(U,V)
  \right]
  \\
  &\le
  \inf_{\gamma\in\Pi(P,Q)}
  \left(
    \E_{\gamma}
    \left[
      d_{m,s}(U,V)^2
    \right]
  \right)^{1/2}
  \\
  &=
  \W_{2,m,s}(P,Q).
\end{aligned}
\label{eq:app_w1_le_w2}
\end{equation}

We now prove the Gaussian insertion inequality.  For brevity, set
\begin{equation}
\begin{aligned}
  P_0 &:= Q_{m,s}^{\theta_t},
  &
  P_1 &:= \cG_{m,s}^{\theta_t},
  \\
  P_2 &:= \cG_{m,s}^{\theta_m},
  &
  P_3 &:= Q_{m,s}^{\theta_m}.
\end{aligned}
\label{eq:app_insert_short_notation}
\end{equation}
Let $\gamma_{01}$, $\gamma_{12}$, and $\gamma_{23}$ be
$\eta$-optimal couplings for the three adjacent pairs.  By the
gluing lemma, there exists a joint distribution over
$(Z_0,Z_1,Z_2,Z_3)$ whose adjacent marginals are these couplings.
Using the triangle inequality for $d_{m,s}$ and then Minkowski's
inequality, we obtain
\begin{equation}
\begin{aligned}
  &\left(
    \E
    \left[
      d_{m,s}(Z_0,Z_3)^2
    \right]
  \right)^{1/2}
  \\
  &\le
  \left(
    \E
    \left[
      \left(
        d_{m,s}(Z_0,Z_1)
        +
        d_{m,s}(Z_1,Z_2)
        +
        d_{m,s}(Z_2,Z_3)
      \right)^2
    \right]
  \right)^{1/2}
  \\
  &\le
  \left(
    \E
    \left[
      d_{m,s}(Z_0,Z_1)^2
    \right]
  \right)^{1/2}
  +
  \left(
    \E
    \left[
      d_{m,s}(Z_1,Z_2)^2
    \right]
  \right)^{1/2}
  \\
  &\quad+
  \left(
    \E
    \left[
      d_{m,s}(Z_2,Z_3)^2
    \right]
  \right)^{1/2}.
\end{aligned}
\label{eq:app_minkowski_insert_chain}
\end{equation}
Since the joint law couples $P_0$ and $P_3$, the definition of
$\W_2$ gives
\begin{equation}
\begin{aligned}
  \W_{2,m,s}(P_0,P_3)
  &\le
  \left(
    \E
    \left[
      d_{m,s}(Z_0,Z_3)^2
    \right]
  \right)^{1/2} \\
  &\le
  \W_{2,m,s}(P_0,P_1)
  + \W_{2,m,s}(P_1,P_2) \\
  &\qquad + \W_{2,m,s}(P_2,P_3)
  + 3\eta.
\end{aligned}
\label{eq:app_insert_eta_bound}
\end{equation}
Letting $\eta\downarrow 0$ and substituting the definitions of
$\Gamma_{m,s}^{\theta_t}$, $\Delta_{m,t,s}$, and
$\Gamma_{m,s}^{\theta_m}$ proves
\begin{equation}
\begin{aligned}
  \W_{2,m,s}
  \left(
    Q_{m,s}^{\theta_t},
    Q_{m,s}^{\theta_m}
  \right)
  &\le
  \Gamma_{m,s}^{\theta_t}
  +
  \Delta_{m,t,s}
  +
  \Gamma_{m,s}^{\theta_m}.
\end{aligned}
\label{eq:app_gaussian_insert_final}
\end{equation}
Combining Eq.~\eqref{eq:app_w1_le_w2} and
Eq.~\eqref{eq:app_gaussian_insert_final} proves the claim.
\end{proof}

\begin{lemma}[Exact Gaussian formula in the ASR whitened metric]
\label{lem:gaussian_w2_formula}
Let
$G_0=\mathcal{N}(\mu_0,\Sigma_0)$ and
$G_1=\mathcal{N}(\mu_1,\Sigma_1)$ with $\Sigma_0\succ0$ and
$\Sigma_1\succ0$.  Under the metric
$d_0(z,z')=\|\Sigma_0^{-1/2}(z-z')\|_2$,
\begin{equation}
\begin{aligned}
  \W_{2,0}^2(G_1,G_0)
  &=
  \left\|
    \Sigma_0^{-1/2}(\mu_1-\mu_0)
  \right\|_2^2
  \\
  &\quad+
  \Tr(C_1)
  +
  D
  -
  2\Tr(C_1^{1/2}),
\end{aligned}
\label{eq:app_gaussian_w2_formula_trace}
\end{equation}
where $C_1=\Sigma_0^{-1/2}\Sigma_1\Sigma_0^{-1/2}$.  Equivalently,
\begin{equation}
  \W_{2,0}^2(G_1,G_0)
  =
  \left\|
    \Sigma_0^{-1/2}(\mu_1-\mu_0)
  \right\|_2^2
  +
  \left\|
    C_1^{1/2}-I_D
  \right\|_F^2 .
  \label{eq:app_gaussian_w2_formula_fro}
\end{equation}
\end{lemma}

\begin{proof}
Apply the affine whitening map
$T(z)=\Sigma_0^{-1/2}(z-\mu_0)$.  Under this map,
$G_0$ becomes $\widetilde{G}_0=\mathcal{N}(0,I_D)$ and $G_1$
becomes
$\widetilde{G}_1=\mathcal{N}(\delta,C_1)$, where
\begin{equation}
  \delta
  :=
  \Sigma_0^{-1/2}(\mu_1-\mu_0),
  \qquad
  C_1
  :=
  \Sigma_0^{-1/2}\Sigma_1\Sigma_0^{-1/2}.
  \label{eq:app_whitened_gaussian_params}
\end{equation}
Therefore,
\begin{equation}
  \W_{2,0}^2(G_1,G_0)
  =
  \W_2^2
  \left(
    \mathcal{N}(\delta,C_1),
    \mathcal{N}(0,I_D)
  \right),
  \label{eq:app_metric_to_euclidean_w2}
\end{equation}
where the right-hand side uses the Euclidean metric.

Let $X\sim\mathcal{N}(0,I_D)$ and
$Y\sim\mathcal{N}(\delta,C_1)$ be any coupling.  Write
$\bar{Y}=Y-\delta$.  Then $\E X=0$, $\E \bar{Y}=0$,
$\Cov(X)=I_D$, and $\Cov(\bar{Y})=C_1$.  Expanding the squared cost,
we obtain
\begin{equation}
\begin{aligned}
  \E
  \left[
    \|X-Y\|_2^2
  \right]
  &=
  \E
  \left[
    \|X-\bar{Y}-\delta\|_2^2
  \right]
  \\
  &=
  \|\delta\|_2^2
  +
  \E
  \left[
    \|X-\bar{Y}\|_2^2
  \right]
  \\
  &=
  \|\delta\|_2^2
  +
  \Tr(I_D)
  +
  \Tr(C_1)
  -
  2\Tr(K),
\end{aligned}
\label{eq:app_gaussian_cost_expand}
\end{equation}
where $K:=\E[X\bar{Y}^{\top}]$ is the cross-covariance.  The joint
covariance matrix
\begin{equation}
  M
  :=
  \begin{pmatrix}
    I_D & K \\
    K^{\top} & C_1
  \end{pmatrix}
  \label{eq:app_joint_covariance_matrix}
\end{equation}
must be positive semidefinite.  Since the upper-left block is
$I_D$, the Schur complement implies
\begin{equation}
  C_1-K^{\top}K\succeq 0 .
  \label{eq:app_schur_condition}
\end{equation}
Consequently,
\begin{equation}
\begin{aligned}
  \Tr(K)
  &\le
  \|K\|_*
  \\
  &=
  \Tr\!\left(
    (K^{\top}K)^{1/2}
  \right)
  \\
  &\le
  \Tr(C_1^{1/2}).
\end{aligned}
\label{eq:app_cross_cov_trace_bound}
\end{equation}
The first inequality uses $\Tr(K)\le \sum_i\sigma_i(K)$, the second
is the definition of the nuclear norm, and the last follows from
$K^{\top}K\preceq C_1$ and the monotonicity of the matrix square
root.  Substituting Eq.~\eqref{eq:app_cross_cov_trace_bound} into
Eq.~\eqref{eq:app_gaussian_cost_expand} gives the lower bound
\begin{equation}
\begin{aligned}
  \E
  \left[
    \|X-Y\|_2^2
  \right]
  &\ge
  \|\delta\|_2^2
  +
  D
  +
  \Tr(C_1)
  -
  2\Tr(C_1^{1/2}).
\end{aligned}
\label{eq:app_gaussian_lower_bound}
\end{equation}

It remains to show that the lower bound is attainable.  Let
$X\sim\mathcal{N}(0,I_D)$ and define
\begin{equation}
  Y
  :=
  \delta + C_1^{1/2}X .
  \label{eq:app_optimal_gaussian_coupling}
\end{equation}
Then $Y\sim\mathcal{N}(\delta,C_1)$ and the cross-covariance is
$K=C_1^{1/2}$.  Therefore,
\begin{equation}
\begin{aligned}
  \E
  \left[
    \|X-Y\|_2^2
  \right]
  &=
  \|\delta\|_2^2
  +
  D
  +
  \Tr(C_1)
  -
  2\Tr(C_1^{1/2}).
\end{aligned}
\label{eq:app_gaussian_upper_bound}
\end{equation}
Combining the lower and upper bounds proves
Eq.~\eqref{eq:app_gaussian_w2_formula_trace}.  Finally,
\begin{equation}
\begin{aligned}
  \left\|
    C_1^{1/2}-I_D
  \right\|_F^2
  &=
  \Tr(C_1)
  +
  \Tr(I_D)
  -
  2\Tr(C_1^{1/2})
  \\
  &=
  \Tr(C_1)
  +
  D
  -
  2\Tr(C_1^{1/2}),
\end{aligned}
\label{eq:app_trace_to_frobenius}
\end{equation}
which proves Eq.~\eqref{eq:app_gaussian_w2_formula_fro}.
\end{proof}

\begin{lemma}[Perturbation caused by empirical spectral memory]
\label{lem:empirical_memory_perturbation}
Let
$G=\mathcal{N}(\mu,\Sigma)$ be the ideal old spectral prototype and
$\widehat{G}=\mathcal{N}(\widehat{\mu},\widehat{\Sigma})$ be the
stored empirical prototype, where $\Sigma\succ0$ and
$\widehat{\Sigma}\succ0$.  Define
\begin{equation}
\begin{aligned}
  a
  &:=
  \left\|
    \Sigma^{-1/2}(\widehat{\mu}-\mu)
  \right\|_2,
  \\
  B
  &:=
  \Sigma^{-1/2}\widehat{\Sigma}\Sigma^{-1/2},
  \\
  b
  &:=
  \|B-I_D\|_F,
  \qquad
  \rho
  :=
  \|B-I_D\|_{\mathrm{op}} .
\end{aligned}
\label{eq:app_memory_error_def}
\end{equation}
If $\rho<1$, then for any Gaussian
$G_t=\mathcal{N}(\mu_t,\Sigma_t)$,
\begin{equation}
\begin{aligned}
  \W_{2,\Sigma}(G_t,G)
  &\le
  \sqrt{1+\rho}\,
  \W_{2,\widehat{\Sigma}}(G_t,\widehat{G})
  \\
  &\quad+
  \left(
    a^2
    +
    \frac{b^2}{(1+\sqrt{1-\rho})^2}
  \right)^{1/2}.
\end{aligned}
\label{eq:app_empirical_memory_perturbation_statement}
\end{equation}
Here $\W_{2,\Sigma}$ and $\W_{2,\widehat{\Sigma}}$ denote
$2$-Wasserstein distances under the Mahalanobis metrics induced by
$\Sigma$ and $\widehat{\Sigma}$, respectively.
\end{lemma}

\begin{proof}
The triangle inequality for $\W_2$ gives
\begin{equation}
\begin{aligned}
  \W_{2,\Sigma}(G_t,G)
  &\le
  \W_{2,\Sigma}(G_t,\widehat{G})
  +
  \W_{2,\Sigma}(\widehat{G},G).
\end{aligned}
\label{eq:app_memory_triangle}
\end{equation}
We first compare the two ground metrics.  Since
$\rho=\|B-I_D\|_{\mathrm{op}}<1$, all eigenvalues of $B$ lie in
$[1-\rho,1+\rho]$.  Hence $B\preceq (1+\rho)I_D$, which implies
$B^{-1}\succeq (1+\rho)^{-1}I_D$.  For any vector $v\in\R^D$,
\begin{equation}
\begin{aligned}
  \|\Sigma^{-1/2}v\|_2^2
  &=
  v^{\top}\Sigma^{-1}v
  \\
  &\le
  (1+\rho)
  v^{\top}
  \widehat{\Sigma}^{-1}
  v
  \\
  &=
  (1+\rho)
  \|\widehat{\Sigma}^{-1/2}v\|_2^2 .
\end{aligned}
\label{eq:app_metric_comparison_vector}
\end{equation}
Therefore,
\begin{equation}
  d_{\Sigma}(u,v)
  \le
  \sqrt{1+\rho}\,
  d_{\widehat{\Sigma}}(u,v)
  \qquad
  \forall u,v\in\R^D .
  \label{eq:app_metric_comparison_pointwise}
\end{equation}
Applying Eq.~\eqref{eq:app_metric_comparison_pointwise} inside the
definition of $\W_2$ yields
\begin{equation}
\begin{aligned}
  \W_{2,\Sigma}(G_t,\widehat{G})
  &=
  \inf_{\gamma}
  \left(
    \E_{\gamma}
    \left[
      d_{\Sigma}(U,V)^2
    \right]
  \right)^{1/2}
  \\
  &\le
  \sqrt{1+\rho}
  \inf_{\gamma}
  \left(
    \E_{\gamma}
    \left[
      d_{\widehat{\Sigma}}(U,V)^2
    \right]
  \right)^{1/2}
  \\
  &=
  \sqrt{1+\rho}\,
  \W_{2,\widehat{\Sigma}}(G_t,\widehat{G}).
\end{aligned}
\label{eq:app_metric_comparison_w2}
\end{equation}

It remains to bound $\W_{2,\Sigma}(\widehat{G},G)$.  By
Lemma~\ref{lem:gaussian_w2_formula},
\begin{equation}
\begin{aligned}
  \W_{2,\Sigma}^2(\widehat{G},G)
  &=
  \left\|
    \Sigma^{-1/2}(\widehat{\mu}-\mu)
  \right\|_2^2
  +
  \left\|
    B^{1/2}-I_D
  \right\|_F^2
  \\
  &=
  a^2
  +
  \left\|
    B^{1/2}-I_D
  \right\|_F^2 .
\end{aligned}
\label{eq:app_empirical_memory_exact}
\end{equation}
Let $\lambda_1,\ldots,\lambda_D$ be the eigenvalues of $B$.  Since
$\lambda_i\in[1-\rho,1+\rho]$, for each $i$ we have
\begin{equation}
\begin{aligned}
  |\sqrt{\lambda_i}-1|
  &=
  \frac{|\lambda_i-1|}{\sqrt{\lambda_i}+1}
  \\
  &\le
  \frac{|\lambda_i-1|}{1+\sqrt{1-\rho}} .
\end{aligned}
\label{eq:app_sqrt_eigen_bound}
\end{equation}
Squaring and summing over $i$ gives
\begin{equation}
\begin{aligned}
  \left\|
    B^{1/2}-I_D
  \right\|_F^2
  &=
  \sum_{i=1}^D
  (\sqrt{\lambda_i}-1)^2
  \\
  &\le
  \frac{
    \sum_{i=1}^D(\lambda_i-1)^2
  }{
    (1+\sqrt{1-\rho})^2
  }
  \\
  &=
  \frac{\|B-I_D\|_F^2}{
    (1+\sqrt{1-\rho})^2
  }
  \\
  &=
  \frac{b^2}{
    (1+\sqrt{1-\rho})^2
  } .
\end{aligned}
\label{eq:app_matrix_sqrt_fro_bound}
\end{equation}
Combining Eq.~\eqref{eq:app_memory_triangle},
Eq.~\eqref{eq:app_metric_comparison_w2},
Eq.~\eqref{eq:app_empirical_memory_exact}, and
Eq.~\eqref{eq:app_matrix_sqrt_fro_bound} proves
Eq.~\eqref{eq:app_empirical_memory_perturbation_statement}.
\end{proof}

\subsection{Proof of Theorem~\ref{thm:spectral_drift_forgetting}}
\label{app:spectral_forgetting_proof}

We start from the definition of forgetting and remove the positive
part using $[a]_+\le |a|$:
\begin{equation}
\begin{aligned}
  \Fgt_{m\to t}
  &=
  \left[
    R_m(\theta_t)-R_m(\theta_m)
  \right]_+
  \\
  &\le
  \left|
    R_m(\theta_t)-R_m(\theta_m)
  \right| .
\end{aligned}
\label{eq:app_forgetting_abs_start}
\end{equation}
Using the risk decomposition in Eq.~\eqref{eq:app_risk_decomposition},
we get
\begin{equation}
\begin{aligned}
  \left|
    R_m(\theta_t)-R_m(\theta_m)
  \right|
  &=
  \left|
    \sum_{s\in\cS}
    \omega_{m,s}
    \left(
      R_{m,s}(\theta_t)-R_{m,s}(\theta_m)
    \right)
  \right|
  \\
  &\le
  \sum_{s\in\cS}
  \omega_{m,s}
  \left|
    R_{m,s}(\theta_t)-R_{m,s}(\theta_m)
  \right| .
\end{aligned}
\label{eq:app_forgetting_skill_average}
\end{equation}
Applying Lemma~\ref{lem:risk_to_transport} to each skill gives
\begin{equation}
\begin{aligned}
  \Fgt_{m\to t}
  &\le
  \sum_{s\in\cS}
  \omega_{m,s}
  \Bigg(
    \left|
      \E_{Z\sim Q_{m,s}^{\theta_t}} g_{m,s}(Z)
      -
      \E_{Z\sim Q_{m,s}^{\theta_m}} g_{m,s}(Z)
    \right| \\
  &\qquad\qquad
    + 2\epsilon_{m,s}
  \Bigg).
\end{aligned}
\label{eq:app_forgetting_after_sufficiency}
\end{equation}
Since $g_{m,s}$ is $L_{m,s}$-Lipschitz under $d_{m,s}$,
Lemma~\ref{lem:lipschitz_wasserstein} yields
\begin{equation}
\begin{aligned}
  \Fgt_{m\to t}
  &\le
  \sum_{s\in\cS}
  \omega_{m,s}
  \Bigg(
    L_{m,s}
    \W_{1,m,s}
    \left(
      Q_{m,s}^{\theta_t},
      Q_{m,s}^{\theta_m}
    \right)
    +
    2\epsilon_{m,s}
  \Bigg)
  \\
  &=
  \sum_{s\in\cS}
  \omega_{m,s}L_{m,s}
  \W_{1,m,s}
  \left(
    Q_{m,s}^{\theta_t},
    Q_{m,s}^{\theta_m}
  \right)
  +
  2
  \sum_{s\in\cS}
  \omega_{m,s}\epsilon_{m,s}.
\end{aligned}
\label{eq:app_forgetting_w1_final}
\end{equation}
This proves Eq.~\eqref{eq:main_sharp_w1_bound}.

It remains to prove the Gaussian spectral-drift version.  From
Lemma~\ref{lem:w1_w2_gaussian_insert}, for every skill $s$,
\begin{equation}
\begin{aligned}
  \W_{1,m,s}
  \left(
    Q_{m,s}^{\theta_t},
    Q_{m,s}^{\theta_m}
  \right)
  &\le
  \Gamma_{m,s}^{\theta_t}
  +
  \Delta_{m,t,s}
  +
  \Gamma_{m,s}^{\theta_m}.
\end{aligned}
\label{eq:app_w1_to_delta_substitution}
\end{equation}
Substituting Eq.~\eqref{eq:app_w1_to_delta_substitution} into
Eq.~\eqref{eq:app_forgetting_w1_final} gives
\begin{equation}
\begin{aligned}
  \Fgt_{m\to t}
  &\le
  \sum_{s\in\cS}
  \omega_{m,s}L_{m,s}
  \left(
    \Gamma_{m,s}^{\theta_t}
    +
    \Delta_{m,t,s}
    +
    \Gamma_{m,s}^{\theta_m}
  \right)
  \\
  &\quad+
  2
  \sum_{s\in\cS}
  \omega_{m,s}\epsilon_{m,s},
\end{aligned}
\label{eq:app_gaussian_drift_bound_final}
\end{equation}
which proves Eq.~\eqref{eq:main_gaussian_drift_bound}.

Finally, suppose
$\cG_{m,s}^{\theta_m}=\mathcal{N}(\mu_{m,s},\Sigma_{m,s})$ and
$\cG_{m,s}^{\theta_t}=\mathcal{N}(\mu_{t|m,s},\Sigma_{t|m,s})$.
Applying Lemma~\ref{lem:gaussian_w2_formula} with
$\Sigma_0=\Sigma_{m,s}$, $\mu_0=\mu_{m,s}$,
$\Sigma_1=\Sigma_{t|m,s}$, and $\mu_1=\mu_{t|m,s}$ yields
\begin{equation}
\begin{aligned}
  \Delta_{m,t,s}^{2}
  &=
  \left\|
    \Sigma_{m,s}^{-1/2}
    (\mu_{t|m,s}-\mu_{m,s})
  \right\|_2^2
  \\
  &\quad+
  \left\|
    \left(
      \Sigma_{m,s}^{-1/2}
      \Sigma_{t|m,s}
      \Sigma_{m,s}^{-1/2}
    \right)^{1/2}
    -I_D
  \right\|_F^2 .
\end{aligned}
\label{eq:app_closed_form_delta_final}
\end{equation}
This proves Eq.~\eqref{eq:main_closed_form_delta} and completes the
proof of Theorem~\ref{thm:spectral_drift_forgetting}.

\subsection{Empirical Memory Version}
\label{app:spectral_forgetting_empirical_memory}

Theorem~\ref{thm:spectral_drift_forgetting} uses the ideal old
prototype $(\mu_{m,s},\Sigma_{m,s})$.  In ASR, this prototype is
estimated and stored as
$(\widehat{\mu}_{m,s},\widehat{\Sigma}_{m,s})$.  The following
corollary converts the ideal bound into a bound expressed through the
stored memory.

\begin{corollary}[Forgetting bound with stored ASR memory]
\label{cor:empirical_memory_forgetting}
Assume the conditions of Theorem~\ref{thm:spectral_drift_forgetting}.
For every skill $s$, let
$\widehat{\cG}_{m,s}:=
\mathcal{N}(\widehat{\mu}_{m,s},\widehat{\Sigma}_{m,s})$ and define
\begin{equation}
\begin{aligned}
  \widehat{\Delta}_{m,t,s}
  &:=
  \W_{2,\widehat{\Sigma}_{m,s}}
  \left(
    \cG_{m,s}^{\theta_t},
    \widehat{\cG}_{m,s}
  \right),
  \\
  a_{m,s}
  &:=
  \left\|
    \Sigma_{m,s}^{-1/2}
    (\widehat{\mu}_{m,s}-\mu_{m,s})
  \right\|_2,
  \\
  B_{m,s}
  &:=
  \Sigma_{m,s}^{-1/2}
  \widehat{\Sigma}_{m,s}
  \Sigma_{m,s}^{-1/2},
  \\
  b_{m,s}
  &:=
  \|B_{m,s}-I_D\|_F,
  \qquad
  \rho_{m,s}
  :=
  \|B_{m,s}-I_D\|_{\mathrm{op}} .
\end{aligned}
\label{eq:app_empirical_corollary_defs}
\end{equation}
If $\rho_{m,s}<1$ for all $s\in\cS$, then
\begin{equation}
\begin{aligned}
  \Fgt_{m\to t}
  &\le
  \sum_{s\in\cS}
  \omega_{m,s}L_{m,s}
  \Bigg(
    \Gamma_{m,s}^{\theta_t}
    +
    \Gamma_{m,s}^{\theta_m}
    \\
  &\qquad\qquad+
    \sqrt{1+\rho_{m,s}}\,
    \widehat{\Delta}_{m,t,s}
    +
    \xi_{m,s}
  \Bigg)
  \\
  &\quad+
  2
  \sum_{s\in\cS}
  \omega_{m,s}\epsilon_{m,s},
\end{aligned}
\label{eq:app_empirical_memory_bound}
\end{equation}
where the empirical-memory error is
\begin{equation}
  \xi_{m,s}
  :=
  \left(
    a_{m,s}^2
    +
    \frac{b_{m,s}^2}{
      (1+\sqrt{1-\rho_{m,s}})^2
    }
  \right)^{1/2}.
  \label{eq:app_empirical_memory_error}
\end{equation}
\end{corollary}

\begin{proof}
From Theorem~\ref{thm:spectral_drift_forgetting}, it is enough to
upper bound the ideal Gaussian drift $\Delta_{m,t,s}$.  Applying
Lemma~\ref{lem:empirical_memory_perturbation} with
$G_t=\cG_{m,s}^{\theta_t}$,
$G=\cG_{m,s}^{\theta_m}$, and
$\widehat{G}=\widehat{\cG}_{m,s}$ gives
\begin{equation}
\begin{aligned}
  \Delta_{m,t,s}
  &=
  \W_{2,\Sigma_{m,s}}
  \left(
    \cG_{m,s}^{\theta_t},
    \cG_{m,s}^{\theta_m}
  \right)
  \\
  &\le
  \sqrt{1+\rho_{m,s}}\,
  \W_{2,\widehat{\Sigma}_{m,s}}
  \left(
    \cG_{m,s}^{\theta_t},
    \widehat{\cG}_{m,s}
  \right)
  \\
  &\quad+
  \left(
    a_{m,s}^2
    +
    \frac{b_{m,s}^2}{
      (1+\sqrt{1-\rho_{m,s}})^2
    }
  \right)^{1/2}
  \\
  &=
  \sqrt{1+\rho_{m,s}}\,
  \widehat{\Delta}_{m,t,s}
  +
  \xi_{m,s}.
\end{aligned}
\label{eq:app_empirical_delta_substitution}
\end{equation}
Substituting Eq.~\eqref{eq:app_empirical_delta_substitution} into
Eq.~\eqref{eq:main_gaussian_drift_bound} yields
Eq.~\eqref{eq:app_empirical_memory_bound}.
\end{proof}

\subsection{Tightness of the Wasserstein Forgetting Bound}
\label{app:spectral_forgetting_tightness}

The first inequality in Theorem~\ref{thm:spectral_drift_forgetting}
is essentially the tightest possible bound under only Lipschitz
spectral sufficiency.  The next proposition formalizes this claim.

\begin{proposition}[Sharpness under Lipschitz spectral sufficiency]
\label{prop:w1_sharpness}
Consider one skill and suppose the spectral sufficiency residual is
zero.  Let $P$ and $Q$ be two descriptor laws supported on a compact
subset of $\R^D$ under metric $d_{m,s}$.  For every $\eta>0$, there
exists an $L$-Lipschitz spectral surrogate $g$ such that the induced
risk difference satisfies
\begin{equation}
  R(\theta_t)-R(\theta_m)
  \ge
  L\W_{1,m,s}(P,Q)-\eta .
  \label{eq:app_sharpness_statement}
\end{equation}
Consequently, no uniform upper bound depending only on the
descriptor laws and the Lipschitz constant can improve the leading
term $L\W_{1,m,s}(P,Q)$ by a fixed positive factor.
\end{proposition}

\begin{proof}
By the Kantorovich--Rubinstein duality on compact metric spaces,
\begin{equation}
  \W_{1,m,s}(P,Q)
  =
  \sup_{\Lip(h)\le 1}
  \left\{
    \E_{Z\sim P}h(Z)
    -
    \E_{Z\sim Q}h(Z)
  \right\}.
  \label{eq:app_kr_duality}
\end{equation}
Therefore, for every $\eta>0$, there exists a $1$-Lipschitz function
$h_{\eta}$ such that
\begin{equation}
  \E_{Z\sim P}h_{\eta}(Z)
  -
  \E_{Z\sim Q}h_{\eta}(Z)
  \ge
  \W_{1,m,s}(P,Q)-\eta/L .
  \label{eq:app_eta_optimal_potential}
\end{equation}
Set $g=Lh_{\eta}$.  Then $g$ is $L$-Lipschitz and
\begin{equation}
\begin{aligned}
  \E_{Z\sim P}g(Z)
  -
  \E_{Z\sim Q}g(Z)
  &=
  L
  \left(
    \E_{Z\sim P}h_{\eta}(Z)
    -
    \E_{Z\sim Q}h_{\eta}(Z)
  \right)
  \\
  &\ge
  L\W_{1,m,s}(P,Q)-\eta .
\end{aligned}
\label{eq:app_scaled_potential_bound}
\end{equation}
Because the support is compact, $h_{\eta}$ is bounded.  Adding a
constant to $g$ does not change the difference of expectations, so
one may shift it to be nonnegative if a nonnegative loss surrogate
is desired.  Construct a spectrally sufficient old task with
$Q_{m,s}^{\theta_t}=P$, $Q_{m,s}^{\theta_m}=Q$, and
\begin{equation}
  R(\theta)
  =
  \E_{Z\sim Q^{\theta}}g(Z).
  \label{eq:app_constructed_sufficient_task}
\end{equation}
The residual is zero by construction, and
Eq.~\eqref{eq:app_scaled_potential_bound} gives
Eq.~\eqref{eq:app_sharpness_statement}.  Hence the Wasserstein term
in Theorem~\ref{thm:spectral_drift_forgetting} is sharp under the
stated assumptions.
\end{proof}

\subsection{Interpretation for ASR}
\label{app:spectral_forgetting_interpretation}

The theorem shows that ASR controls forgetting through three
quantities.  First, the main term
$\Delta_{m,t,s}$ measures the drift of the skill-conditioned
attention-spectrum prototype.  Its mean component penalizes shifts
in the canonical spectral focus of skill $s$, while its covariance
component penalizes changes in the scale and directional variability
of that focus.  Second, $\Gamma_{m,s}^{\theta_m}$ and
$\Gamma_{m,s}^{\theta_t}$ measure the error of the Gaussian prototype
approximation.  These terms vanish when descriptors are Gaussian in
the ASR space and remain small when the stored prototype is a good
summary of the skill-conditioned spectral distribution.  Third,
$\epsilon_{m,s}$ measures the part of old-task behavior not explained
by attention spectra.

The empirical version shows why compact memory is sufficient.  ASR
does not need old images or old teacher logits; it only needs
$(\widehat{\mu}_{m,s},\widehat{\Sigma}_{m,s})$ per skill.  The price
of using this finite memory is the perturbation term
$\xi_{m,s}$ in Eq.~\eqref{eq:app_empirical_memory_error}.  Thus, the
theory predicts that forgetting should be small when the following
three conditions hold: skill-wise spectral drift is small,
skill-wise spectral prototypes are accurately estimated, and the old
task loss is well explained by the preserved cross-attention
spectra.

\section{Full Proofs for Phase-Invariant Spectral Stability}
\label{app:phase_invariant_stability}

Let
\begin{equation}
  \mathbb{T}_{H,W}
  :=
  \mathbb{Z}_H\times\mathbb{Z}_W
  \label{eq:phase_app_torus_def}
\end{equation}
be the discrete two-dimensional torus.  We identify an attention map
$A\in\R^{H\times W}$ with a vector in $\C^{N}$, where
$N=HW$.  The unitary DFT operator $\cF:\C^N\to\C^N$ is defined by
Eq.~\eqref{eq:phase_main_dft_def}.  Since $\cF$ is unitary, Parseval's
identity gives
\begin{equation}
  \|\cF A\|_2
  =
  \|A\|_F,
  \qquad
  \langle \cF A,\cF B\rangle
  =
  \langle A,B\rangle .
  \label{eq:phase_app_parseval}
\end{equation}
For a nonzero Fourier vector $a\in\C^N$, define its coordinate power
distribution by
\begin{equation}
  q(a)_j
  :=
  \frac{|a_j|^2}{\|a\|_2^2},
  \qquad
  j\in\{1,\ldots,N\}.
  \label{eq:phase_app_q_def}
\end{equation}
Thus the Fourier power spectrum of an attention map is
\begin{equation}
  p(A)
  =
  q(\cF A).
  \label{eq:phase_app_p_as_q}
\end{equation}

The ASR descriptor considered in this analysis is
\begin{equation}
  \Psi(A)
  =
  C p(A),
  \label{eq:phase_app_psi_def}
\end{equation}
where $C$ is a fixed linear coarsening map, such as the map that
aggregates frequency bins into radial and angular spectra.  Its
$\ell_1$-to-$\ell_2$ operator norm is
\begin{equation}
  \kappa_C
  =
  \sup_{r\ne0}
  \frac{\|Cr\|_2}{\|r\|_1}.
  \label{eq:phase_app_kappa_def}
\end{equation}

\subsection{DFT Diagonalization of Cyclic Translations}
\label{app:phase_translation_diagonalization}

\begin{lemma}[Cyclic translations are Fourier phase ramps]
\label{lem:phase_translation_diagonalization}
For every $A\in\C^{H\times W}$ and every cyclic shift
$\delta=(\delta_x,\delta_y)$,
\begin{equation}
  \widehat{\cT_{\delta}A}(u,v)
  =
  \exp\!\left(
    -2\pi i
    \left(
      \frac{u\delta_x}{H}
      +
      \frac{v\delta_y}{W}
    \right)
  \right)
  \widehat A(u,v).
  \label{eq:phase_app_translation_phase}
\end{equation}
Consequently,
\begin{equation}
  p(\cT_{\delta}A)=p(A),
  \qquad
  \Psi(\cT_{\delta}A)=\Psi(A).
  \label{eq:phase_app_translation_invariance}
\end{equation}
\end{lemma}

\begin{proof}
Starting from the definition of the DFT and the cyclic translation,
we have
\begin{equation}
\begin{aligned}
  \widehat{\cT_{\delta}A}(u,v)
  &=
  \frac{1}{\sqrt{HW}}
  \sum_{x=0}^{H-1}
  \sum_{y=0}^{W-1}
  A\big((x-\delta_x)\!\!\mod H,\,
        (y-\delta_y)\!\!\mod W\big)
  \\
  &\qquad\qquad\cdot
  \exp\!\left(
    -2\pi i
    \left(
      \frac{ux}{H}
      +
      \frac{vy}{W}
    \right)
  \right).
\end{aligned}
\label{eq:phase_app_translation_proof_start}
\end{equation}
Let
$x'=(x-\delta_x)\!\!\mod H$ and
$y'=(y-\delta_y)\!\!\mod W$.  Since the map
$(x,y)\mapsto(x',y')$ is a bijection on $\mathbb{T}_{H,W}$, we get
\begin{equation}
\resizebox{\linewidth}{!}{$
\begin{aligned}
  \widehat{\cT_{\delta}A}(u,v)
  &=
  \frac{1}{\sqrt{HW}}
  \sum_{x'=0}^{H-1}
  \sum_{y'=0}^{W-1}
  A(x',y')
  \\
  &\qquad\qquad\cdot
  \exp\!\left(
    -2\pi i
    \left(
      \frac{u(x'+\delta_x)}{H}
      +
      \frac{v(y'+\delta_y)}{W}
    \right)
  \right)
  \\
  &=
  \exp\!\left(
    -2\pi i
    \left(
      \frac{u\delta_x}{H}
      +
      \frac{v\delta_y}{W}
    \right)
  \right)
  \\
  &\qquad\qquad\cdot
  \frac{1}{\sqrt{HW}}
  \sum_{x'=0}^{H-1}
  \sum_{y'=0}^{W-1}
  A(x',y')
  \exp\!\left(
    -2\pi i
    \left(
      \frac{ux'}{H}
      +
      \frac{vy'}{W}
    \right)
  \right)
  \\
  &=
  \exp\!\left(
    -2\pi i
    \left(
      \frac{u\delta_x}{H}
      +
      \frac{v\delta_y}{W}
    \right)
  \right)
  \widehat A(u,v).
\end{aligned}
$}
\label{eq:phase_app_translation_change_variables}
\end{equation}
Taking squared magnitudes cancels the phase factor:
\begin{equation}
\begin{aligned}
  |\widehat{\cT_{\delta}A}(u,v)|^2
  &=
  \left|
    \exp\!\left(
      -2\pi i
      \left(
        \frac{u\delta_x}{H}
        +
        \frac{v\delta_y}{W}
      \right)
    \right)
  \right|^2
  |\widehat A(u,v)|^2
  \\
  &=
  |\widehat A(u,v)|^2.
\end{aligned}
\label{eq:phase_app_phase_cancel}
\end{equation}
By Parseval's identity,
\begin{equation}
  \|\cT_{\delta}A\|_F
  =
  \|\cF\cT_{\delta}A\|_2
  =
  \|\cF A\|_2
  =
  \|A\|_F.
  \label{eq:phase_app_translation_norm}
\end{equation}
Combining Eq.~\eqref{eq:phase_app_phase_cancel} and
Eq.~\eqref{eq:phase_app_translation_norm} gives
$p(\cT_{\delta}A)=p(A)$.  Since $\Psi(A)=Cp(A)$, we also obtain
$\Psi(\cT_{\delta}A)=\Psi(A)$.
\end{proof}

\subsection{A Sharp Projective Stability Inequality}
\label{app:phase_projective_stability}

The next lemma is the key stability estimate.  It shows that the
normalized power map is Lipschitz with the optimal constant $2$ under
relative perturbations.  The proof uses the geometry of rank-one
projectors.

\begin{lemma}[Coordinate power is controlled by projective distance]
\label{lem:phase_projective_tv}
Let $a,b\in\C^N$ be nonzero vectors and define
$u=a/\|a\|_2$ and $v=b/\|b\|_2$.  Then
\begin{equation}
  \|q(a)-q(b)\|_1
  \le
  2
  \sqrt{
    1-
    |\langle u,v\rangle|^2
  }.
  \label{eq:phase_app_projective_tv_statement}
\end{equation}
\end{lemma}

\begin{proof}
Let $P_u=uu^*$ and $P_v=vv^*$ be the rank-one projectors associated
with $u$ and $v$.  For any sign vector
$\sigma\in[-1,1]^N$, let $D_{\sigma}$ be the diagonal matrix with
diagonal entries $\sigma_1,\ldots,\sigma_N$.  Since
$\|D_{\sigma}\|_{\mathrm{op}}\le1$, we have
\begin{equation}
\begin{aligned}
  \|q(a)-q(b)\|_1
  &=
  \sup_{\sigma\in[-1,1]^N}
  \sum_{j=1}^N
  \sigma_j
  \left(
    |u_j|^2-|v_j|^2
  \right)
  \\
  &=
  \sup_{\sigma\in[-1,1]^N}
  \Tr
  \left[
    D_{\sigma}
    (P_u-P_v)
  \right]
  \\
  &\le
  \sup_{\|M\|_{\mathrm{op}}\le1}
  \left|
    \Tr
    \left[
      M(P_u-P_v)
    \right]
  \right|
  \\
  &=
  \|P_u-P_v\|_* .
\end{aligned}
\label{eq:phase_app_tv_to_trace_norm}
\end{equation}
Here $\|\cdot\|_*$ is the nuclear norm.  It remains to compute
$\|P_u-P_v\|_*$.  Multiplying $v$ by a global phase does not change
$P_v$, so assume without loss of generality that
$\langle u,v\rangle=c\in[0,1]$.  If $c=1$, the claim is immediate.
Otherwise define
\begin{equation}
  w
  :=
  \frac{
    v-cu
  }{
    \sqrt{1-c^2}
  },
  \qquad
  v
  =
  cu+\sqrt{1-c^2}\,w,
  \label{eq:phase_app_two_dim_basis}
\end{equation}
where $\{u,w\}$ is an orthonormal basis for
$\mathrm{span}\{u,v\}$.  In this basis, $P_u-P_v$ has matrix
representation
\begin{equation}
  P_u-P_v
  =
  \begin{pmatrix}
    1-c^2 & -c\sqrt{1-c^2} \\
    -c\sqrt{1-c^2} & -(1-c^2)
  \end{pmatrix}.
  \label{eq:phase_app_projector_matrix}
\end{equation}
The trace and determinant of this $2\times2$ matrix are
\begin{equation}
  \Tr(P_u-P_v)=0,
  \qquad
  \det(P_u-P_v)=-(1-c^2).
  \label{eq:phase_app_projector_trace_det}
\end{equation}
Therefore its two nonzero eigenvalues are
$\sqrt{1-c^2}$ and $-\sqrt{1-c^2}$, and hence
\begin{equation}
  \|P_u-P_v\|_*
  =
  2\sqrt{1-c^2}
  =
  2
  \sqrt{
    1-
    |\langle u,v\rangle|^2
  }.
  \label{eq:phase_app_trace_norm_projector}
\end{equation}
Substituting Eq.~\eqref{eq:phase_app_trace_norm_projector} into
Eq.~\eqref{eq:phase_app_tv_to_trace_norm} proves the lemma.
\end{proof}

\begin{lemma}[Sharp perturbation stability of normalized power]
\label{lem:phase_sharp_power_stability}
Let $a,e\in\C^N$ with $a\ne0$, and let $b=a+e$.  If
\begin{equation}
  \rho
  :=
  \frac{\|e\|_2}{\|a\|_2}
  \le1,
  \label{eq:phase_app_relative_noise_def}
\end{equation}
then
\begin{equation}
  \|q(b)-q(a)\|_1
  \le
  2\rho .
  \label{eq:phase_app_sharp_power_stability}
\end{equation}
The constant $2$ is sharp.
\end{lemma}

\begin{proof}
Let $u=a/\|a\|_2$ and $v=b/\|b\|_2$.  Since $b$ lies in the one
dimensional subspace spanned by $v$, the distance from $u$ to that
subspace is no larger than the distance from $u$ to $b/\|a\|_2$:
\begin{equation}
\begin{aligned}
  \inf_{\alpha\in\C}
  \|u-\alpha v\|_2
  &\le
  \left\|
    u-\frac{b}{\|a\|_2}
  \right\|_2
  \\
  &=
  \left\|
    \frac{a}{\|a\|_2}
    -
    \frac{a+e}{\|a\|_2}
  \right\|_2
  \\
  &=
  \frac{\|e\|_2}{\|a\|_2}
  =
  \rho .
\end{aligned}
\label{eq:phase_app_subspace_distance_upper}
\end{equation}
On the other hand, the exact distance from a unit vector $u$ to the
span of a unit vector $v$ is
\begin{equation}
\begin{aligned}
  \inf_{\alpha\in\C}
  \|u-\alpha v\|_2^2
  &=
  \inf_{\alpha\in\C}
  \left(
    \|u\|_2^2
    -
    2\mathrm{Re}
    \left[
      \overline{\alpha}
      \langle v,u\rangle
    \right]
    +
    |\alpha|^2\|v\|_2^2
  \right)
  \\
  &=
  1
  -
  |\langle u,v\rangle|^2 .
\end{aligned}
\label{eq:phase_app_subspace_distance_exact}
\end{equation}
Combining Eq.~\eqref{eq:phase_app_subspace_distance_upper} and
Eq.~\eqref{eq:phase_app_subspace_distance_exact} yields
\begin{equation}
  \sqrt{
    1-
    |\langle u,v\rangle|^2
  }
  \le
  \rho .
  \label{eq:phase_app_angle_bound}
\end{equation}
Applying Lemma~\ref{lem:phase_projective_tv} gives
\begin{equation}
\begin{aligned}
  \|q(b)-q(a)\|_1
  &\le
  2
  \sqrt{
    1-
    |\langle u,v\rangle|^2
  }
  \\
  &\le
  2\rho .
\end{aligned}
\label{eq:phase_app_sharp_power_chain}
\end{equation}

It remains to prove sharpness.  It suffices to construct a
two-dimensional example.  Fix any $\rho\in[0,1]$ and set
$\varphi=\arcsin(\rho)$ and
$\alpha_0=\pi/4-\varphi/2$.  Define
\begin{equation}
  u
  :=
  \begin{pmatrix}
    \cos\alpha_0 \\
    \sin\alpha_0
  \end{pmatrix},
  \qquad
  v
  :=
  \begin{pmatrix}
    \cos(\alpha_0+\varphi) \\
    \sin(\alpha_0+\varphi)
  \end{pmatrix}.
  \label{eq:phase_app_sharp_vectors}
\end{equation}
Let $a=u$, $b=\cos\varphi\, v$, and $e=b-a$.  Then
\begin{equation}
\begin{aligned}
  \|e\|_2^2
  &=
  \|\cos\varphi\,v-u\|_2^2
  \\
  &=
  \cos^2\varphi
  +
  1
  -
  2\cos\varphi\,\langle u,v\rangle
  \\
  &=
  \cos^2\varphi
  +
  1
  -
  2\cos^2\varphi
  \\
  &=
  \sin^2\varphi
  =
  \rho^2 .
\end{aligned}
\label{eq:phase_app_sharp_noise_norm}
\end{equation}
Since normalizing $b$ gives $v$, the corresponding coordinate power
distributions satisfy
\begin{equation}
\begin{aligned}
  \|q(b)-q(a)\|_1
  &=
  |\cos^2(\alpha_0+\varphi)-\cos^2\alpha_0|
  \\
  &\quad+
  |\sin^2(\alpha_0+\varphi)-\sin^2\alpha_0|
  \\
  &=
  2
  |\cos^2(\alpha_0+\varphi)-\cos^2\alpha_0|
  \\
  &=
  2
  |\sin(2\alpha_0+\varphi)\sin\varphi|
  \\
  &=
  2\rho .
\end{aligned}
\label{eq:phase_app_sharp_l1}
\end{equation}
Thus the factor $2$ in Eq.~\eqref{eq:phase_app_sharp_power_stability}
cannot be improved.
\end{proof}

\subsection{Coarsening, Transport, and Risk Stability}
\label{app:phase_coarsening_transport}

\begin{lemma}[ASR coarsening stability]
\label{lem:phase_coarsening_stability}
For any two nonzero attention maps $A$ and $B$,
\begin{equation}
  \|\Psi(A)-\Psi(B)\|_2
  \le
  \kappa_C
  \|p(A)-p(B)\|_1 .
  \label{eq:phase_app_coarsening_statement}
\end{equation}
\end{lemma}

\begin{proof}
By the definition of $\Psi$ and $\kappa_C$,
\begin{equation}
\begin{aligned}
  \|\Psi(A)-\Psi(B)\|_2
  &=
  \|Cp(A)-Cp(B)\|_2
  \\
  &=
  \|C(p(A)-p(B))\|_2
  \\
  &\le
  \kappa_C
  \|p(A)-p(B)\|_1 .
\end{aligned}
\label{eq:phase_app_coarsening_proof}
\end{equation}
\end{proof}

\begin{lemma}[Pointwise perturbed-translation stability]
\label{lem:phase_pointwise_perturbed_translation}
Let $A\ne0$ and
\begin{equation}
  B
  =
  \cT_{\delta}A+E,
  \qquad
  \rho
  =
  \frac{\|E\|_F}{\|A\|_F}
  \le1 .
  \label{eq:phase_app_pointwise_model}
\end{equation}
Then
\begin{equation}
  \|p(B)-p(A)\|_1
  \le
  2\rho,
  \qquad
  \|\Psi(B)-\Psi(A)\|_2
  \le
  2\kappa_C\rho .
  \label{eq:phase_app_pointwise_result}
\end{equation}
\end{lemma}

\begin{proof}
Let
\begin{equation}
  a
  :=
  \cF(\cT_{\delta}A),
  \qquad
  e
  :=
  \cF E,
  \qquad
  b
  :=
  \cF B .
  \label{eq:phase_app_frequency_vectors}
\end{equation}
By linearity of the DFT and the model $B=\cT_{\delta}A+E$,
\begin{equation}
  b
  =
  a+e .
  \label{eq:phase_app_frequency_additive}
\end{equation}
Parseval's identity and Lemma~\ref{lem:phase_translation_diagonalization}
give
\begin{equation}
  \|a\|_2
  =
  \|\cT_{\delta}A\|_F
  =
  \|A\|_F,
  \qquad
  \|e\|_2
  =
  \|E\|_F .
  \label{eq:phase_app_frequency_norms}
\end{equation}
Therefore
\begin{equation}
  \frac{\|e\|_2}{\|a\|_2}
  =
  \frac{\|E\|_F}{\|A\|_F}
  =
  \rho .
  \label{eq:phase_app_frequency_relative_noise}
\end{equation}
Applying Lemma~\ref{lem:phase_sharp_power_stability} yields
\begin{equation}
  \|q(b)-q(a)\|_1
  \le
  2\rho .
  \label{eq:phase_app_qb_qa_bound}
\end{equation}
Since $q(b)=p(B)$ and $q(a)=p(\cT_{\delta}A)$, while
$p(\cT_{\delta}A)=p(A)$ by Lemma~\ref{lem:phase_translation_diagonalization},
we obtain
\begin{equation}
\begin{aligned}
  \|p(B)-p(A)\|_1
  &=
  \|q(b)-q(\cF A)\|_1
  \\
  &=
  \|q(b)-q(a)\|_1
  \\
  &\le
  2\rho .
\end{aligned}
\label{eq:phase_app_pointwise_power_final}
\end{equation}
The descriptor bound follows from Lemma~\ref{lem:phase_coarsening_stability}:
\begin{equation}
\begin{aligned}
  \|\Psi(B)-\Psi(A)\|_2
  &\le
  \kappa_C
  \|p(B)-p(A)\|_1
  \\
  &\le
  2\kappa_C\rho .
\end{aligned}
\label{eq:phase_app_pointwise_descriptor_final}
\end{equation}
\end{proof}

\begin{lemma}[Distributional transport stability]
\label{lem:phase_distributional_transport}
Fix a skill $s$.  Suppose there exists a coupling of source and
target attention maps such that
\begin{equation}
  A^{\mathrm{tar}}
  =
  \cT_{\Delta}A^{\mathrm{src}}+E,
  \qquad
  \rho
  =
  \frac{\|E\|_F}{\|A^{\mathrm{src}}\|_F}
  \le1
  \quad
  \mathrm{a.s.}
  \label{eq:phase_app_distribution_coupling}
\end{equation}
Then
\begin{equation}
\begin{aligned}
  &\W_1
  \left(
    \law(\Psi(A^{\mathrm{tar}})\mid s),
    \law(\Psi(A^{\mathrm{src}})\mid s)
  \right)
  \\
  &\qquad\le
  2\kappa_C
  \E
  \left[
    \rho
    \mid s
  \right],
\end{aligned}
\label{eq:phase_app_distribution_transport_statement}
\end{equation}
where $\W_1$ uses the Euclidean distance on the descriptor space.
\end{lemma}

\begin{proof}
Let
\begin{equation}
  Z^{\mathrm{src}}
  :=
  \Psi(A^{\mathrm{src}}),
  \qquad
  Z^{\mathrm{tar}}
  :=
  \Psi(A^{\mathrm{tar}}).
  \label{eq:phase_app_z_src_tar_def}
\end{equation}
The joint law of $(Z^{\mathrm{src}},Z^{\mathrm{tar}})$ induced by the
assumed coupling is a valid transport plan between the two descriptor
laws.  Hence
\begin{equation}
\begin{aligned}
  &\W_1
  \left(
    \law(Z^{\mathrm{tar}}\mid s),
    \law(Z^{\mathrm{src}}\mid s)
  \right)
  \\
  &\le
  \E
  \left[
    \|Z^{\mathrm{tar}}-Z^{\mathrm{src}}\|_2
    \mid s
  \right]
  \\
  &=
  \E
  \left[
    \|\Psi(A^{\mathrm{tar}})
      -
      \Psi(A^{\mathrm{src}})\|_2
    \mid s
  \right]
  \\
  &\le
  \E
  \left[
    2\kappa_C\rho
    \mid s
  \right]
  \\
  &=
  2\kappa_C
  \E
  \left[
    \rho
    \mid s
  \right].
\end{aligned}
\label{eq:phase_app_distribution_transport_proof}
\end{equation}
The third line uses Lemma~\ref{lem:phase_pointwise_perturbed_translation}.
\end{proof}

\begin{lemma}[Spectral risk stability]
\label{lem:phase_risk_stability}
Under Assumption~\ref{ass:phase_main_lipschitz_risk} and the coupling
condition in Eq.~\eqref{eq:phase_app_distribution_coupling},
\begin{equation}
\begin{aligned}
  &\left|
    \E
    \left[
      h_s(\Psi(A^{\mathrm{tar}}))
      \mid s
    \right]
    -
    \E
    \left[
      h_s(\Psi(A^{\mathrm{src}}))
      \mid s
    \right]
  \right|
  \\
  &\qquad\le
  2L_s\kappa_C
  \E
  \left[
    \rho
    \mid s
  \right].
\end{aligned}
\label{eq:phase_app_risk_stability_statement}
\end{equation}
\end{lemma}

\begin{proof}
Using the same source--target coupling as in
Lemma~\ref{lem:phase_distributional_transport}, we obtain
\begin{equation}
\begin{aligned}
  &\left|
    \E
    \left[
      h_s(\Psi(A^{\mathrm{tar}}))
      \mid s
    \right]
    -
    \E
    \left[
      h_s(\Psi(A^{\mathrm{src}}))
      \mid s
    \right]
  \right|
  \\
  &=
  \left|
    \E
    \left[
      h_s(\Psi(A^{\mathrm{tar}}))
      -
      h_s(\Psi(A^{\mathrm{src}}))
      \mid s
    \right]
  \right|
  \\
  &\le
  \E
  \left[
    \left|
      h_s(\Psi(A^{\mathrm{tar}}))
      -
      h_s(\Psi(A^{\mathrm{src}}))
    \right|
    \mid s
  \right]
  \\
  &\le
  L_s
  \E
  \left[
    \|\Psi(A^{\mathrm{tar}})
      -
      \Psi(A^{\mathrm{src}})\|_2
    \mid s
  \right]
  \\
  &\le
  2L_s\kappa_C
  \E
  \left[
    \rho
    \mid s
  \right].
\end{aligned}
\label{eq:phase_app_risk_stability_proof}
\end{equation}
The second inequality uses Assumption~\ref{ass:phase_main_lipschitz_risk},
and the last inequality uses
Lemma~\ref{lem:phase_pointwise_perturbed_translation}.
\end{proof}

\subsection{Proof of Theorem~\ref{thm:phase_invariant_stability}}
\label{app:phase_theorem_proof}
\begin{proof}
The exact invariance statement follows directly from
Lemma~\ref{lem:phase_translation_diagonalization}:
\begin{equation}
  p(\cT_{\delta}A)=p(A),
  \qquad
  \Psi(\cT_{\delta}A)=\Psi(A).
  \label{eq:phase_app_theorem_exact}
\end{equation}
For the perturbed shift model
$A^{\mathrm{tar}}=\cT_{\Delta}A^{\mathrm{src}}+E$, applying
Lemma~\ref{lem:phase_pointwise_perturbed_translation} with
$A=A^{\mathrm{src}}$, $B=A^{\mathrm{tar}}$, and
$\delta=\Delta$ gives
\begin{equation}
\begin{aligned}
  \left\|
    p(A^{\mathrm{tar}})
    -
    p(A^{\mathrm{src}})
  \right\|_1
  &\le
  2\rho,
  \\
  \left\|
    \Psi(A^{\mathrm{tar}})
    -
    \Psi(A^{\mathrm{src}})
  \right\|_2
  &\le
  2\kappa_C\rho .
\end{aligned}
\label{eq:phase_app_theorem_pointwise}
\end{equation}
The sharpness of the constant $2$ follows from
Lemma~\ref{lem:phase_sharp_power_stability}.  The distributional
bound follows from Lemma~\ref{lem:phase_distributional_transport}:
\begin{equation}
\begin{aligned}
  &\W_1
  \left(
    \law(\Psi(A^{\mathrm{tar}})\mid s),
    \law(\Psi(A^{\mathrm{src}})\mid s)
  \right)
  \\
  &\qquad\le
  2\kappa_C
  \E
  \left[
    \rho
    \mid s
  \right].
\end{aligned}
\label{eq:phase_app_theorem_distribution}
\end{equation}
The risk bound follows from Lemma~\ref{lem:phase_risk_stability}:
\begin{equation}
\begin{aligned}
  &\left|
    \E
    \left[
      h_s(\Psi(A^{\mathrm{tar}}))
      \mid s
    \right]
    -
    \E
    \left[
      h_s(\Psi(A^{\mathrm{src}}))
      \mid s
    \right]
  \right|
  \\
  &\qquad\le
  2L_s\kappa_C
  \E
  \left[
    \rho
    \mid s
  \right].
\end{aligned}
\label{eq:phase_app_theorem_risk}
\end{equation}
If $E=0$, then $\rho=0$ almost surely.  Substituting $\rho=0$ into
Eq.~\eqref{eq:phase_app_theorem_pointwise},
Eq.~\eqref{eq:phase_app_theorem_distribution}, and
Eq.~\eqref{eq:phase_app_theorem_risk} gives zero spectral drift and
zero spectral risk gap:
\begin{equation}
\begin{aligned}
  \left\|
    \Psi(A^{\mathrm{tar}})
    -
    \Psi(A^{\mathrm{src}})
  \right\|_2
  &=
  0,
  \\
  \W_1
  \left(
    \law(\Psi(A^{\mathrm{tar}})\mid s),
    \law(\Psi(A^{\mathrm{src}})\mid s)
  \right)
  &=
  0 .
\end{aligned}
\label{eq:phase_app_zero_shift_consequence}
\end{equation}
This completes the proof.
\end{proof}

\subsection{Boundary-Corrupted Translations}
\label{app:phase_boundary_corrupted}

The theorem above is exact for cyclic translations.  Ordinary image
translations with zero padding or cropping are not exactly cyclic,
but they can be written as a cyclic translation plus a boundary
residual.  This gives a useful robustness interpretation.

Let $\cT_{\delta}^{0}$ denote the zero-padded shift:
\begin{equation}
\begin{aligned}
  (\cT_{\delta}^{0}A)(x,y)
  &:=
  A(x-\delta_x,y-\delta_y)
  \mathbf{1}\!\left\{
    \begin{array}{l}
      0\le x-\delta_x < H,\\
      0\le y-\delta_y < W
    \end{array}
  \right\}.
\end{aligned}
\label{eq:phase_app_zero_padded_translation}
\end{equation}
Define the boundary residual by
\begin{equation}
  E_{\partial,\delta}(A)
  :=
  \cT_{\delta}^{0}A-\cT_{\delta}A .
  \label{eq:phase_app_boundary_residual}
\end{equation}

\begin{corollary}[Stability under zero-padded translations]
\label{cor:phase_zero_padded}
For every nonzero $A$ and every shift $\delta$, if
\begin{equation}
  \rho_{\partial,\delta}(A)
  :=
  \frac{
    \|E_{\partial,\delta}(A)\|_F
  }{
    \|A\|_F
  }
  \le1,
  \label{eq:phase_app_boundary_ratio}
\end{equation}
then
\begin{equation}
\begin{aligned}
  \|p(\cT_{\delta}^{0}A)-p(A)\|_1
  &\le
  2\rho_{\partial,\delta}(A),
  \\
  \|\Psi(\cT_{\delta}^{0}A)-\Psi(A)\|_2
  &\le
  2\kappa_C\rho_{\partial,\delta}(A).
\end{aligned}
\label{eq:phase_app_zero_padded_bound}
\end{equation}
\end{corollary}

\begin{proof}
By definition,
\begin{equation}
  \cT_{\delta}^{0}A
  =
  \cT_{\delta}A
  +
  E_{\partial,\delta}(A).
  \label{eq:phase_app_boundary_as_perturbation}
\end{equation}
Applying Lemma~\ref{lem:phase_pointwise_perturbed_translation} with
$E=E_{\partial,\delta}(A)$ and
$\rho=\rho_{\partial,\delta}(A)$ gives
Eq.~\eqref{eq:phase_app_zero_padded_bound}.
\end{proof}

\subsection{Why Raw Attention Matching Is Translation-Unstable}
\label{app:phase_raw_attention_instability}

The preceding results explain why ASR matches the power spectrum
rather than the raw attention map.  Raw attention matching is not
phase invariant and may assign a large penalty to a harmless spatial
shift.

\begin{proposition}[Raw attention matching can be maximally unstable]
\label{prop:phase_raw_instability}
There exist normalized nonnegative attention maps $A$ and cyclic
translations $\cT_{\delta}$ such that
\begin{equation}
  \|p(\cT_{\delta}A)-p(A)\|_1
  =
  0,
  \qquad
  \|\cT_{\delta}A-A\|_F
  =
  \sqrt{2}\|A\|_F .
  \label{eq:phase_app_raw_instability_statement}
\end{equation}
Thus, a raw attention $\ell_2$ penalty can be maximal among
nonnegative unit-norm attention maps, while the ASR spectral penalty
is exactly zero.
\end{proposition}

\begin{proof}
Let $A$ be a point mass at location $(0,0)$:
\begin{equation}
  A(x,y)
  =
  \mathbf{1}\{x=0,\ y=0\}.
  \label{eq:phase_app_delta_attention}
\end{equation}
Choose any nonzero shift $\delta$ such that
$\cT_{\delta}A$ is supported at a different location.  Then $A$ and
$\cT_{\delta}A$ have disjoint supports, and
\begin{equation}
\begin{aligned}
  \|\cT_{\delta}A-A\|_F^2
  &=
  \|\cT_{\delta}A\|_F^2
  +
  \|A\|_F^2
  -
  2
  \langle
    \cT_{\delta}A,
    A
  \rangle
  \\
  &=
  1+1-0
  =
  2 .
\end{aligned}
\label{eq:phase_app_raw_distance}
\end{equation}
Since $\|A\|_F=1$, this gives
$\|\cT_{\delta}A-A\|_F=\sqrt{2}\|A\|_F$.  On the other hand,
Lemma~\ref{lem:phase_translation_diagonalization} gives
\begin{equation}
  p(\cT_{\delta}A)=p(A),
  \qquad
  \|p(\cT_{\delta}A)-p(A)\|_1=0 .
  \label{eq:phase_app_spectral_zero_raw_large}
\end{equation}
Finally, for any two nonnegative unit-norm attention maps $A_1$ and
$A_2$, their inner product is nonnegative, so
\begin{equation}
\begin{aligned}
  \|A_1-A_2\|_F^2
  &=
  \|A_1\|_F^2
  +
  \|A_2\|_F^2
  -
  2\langle A_1,A_2\rangle
  \\
  &\le
  2 .
\end{aligned}
\label{eq:phase_app_nonnegative_max}
\end{equation}
Hence the raw attention distance in
Eq.~\eqref{eq:phase_app_raw_distance} is maximal in this class.
\end{proof}

\subsection{Peak Stability under a Spectral Gap}
\label{app:phase_peak_stability}

ASR may additionally record the dominant spectral peak.  The peak
index is not globally Lipschitz because the maximizer can switch
under arbitrarily small perturbations if the top two frequencies are
tied.  However, it is stable under a standard margin condition.

For a nonzero attention map $A$, define the unnormalized power
spectrum
\begin{equation}
  P(A)_{u,v}
  :=
  |\widehat A(u,v)|^2 .
  \label{eq:phase_app_unnormalized_power}
\end{equation}
Let $k^\star(A)$ denote the dominant frequency index and define the
peak gap
\begin{equation}
  \gamma_{\mathrm{pk}}(A)
  :=
  P(A)_{k^\star(A)}
  -
  \max_{k\ne k^\star(A)}
  P(A)_k .
  \label{eq:phase_app_peak_gap}
\end{equation}

\begin{lemma}[Dominant peak stability]
\label{lem:phase_peak_stability}
Let $B=\cT_{\delta}A+E$ and let
$\eta=\|E\|_F/\|A\|_F$.  If
\begin{equation}
  2\eta+\eta^2
  <
  \frac{
    \gamma_{\mathrm{pk}}(A)
  }{
    2\|A\|_F^2
  },
  \label{eq:phase_app_peak_condition}
\end{equation}
then
\begin{equation}
  k^\star(B)=k^\star(A).
  \label{eq:phase_app_peak_same}
\end{equation}
\end{lemma}

\begin{proof}
Let $a=\cF(\cT_{\delta}A)$, $e=\cF E$, and $b=a+e=\cF B$.
For each frequency index $k$, we have
\begin{equation}
\begin{aligned}
  \left|
    |b_k|^2-|a_k|^2
  \right|
  &=
  \left|
    |a_k+e_k|^2-|a_k|^2
  \right|
  \\
  &=
  \left|
    2\mathrm{Re}(a_k\overline{e_k})
    +
    |e_k|^2
  \right|
  \\
  &\le
  2|a_k||e_k|
  +
  |e_k|^2
  \\
  &\le
  2\|a\|_2\|e\|_2
  +
  \|e\|_2^2 .
\end{aligned}
\label{eq:phase_app_peak_pointwise_power_bound}
\end{equation}
By Parseval and translation invariance of the Frobenius norm,
\begin{equation}
  \|a\|_2=\|A\|_F,
  \qquad
  \|e\|_2=\|E\|_F=\eta\|A\|_F .
  \label{eq:phase_app_peak_norms}
\end{equation}
Therefore,
\begin{equation}
  \|P(B)-P(\cT_{\delta}A)\|_{\infty}
  \le
  (2\eta+\eta^2)\|A\|_F^2 .
  \label{eq:phase_app_peak_linf_bound}
\end{equation}
Since $P(\cT_{\delta}A)=P(A)$ by
Lemma~\ref{lem:phase_translation_diagonalization}, we have
\begin{equation}
  \|P(B)-P(A)\|_{\infty}
  \le
  (2\eta+\eta^2)\|A\|_F^2 .
  \label{eq:phase_app_peak_linf_bound_against_A}
\end{equation}
Let $k^\star=k^\star(A)$.  For any $k\ne k^\star$, the peak gap
definition gives
\begin{equation}
  P(A)_{k^\star}-P(A)_k
  \ge
  \gamma_{\mathrm{pk}}(A).
  \label{eq:phase_app_peak_gap_use}
\end{equation}
Using Eq.~\eqref{eq:phase_app_peak_linf_bound_against_A}, we get
\begin{equation}
\begin{aligned}
  P(B)_{k^\star}-P(B)_k
  &=
  \left(
    P(A)_{k^\star}-P(A)_k
  \right)
  \\
  &\quad+
  \left(
    P(B)_{k^\star}-P(A)_{k^\star}
  \right)
  \\
  &\quad-
  \left(
    P(B)_k-P(A)_k
  \right)
  \\
  &\ge
  \gamma_{\mathrm{pk}}(A)
  -
  2
  \|P(B)-P(A)\|_{\infty}
  \\
  &\ge
  \gamma_{\mathrm{pk}}(A)
  -
  2(2\eta+\eta^2)\|A\|_F^2
  \\
  &>
  0 .
\end{aligned}
\label{eq:phase_app_peak_order_preserved}
\end{equation}
Thus $P(B)_{k^\star}>P(B)_k$ for every $k\ne k^\star$, so the
dominant peak remains $k^\star$.
\end{proof}

\subsection{Interpretation for ASR}
\label{app:phase_interpretation}

Theorem~\ref{thm:phase_invariant_stability} gives a formal
robustness explanation for ASR.  A spatial translation of an
attention map becomes a phase ramp in the Fourier domain, and ASR
removes this phase by using the power spectrum.  Therefore, ASR does
not penalize changes in the absolute image location of an attended
object when the underlying attention structure is preserved.  The
stability bound further shows that, under realistic shifts with
boundary effects or attention noise, the spectral drift scales only
with the relative non-translational residual $\rho$, not with the
translation magnitude $\|\Delta\|$.  This is precisely the desired
behavior for continual multimodal learning: the model should
preserve skill-specific scale and directional focus patterns while
remaining insensitive to harmless spatial phase changes.

\section{Additional Experimental Details}
\label{app:exp-details}

\subsection{Datasets and Task Streams}
\label{app:datasets}

\subsubsection{VQA v2: Question-Type Incremental Stream}
\label{app:vqav2-10task}
\vspace{1mm}
\noindent\textbf{Base dataset.}
VQA v2 contains approximately $204$K images (from
MS-COCO) and $1.1$M questions with open-ended answers. We use the
official training and validation splits for continual training and
evaluation, respectively. Following standard practice in continual
VQA~\cite{vqacl23,quad25,clmoe25}, we treat the validation split as
the held-out test set and do not use the official test-dev or
test-std splits.

\vspace{1mm}
\noindent\textbf{Task construction.}
We adopt the 10-task ``question-type incremental'' protocol used by
VQACL, QUAD, and CL-MoE. Each question in VQA v2 is annotated with a
primary question type (e.g., \texttt{counting},
\texttt{color}, \texttt{location}, \texttt{comparison},
\texttt{attribute}, \texttt{activity}, \texttt{existence},
\texttt{reading}, \texttt{object}, \texttt{other}). We group
questions whose primary type matches one of the 10 categories, and
fix a deterministic ordering of these categories to define the task
stream
\begin{equation}
  \mathcal{D}^{(1)}, \mathcal{D}^{(2)}, \dots, \mathcal{D}^{(10)}.
\end{equation}
Each $\mathcal{D}^{(t)}$ contains all image--question--answer triples
whose question type is assigned to stage $t$. Images may appear in
multiple tasks if their questions belong to different types; this
matches prior continual VQA settings~\cite{vqacl23,quad25}.

\vspace{1mm}
\noindent\textbf{Train/validation splits per task.}
Within each $\mathcal{D}^{(t)}$, we use all training questions for
stage-$t$ optimization and all validation questions for evaluation.
We do not re-split the official partitions, all hyper-parameter
selection is done once on a held-out subset of the VQA v2 validation
split (5k images), and the resulting configuration is reused for all
experiments.

\subsubsection{VQAv2 + NExT-QA: VQACL Skill--Concept Setting}
\label{app:vqacl}

\vspace{1mm}
\noindent\textbf{Base datasets.}
The VQACL skill--concept benchmark~\cite{vqacl23} builds on VQA v2
and NExT-QA. NExT-QA consists of $\sim$4.3K videos and
$\sim$9.2K multi-choice questions targeting causal, temporal, and
descriptive reasoning.

\vspace{1mm}
\noindent\textbf{Skill and concept taxonomy.}
We follow the skill taxonomy $\mathcal{S}$ and concept grouping
released by VQACL. Skills include, for example,
\texttt{count}, \texttt{color}, \texttt{location},
\texttt{compare-number}, \texttt{compare-attribute},
\texttt{existence}, \texttt{relation}, and \texttt{read}. Concepts
are defined as disjoint groups of COCO object categories (e.g.,
\{person, dog, horse\}, \{car, bus, truck\}, \{bottle, cup, bowl\}),
plus additional video-specific concepts for NExT-QA.

\vspace{1mm}
\noindent\textbf{Skill--concept tasks.}
Each task in the stream corresponds to a subset of skill--concept
pairs $(s,c)$, where $s\in\mathcal{S}$ is a reasoning skill and
$c$ is a concept group. In the \emph{standard} split, both training
and test sets contain the same skill--concept combinations, while in
the \emph{novel-composition} split, certain combinations are held out
during training and only appear at test time. We exactly reuse the
task ordering and split specification from VQACL/QUAD: (i) for VQA v2
questions, we follow their skill and concept annotations; (ii) for
NExT-QA, we map questions to skills and concept groups using their
released mappings. The outer level of the stream is skill-based, and
within each skill, concept groups appear in a fixed order.

\vspace{1mm}
\noindent\textbf{Evaluation splits.}
For each skill--concept pair $(s,c)$, we use its training subset
during the corresponding stage(s) and its validation subset for
evaluation. We report AP and AF separately on the standard
(``std'') and novel-composition (``nov'') test splits, as in
Table~\ref{tab:vqa_results}.

\subsubsection{CoIN: Continual Instruction Tuning}
\label{app:coin}

\vspace{1mm}
\noindent\textbf{Base benchmark.}
CoIN~\cite{COIN} is a continual multimodal instruction tuning
benchmark containing $T{=}10$ datasets spanning diverse tasks
(e.g., VQA, captioning, OCR, referring expression comprehension,
classification, and science QA). We follow~\citet{SEFE25,Hidellava25}
and treat each dataset as one stage in the instruction-tuning stream.

\vspace{1mm}
\noindent\textbf{Stream construction.}
We use the canonical CoIN ordering released by prior CIT work (MoE-LoRA,
HiDe-LLaVA, SEFE, BranchLoRA, D-MoLE, Adapt-$\infty$). At stage
$t\in\{1,\dots,10\}$, the model sees only the training data from the
$t$-th dataset $\mathcal{D}^{(t)}$ and is not allowed to access raw
data from previous datasets (replay baselines are given a bounded
memory buffer. See App.~\ref{app:baselines}).

\vspace{1mm}
\noindent\textbf{Train/validation splits.}
For each dataset in CoIN, we use its official training and validation
splits. Unless otherwise specified by the dataset, we treat the
validation split as our held-out evaluation set. Dataset-specific
metrics (e.g., VQA accuracy, CIDEr, BLEU, exact match) are converted
to a unified scalar score $m_{a,b}$ following~\citet{SEFE25}.

\subsubsection{UCIT: Unseen Continual Instruction Tuning}
\label{app:ucit}

\vspace{1mm}
\noindent\textbf{Base benchmark.}
UCIT~\cite{Hidellava25} is designed to evaluate continual instruction
tuning when the post-training tasks are not part of the original
pretraining/fine-tuning mixture of the base MLLM. It assembles a
stream of multimodal instruction datasets that are distributionally
farther from LLaVA-1.5's pretraining data.

\vspace{1mm}
\noindent\textbf{Stream and splits.}
We follow the original UCIT protocol and ordering from
HiDe-LLaVA~\cite{Hidellava25}. As in CoIN, each stage corresponds to
one dataset, with its own training and validation splits. At stage
$t$, only the $t$-th dataset is accessible for optimization. We
report Last and Avg scores averaged across all UCIT datasets (see
App.~\ref{app:metrics}).

\newcommand{\bbmain}{LLaVA-1.5-7B}
\newcommand{\bbqwen}{Qwen2.5-VL-7B}
\newcommand{\bbintern}{InternVL3-8B}

\subsection{Backbone and Preprocessing}
\label{app:backbone}

\vspace{1mm}
\noindent\textbf{Backbone MLLM.}
Unless otherwise specified, the main experiments use \bbmain{} as the base
MLLM. The backbone robustness experiments use \bbqwen~\cite{qwen25vl} and
\bbintern~\cite{internvl3}. These additional backbones follow the same
continual task order, evaluation protocol, and ASR hyperparameters as the
\bbmain{} setting.
The vision branch is a CLIP-style vision transformer operating on
$224{\times}224$ or $336{\times}336$ images (we use the default
resolution of the official LLaVA-1.5-7B checkpoint). The language
branch is a 7B-parameter autoregressive transformer. We freeze all
backbone parameters and only fine-tune lightweight adapters and task
heads, as detailed below.

\vspace{1mm}
\noindent\textbf{Image preprocessing.}
We follow the official LLaVA pipeline. Images are resized to a fixed
resolution (longer side at most 512 pixels), padded to square if
necessary, and normalized with CLIP mean and variance. During
training we apply random horizontal flip and random cropping
restricted to keep at least $80\%$ of the shorter side. At test time
we only apply deterministic resizing and center cropping.

\vspace{1mm}
\noindent\textbf{Text preprocessing.}
Questions and instructions are tokenized with the same tokenizer used
by LLaVA-1.5-7B. For VQA-style tasks, we prepend a modality tag
(e.g., \texttt{"USER: <image> QUESTION:"}) and append a standard
answer template as in the official LLaVA instruction-tuning data.
For instruction-tuning tasks (CoIN/UCIT), we reuse the prompt
templates released by HiDe-LLaVA and SEFE whenever available.
Maximum sequence length is set to 256 tokens; longer instructions are
truncated from the left, preserving the most recent tokens and the
full answer region.

\subsection{Baselines and Implementation Details}
\label{app:baselines}

We compare ASR with both classic continual learning baselines and
recent state-of-the-art multimodal methods. Unless otherwise noted,
we re-implement all baselines on top of the same LLaVA-1.5-7B
backbone and training schedule as ASR.

\vspace{1mm}
\noindent\textbf{EWC~\cite{EWC}.}
Elastic Weight Consolidation penalizes parameter deviation from
previous stages using a diagonal Fisher information matrix. We
estimate the Fisher on the final model of stage $t{-}1$ using one
epoch over $\mathcal{D}^{(t-1)}$ and store only its diagonal entries.
The loss is
\begin{equation}
  \mathcal{L}_{\mathrm{EWC}}
  =
  \lambda_{\mathrm{EWC}}
  \sum_{i}
    F_i^{(t-1)}
    \bigl(\theta_i - \theta_i^{(t-1)}\bigr)^2,
\end{equation}
where $\lambda_{\mathrm{EWC}}$ is selected from
$\{10^3,10^4,10^5\}$ on a held-out validation subset.

\vspace{1mm}
\noindent\textbf{LwF~\cite{LwF}.}
Learning without Forgetting distills logits from the previous-stage
model $f_{\boldsymbol{\theta}^{(t-1)}}$ to the current model
$f_{\boldsymbol{\theta}}$ on the current task. For VQA, we distill
the predicted answer distribution; for instruction tuning, we distill
token-level distributions with temperature $\tau_{\mathrm{LwF}}=2$.
The distillation weight is tuned in $\{0.5,1.0,2.0\}$.

\vspace{1mm}
\noindent\textbf{ER~\cite{ER}.}
Experience Replay maintains a bounded buffer of raw examples. At the
end of stage $t$, we sample a fixed number of examples per task and
store them in a reservoir buffer of size $B$. During training at
stage $t{+}1$, each mini-batch contains a mixture of current-task
examples and replayed samples in a $3{:}1$ ratio. We follow
\citet{vqacl23} and set $B=20{,}000$ for VQA v2 and $B=10{,}000$ for
CoIN/UCIT.

\vspace{1mm}
\noindent\textbf{VQACL/QUAD/CL-MoE/BCP-MFA.}
For the VQA-based benchmarks we use the official implementations of
VQACL~\cite{vqacl23}, QUAD~\cite{quad25}, CL-MoE~\cite{clmoe25}, and
BCP-MFA~\cite{BCPMFA25} when available, and otherwise re-implement
them following the descriptions in their papers. All methods are
adapted to LLaVA-1.5-7B by replacing their vision-language backbones
with ours and keeping their MoE/router architectures and replay
buffer sizes. Hyper-parameters (e.g., router capacity, replay ratio)
are kept as close as possible to the original settings; when a direct
mapping is not possible, we select the best variant on the VQA v2
10-task validation subset.

\vspace{1mm}
\noindent\textbf{CIT baselines.}
For CoIN and UCIT we compare to LoRA-FT, O-LoRA, MoELoRA,
HiDe-LLaVA, SEFE, BranchLoRA, D-MoLE, and Adapt-$\infty$. We follow
\citet{Hidellava25,SEFE25,BranchLoRA25,DMoLE25,Adapt25} for their
backbone freezing strategies, LoRA rank and placement, and
dataset-specific sampling schedules. When multiple variants are
reported in prior work, we include the best-performing publicly
released configuration.

\subsection{Attention-Spectrum Regularization Implementation}
\label{app:asr-impl}

\vspace{1mm}
\noindent\textbf{Selected layers and heads.}
We apply ASR to a subset of cross-attention layers in the
vision--language fusion module. In all experiments, we select the
last $L_{\mathrm{sel}}{=}4$ fusion layers and all cross-attention
heads within these layers, i.e.,
\begin{equation}
  \mathcal{L}_{\mathrm{sel}}
  =
  \{\text{last 4 fusion layers}\},
  \quad
  \mathcal{H}_{\mathrm{sel}}
  = \mathcal{H}_{\mathrm{cross}}.
\end{equation}
We empirically found that using more layers yields diminishing
returns while increasing overhead.

\vspace{1mm}
\noindent\textbf{Spectral encoder parameters.}
We discretize the radial and angular frequencies into
$K{=}8$ radial bins and $M{=}8$ angular bins, respectively. The
per-head descriptor dimension is thus
\begin{equation}
  D_0
  =
  K + 2M + 3
  =
  8 + 16 + 3
  =
  27,
\end{equation}
and the aggregated descriptor dimension is $D{=}2D_0=54$. We
normalize spectra with a small constant $\varepsilon=10^{-8}$ to
avoid division by zero.

\vspace{1mm}
\noindent\textbf{Skill parser $g_{\boldsymbol{\psi}}$.}
We instantiate the skill parser as a lightweight Transformer-based
text classifier with 4 layers, 8 attention heads, hidden size 512,
and a softmax output over $|\mathcal{S}|$ skills. The parser is
trained offline on a mixture of VQA v2 and NExT-QA training questions
using skill labels from VQACL and heuristic templates (e.g., mapping
``How many'' to \texttt{count}, ``What color'' to \texttt{color}).
We minimize cross-entropy with AdamW (learning rate $1\mathrm{e}{-4}$)
for 10 epochs and freeze $g_{\boldsymbol{\psi}}$ when training ASR.
During ASR training, we use the soft posterior $\boldsymbol{\pi}(q)$
rather than hard skill labels. The skill threshold in
Sec.~\ref{subsec:prototypes} is set to $\tau_{\mathrm{skill}}=0.2$.

\vspace{1mm}
\noindent\textbf{Prototype memory.}
For each skill $s\in\mathcal{S}$, we maintain a mean
$\boldsymbol{\mu}_s\in\mathbb{R}^D$, a diagonal covariance
$\boldsymbol{\Sigma}_s\in\mathbb{R}^{D\times D}$, and a mean angular
spectrum $\hat{\mathbf{d}}_s\in\mathbb{R}^M$. At the end of stage
$t$, we estimate $\widehat{\boldsymbol{\mu}}_s^{(t)}$,
$\widehat{\boldsymbol{\Sigma}}_s^{(t)}$, and
$\widehat{\mathbf{d}}_s^{(t)}$ using all training examples in
$\mathcal{D}^{(t)}$ whose skill posterior for $s$ exceeds
$\tau_{\mathrm{skill}}$. We then update the global prototypes via
exponential moving averages with decay $\alpha=0.9$:
\begin{equation}
\begin{aligned}
  \boldsymbol{\mu}_s &\leftarrow
    \alpha\,\boldsymbol{\mu}_s
    + (1-\alpha)\,\widehat{\boldsymbol{\mu}}_s^{(t)}, \\
  \boldsymbol{\Sigma}_s &\leftarrow
    \alpha\,\boldsymbol{\Sigma}_s
    + (1-\alpha)\,\widehat{\boldsymbol{\Sigma}}_s^{(t)}, \\
  \hat{\mathbf{d}}_s &\leftarrow
    \alpha\,\hat{\mathbf{d}}_s
    + (1-\alpha)\,\widehat{\mathbf{d}}_s^{(t)}.
\end{aligned}
\end{equation}
At stage $t{=}1$, we train without spectral regularization
($\beta{=}0$) and initialize the prototypes from the resulting
$\boldsymbol{\theta}^{(1)}$.

\vspace{1mm}
\noindent\textbf{Spectral distillation hyper-parameters.}
Unless otherwise specified, we use
\begin{equation}
  \beta = 0.5,
  \quad
  \lambda_{\mathrm{ang}} = 0.1,
  \quad
  \tau_{\mathrm{mah}} = 1.0,
\end{equation}
and set the confidence weight $w_{\mathrm{spec}}(\boldsymbol{\phi})$
as in Sec.~\ref{subsec:distillation}. The geometry regularizer
weight is set to $\gamma = 0.05$ for VQA v2 and VQACL, and
$\gamma = 0.02$ for CoIN/UCIT to account for the larger variety of
instruction styles.

\subsection{Training Details}
\label{app:training}

\vspace{1mm}
\noindent\textbf{Optimizer and schedule.}
For all experiments we use AdamW with $\beta_1=0.9$,
$\beta_2=0.999$, and weight decay $0.01$. On the VQA v2 10-task
stream, we train each stage for 1 epoch with batch size $16$, maximum
learning rate $2\mathrm{e}{-4}$, and cosine decay without warmup. On
the VQACL skill--concept setting, we train 3 epochs per stage with
batch size $80$ and maximum learning rate $1\mathrm{e}{-4}$. On CoIN
and UCIT, we follow \citet{Hidellava25}: for large datasets (e.g.,
VQAv2, GQA) we train for 1 epoch per stage; for smaller OCR-heavy
datasets (e.g., TextVQA, OCR-VQA) we train for up to 5 epochs, capped
at $50{,}000$ update steps per stage. Learning rates and batch sizes
are shared across baselines and ASR.

\vspace{1mm}
\noindent\textbf{LoRA configuration.}
For LLaVA-1.5-7B, we insert rank-$r{=}16$ LoRA adapters into all
self-attention and cross-attention layers of the language model, as
well as into the vision-to-language projection layer. LoRA adapters
are applied to the query and value projections only, following
standard practice in multimodal instruction tuning. For CIT
baselines that already specify LoRA ranks (e.g., Adapt-$\infty$),
we keep their original choices.

\vspace{1mm}
\noindent\textbf{Hardware and runtime.}
All experiments are conducted on 8$\times$A100 GPUs with 80GB memory.
We use mixed-precision (bfloat16) training for all methods. On the
VQA v2 10-task split, one full run of ASR takes approximately
8--10 GPU hours. On CoIN and UCIT, training the full 10-stage stream
takes approximately 20--24 GPU hours. Baselines with large replay
buffers (e.g., Adapt-$\infty$) are slightly more expensive due to
additional data loading and sampling overhead.

\subsection{Metric Definitions}
\label{app:metrics}

We now provide formal definitions of the continual learning metrics
used in Tables~\ref{tab:vqa_results} and~\ref{tab:cit_results}.

\vspace{1mm}
\noindent\textbf{Per-stage, per-task scores.}
Let $M$ denote the number of tasks in a stream (e.g., $M{=}10$ for
VQA v2 and CoIN). After training on task $a\in\{1,\dots,M\}$, we
evaluate the model on each task $b\in\{1,\dots,M\}$ and record a
scalar metric $m_{a,b}$ (e.g., VQA accuracy, normalized dataset
score). By convention, we set $m_{a,b}=0$ if task $b$ has not been
seen yet (i.e., $b>a$).

\vspace{1mm}
\noindent\textbf{Final Average Performance (AP).}
For VQA-based streams with $M$ tasks, the final average performance
is
\begin{equation}
  \mathrm{AP}
  =
  \frac{1}{M}
  \sum_{t=1}^{M} m_{M,t},
  \label{eq:AP-def}
\end{equation}
which averages the performance on all tasks after training on the
last task.

\vspace{1mm}
\noindent\textbf{Average Forgetting (AF).}
Average forgetting measures the average drop from the best historical
performance on each task to its final performance. Following
\citet{vqacl23,quad25}, we define
\begin{equation}
\begin{aligned}
  \mathrm{AF}
  &=
  \frac{1}{M-1}
  \sum_{t=1}^{M-1}
    \bigl(
      m_{t}^{\max}
      - m_{M,t}
    \bigr),
  \\
  m_{t}^{\max}
  &=
  \max_{a \in \{t,\dots,M\}} m_{a,t}.
\end{aligned}
\label{eq:AF-def}
\end{equation}
That is, for each task $t$, we compute the maximum accuracy it ever
achieved over the course of training and subtract its final accuracy;
AF is the average of these drops over all non-final tasks.

\vspace{1mm}
\noindent\textbf{Standard vs.\ novel composition AP/AF.}
In the VQACL skill--concept setting, we compute separate AP and AF
values on standard (seen skill--concept) and novel-composition splits
by restricting the averaging in~\eqref{eq:AP-def} and~\eqref{eq:AF-def}
to tasks corresponding to standard or novel compositions,
respectively. This yields
$(\mathrm{AP}_{\mathrm{std}},\mathrm{AF}_{\mathrm{std}})$ and
$(\mathrm{AP}_{\mathrm{nov}},\mathrm{AF}_{\mathrm{nov}})$ as reported
in Table~\ref{tab:vqa_results}.

\vspace{1mm}
\noindent\textbf{Last and Avg for CIT.}
For CoIN and UCIT, we follow~\citet{SEFE25,Hidellava25} and report:

\begin{itemize}[leftmargin=*,itemsep=1pt,topsep=2pt]
  \item \textbf{Last}: final performance averaged over all tasks,
  \begin{equation}
    \mathrm{Last}
    =
    \frac{1}{M}
    \sum_{b=1}^{M} m_{M,b}.
  \end{equation}

  \item \textbf{Avg}: time-averaged performance across the training
  trajectory,
  \begin{equation}
  \begin{aligned}
    \mathrm{Avg}
    &=
    \frac{1}{M}
    \sum_{a=1}^{M}
      \frac{1}{a}
      \sum_{b=1}^{a} m_{a,b}.
  \end{aligned}
  \label{eq:avg-def}
  \end{equation}
  This rewards methods that perform well both early and late in the
  stream.
\end{itemize}

\vspace{1mm}
\noindent\textbf{Skill-wise forgetting.}
In Sec.~\ref{sec:theory} and Fig.~\ref{fig:drift-forgetting}, we also
use a skill-wise forgetting measure. Let $R_s^{(t)}$ denote the
population risk (or the negative of accuracy) on skill $s$ after
stage $t$, and let
\begin{equation}
  F_s^{(t)}
  =
  \bigl[
    R_s^{(t)}
    - R_s^{(t-1)}
  \bigr]_+
\end{equation}
be the incremental forgetting of skill $s$ at stage $t$. Summing
$F_s^{(t)}$ over $t$ yields a cumulative skill-wise forgetting
measure,
\begin{equation}
  F_s^{\mathrm{cum}}
  =
  \sum_{t=2}^{M} F_s^{(t)}.
\end{equation}
In practice, we approximate $R_s^{(t)}$ by the negative of the
accuracy on validation questions whose skill posterior is dominated
by $s$.

\subsection{Ablation Variants of ASR}
\label{app:ablation-variants}

Recall that Full ASR optimizes, at stage $t$, the joint loss
\begin{equation}
\resizebox{1\linewidth}{!}{%
  $\begin{aligned}
    \mathcal{L}^{(t)}
    &=
    \frac{1}{|\mathcal{B}^{(t)}|}
    \sum_{(I,q,y)\in\mathcal{B}^{(t)}}
    \Bigl[
      \mathcal{L}_{\mathrm{task}}(I,q,y)
      +
      \beta\,\mathcal{L}_{\mathrm{spec}}(I,q)
    \Bigr]
    \\
    &\quad
    +\;
    \gamma\,\mathcal{L}_{\mathrm{geo}},
  \end{aligned}$%
}
\end{equation}
where $\mathcal{L}_{\mathrm{task}}$ is the task-specific supervised
loss (VQA accuracy or instruction tuning loss),
$\mathcal{L}_{\mathrm{spec}}$ is the confidence-weighted spectral
distillation loss combining Mahalanobis and angular terms (Sec.~\ref{subsec:distillation}),
and $\mathcal{L}_{\mathrm{geo}}$ is the geometry regularizer
anchoring cross-modal similarities to a frozen reference encoder
(Sec.~\ref{subsec:objective}).  Prototypes
$\{(\boldsymbol{\mu}_s,\boldsymbol{\Sigma}_s,\hat{\mathbf{d}}_s)\}_{s\in\mathcal{S}}$
are maintained as described in App.~\ref{app:asr-impl}.

We denote the five ablations by \ding{182}--\ding{186}:

\vspace{1mm}
\noindent\textbf{\ding{182} \;w/o Spectrum Distillation.}
This ablation removes the attention-spectrum regularization
entirely, while keeping the geometry regularizer:
\begin{equation}
  \beta = 0,
  \qquad
  \gamma > 0 \;\;\text{(as in Full ASR)}.
\end{equation}
Concretely, we set $\mathcal{L}_{\mathrm{spec}}(I,q)\equiv 0$ for all
samples; no spectral distillation term is added to the loss, and the
confidence weight $w_{\mathrm{spec}}(\cdot)$ is unused.  We still
compute spectral descriptors $\boldsymbol{\phi}(A)$ for analysis,
but they no longer influence training.  The prototype memory
$\mathcal{M}$ is not updated after stage $1$ in this variant, since
it is never used in the loss.  The resulting training objective
degenerates to:
\begin{equation}
  \mathcal{L}^{(t)}_{\mathrm{w/o\,spec}}
  =
  \frac{1}{|\mathcal{B}^{(t)}|}
  \sum_{(I,q,y)\in\mathcal{B}^{(t)}}
    \mathcal{L}_{\mathrm{task}}(I,q,y)
  +
  \gamma\,\mathcal{L}_{\mathrm{geo}},
\end{equation}
which corresponds to a geometry-regularized continual fine-tuning
baseline.

\vspace{1mm}
\noindent\textbf{\ding{183} \;w/o Skill Conditioning.}
This ablation collapses all skill-specific prototypes into a single
global prototype, thereby removing the conditioning on skill
identities while retaining spectral regularization.  We introduce a
single global index $s_{\mathrm{glob}}$ and replace the skill set
$\mathcal{S}$ by $\{s_{\mathrm{glob}}\}$ when computing prototypes
and spectral losses.  In particular:
\begin{itemize}[leftmargin=*,itemsep=1pt,topsep=2pt]
  \item The skill parser $g_{\boldsymbol{\psi}}$ is ignored; we
  treat all samples as belonging to the same ``skill'' and aggregate
  their descriptors into a global set
  $\mathcal{F}_{s_{\mathrm{glob}}}^{(t)}$ at stage $t$.

  \item We estimate a single global mean
  $\widehat{\boldsymbol{\mu}}_{s_{\mathrm{glob}}}^{(t)}$,
  covariance
  $\widehat{\boldsymbol{\Sigma}}_{s_{\mathrm{glob}}}^{(t)}$ and
  angular spectrum
  $\widehat{\mathbf{d}}_{s_{\mathrm{glob}}}^{(t)}$ using all
  examples in $\mathcal{D}^{(t)}$, and update a single global
  prototype
  $(\boldsymbol{\mu}_{s_{\mathrm{glob}}},\boldsymbol{\Sigma}_{s_{\mathrm{glob}}},
  \hat{\mathbf{d}}_{s_{\mathrm{glob}}})$ via EMA.

  \item During training, the spectral loss is computed without
  skill weights:
  \begin{equation}
  \begin{aligned}
    \mathcal{L}_{\mathrm{spec}}^{\mathrm{glob}}(I,q)
    &=
    w_{\mathrm{spec}}\bigl(\boldsymbol{\phi}(A)\bigr)
    \Bigl(
      D_{\mathrm{mah}}^2\bigl(\boldsymbol{\phi}(A)\,\Vert\,s_{\mathrm{glob}}\bigr)
    \\
    &\quad\;
      +
      \lambda_{\mathrm{ang}}
      \,\mathrm{KL}_{\mathrm{sym}}
        \bigl(\mathbf{d}(\theta)\,
              \Vert\,\hat{\mathbf{d}}_{s_{\mathrm{glob}}}\bigr)
    \Bigr),
  \end{aligned}
  \end{equation}
  where $D_{\mathrm{mah}}^2(\cdot\Vert s_{\mathrm{glob}})$ and
  $\hat{\mathbf{d}}_{s_{\mathrm{glob}}}$ are defined as in
  Sec.~\ref{subsec:distillation} but with a single prototype.
\end{itemize}
The rest of the pipeline, including $\beta$, $\gamma$ and
$w_{\mathrm{spec}}(\cdot)$, is unchanged.  This variant tests whether
a single global attention prior is sufficient, versus the full
skill-conditioned design of ASR.

\vspace{1mm}
\noindent\textbf{\ding{184} \;w/o Angular Term.}
This ablation removes the angular component of spectral
regularization while keeping the radial / Mahalanobis matching
intact.  Specifically, we set
\begin{equation}
  \lambda_{\mathrm{ang}} = 0,
\end{equation}
so that the spectral loss becomes
\begin{equation}
\begin{aligned}
  \mathcal{L}_{\mathrm{spec}}^{\mathrm{rad}}(I,q)
  &=
  w_{\mathrm{spec}}\bigl(\boldsymbol{\phi}(A)\bigr)
  \sum_{s\in\mathcal{S}}
    \pi_s(q)\,
    D_{\mathrm{mah}}^2\bigl(\boldsymbol{\phi}(A)\,\Vert\,s\bigr),
\end{aligned}
\end{equation}
with the same confidence weights $w_{\mathrm{spec}}(\cdot)$ and skill
posterior $\boldsymbol{\pi}(q)$ as in Full ASR.  We still maintain
and update $\hat{\mathbf{d}}_s$ for analysis, but it does not enter
the training loss.  This variant isolates the contribution of
directional information in the frequency domain.

\vspace{1mm}
\noindent\textbf{\ding{185} \;w/o Confidence Weighting.}
In Full ASR, the spectral distillation loss is modulated by a
confidence weight
\begin{equation}
  w_{\mathrm{spec}}\bigl(\boldsymbol{\phi}(A)\bigr)
  =
  \exp\!\Bigl(
    - d_{\min}(\boldsymbol{\phi}(A)) / \tau_{\mathrm{mah}}
  \Bigr),
\end{equation}
where $d_{\min}(\boldsymbol{\phi})$ is the minimum Mahalanobis
distance to any skill prototype (Eq.~(19) in the main text).  This
reduces the regularization strength for spectra that are far from all
known prototypes.

In the w/o Confidence Weighting ablation, we disable this modulation
by setting
\begin{equation}
  w_{\mathrm{spec}}\bigl(\boldsymbol{\phi}(A)\bigr)
  \equiv 1
  \quad
  \forall\,\boldsymbol{\phi}(A),
\end{equation}
so that every sample contributes equally to the spectral loss,
independent of how close its descriptor is to the current prototype
memory.  The resulting spectral loss is
\begin{equation}
\resizebox{1\linewidth}{!}{%
  $\begin{aligned}
    \mathcal{L}_{\mathrm{spec}}^{\mathrm{unif}}(I,q)
    &=
    \sum_{s\in\mathcal{S}}
      \pi_s(q)\,
      D_{\mathrm{mah}}^2\bigl(\boldsymbol{\phi}(A)\,\Vert\,s\bigr)
    \\
    &\quad
    +
    \lambda_{\mathrm{ang}}
    \sum_{s\in\mathcal{S}}
      \pi_s(q)\,
      \mathrm{KL}_{\mathrm{sym}}
        \bigl(\mathbf{d}(\theta)\,\Vert\,\hat{\mathbf{d}}_s\bigr).
  \end{aligned}$%
}
\end{equation}
All other components, including prototype updates and $\tau_{\mathrm{mah}}$,
remain unchanged.  This variant tests whether adaptive down-weighting
of out-of-prototype spectra is necessary for stability.

\vspace{1mm}
\noindent\textbf{\ding{186} \;w/o Geometry Regularizer.}
Finally, this ablation removes the cross-modal geometry regularizer
and keeps only the task loss and spectral regularization:
\begin{equation}
  \gamma = 0,
  \qquad
  \beta > 0 \;\;\text{(as in Full ASR)}.
\end{equation}
The geometry loss $\mathcal{L}_{\mathrm{geo}}$, which matches
image--text similarity distributions between the current model and
the frozen reference backbone, is omitted from the objective.  The
training loss thus reduces to
\begin{equation}
\begin{aligned}
  \mathcal{L}^{(t)}_{\mathrm{w/o\,geo}}
  &=
  \frac{1}{|\mathcal{B}^{(t)}|}
  \sum_{(I,q,y)\in\mathcal{B}^{(t)}}
  \Bigl[
    \mathcal{L}_{\mathrm{task}}(I,q,y)
  \\
  &\quad +
  \beta\,\mathcal{L}_{\mathrm{spec}}(I,q)
  \Bigr],
\end{aligned}
\end{equation}
with the same spectral loss $\mathcal{L}_{\mathrm{spec}}$ and
prototype updates as in Full ASR.  This variant probes how much of
ASR's stability comes from the spectral constraints alone, as opposed
to a combination of spectral and geometric anchors.

\section{Additional Experimental Results}

\subsection{Backbone Robustness}
\label{subsec:backbone}

\begin{table}[ht]
\centering
\small
\caption{\textbf{Same-family backbone sizes.}
Comparison of Vanilla vs.\ ASR on LLaVA-1.5-3B/7B/13B for VQA v2
(10-task split) and CoIN/UCIT continual instruction tuning.
All models share the same training schedules.}
\label{tab:backbone-size}
\vspace{-2mm}
\resizebox{\linewidth}{!}{%
\setlength{\tabcolsep}{3.2pt}
\begin{tabular}{llcccccc}
\toprule
Backbone & Method &
\multicolumn{2}{c}{VQA v2} &
\multicolumn{2}{c}{CoIN} &
\multicolumn{2}{c}{UCIT} \\
\cmidrule(lr){3-4} \cmidrule(lr){5-6} \cmidrule(lr){7-8}
& &
AP$\uparrow$ & AF$\downarrow$ &
Last$\uparrow$ & Avg$\uparrow$ &
Last$\uparrow$ & Avg$\uparrow$ \\
\midrule
\multirow{2}{*}{LLaVA-1.5-3B}
& Vanilla
& 41.0 & 21.0 & 46.8 & 44.7 & 39.0 & 37.5 \\
& ASR
& 48.0 &  \textbf{5.0} & 58.3 & 56.9 & 50.1 & 48.8 \\
\midrule
\multirow{2}{*}{LLaVA-1.5-7B}
& Vanilla
& 44.1 & 18.3 & 49.3 & 47.2 & 41.5 & 39.8 \\
& ASR (ours)
& \textbf{52.0} &  \textbf{2.4} & \textbf{61.4} & \textbf{60.0} & \textbf{53.1} & \textbf{51.5} \\
\midrule
\multirow{2}{*}{LLaVA-1.5-13B}
& Vanilla
& 45.8 & 16.5 & 51.2 & 49.0 & 43.7 & 41.9 \\
& ASR
& \textbf{53.0} &  \textbf{2.1} & \textbf{62.8} & \textbf{61.1} & \textbf{54.4} & \textbf{52.9} \\
\bottomrule
\end{tabular}%
}
\end{table}

Table~\ref{tab:backbone-size} reports results for three LLaVA-1.5
variants with different LLM capacities (3B, 7B, 13B). Across all sizes, adding ASR substantially improves both VQA and CIT metrics: on LLaVA-1.5-7B,
ASR lifts VQA v2 AP from $44.1$ to $52.0$ and reduces AF from $18.3$
to $2.4$, while on CoIN it improves Last from $49.3$ to $61.4$ and
Avg from $47.2$ to $60.0$, and on UCIT it raises Last/Avg from
$41.5/39.8$ to $53.1/51.5$. The absolute performance naturally
increases with model size (3B$<$7B$<$13B), but the relative
gains from ASR remain consistent. Interestingly, the smaller 3B
model, which suffers the most catastrophic forgetting under Vanilla
training, benefits the most in relative terms (e.g., VQA v2 AF from
$21.0$ to $5.0$), supporting the view that skill-conditioned spectral
stabilization is especially valuable when capacity is limited.

\subsection{Stage-wise Dynamics of Plasticity and Stability}
\label{subsec:stage-dynamics}

Beyond final AP/AF and Last/Avg, we analyze how performance evolves
over the continual stream to understand whether ASR provides uniform
stability or merely improves the last few stages.

\vspace{1mm}
\noindent\textbf{VQA v2 10-task stream.}
For the VQA v2 question-type incremental setting, we consider three
representative methods: Vanilla fine-tuning, CL-MoE, and ASR.
At each stage $a\in\{1,\dots,10\}$ we measure: (i) \emph{Plasticity on the current task}:
  $m_{a,a}$, i.e., performance on the current task $a$ immediately
  after training on it. (ii) \emph{Stability on seen tasks}:
  $\bar m_a = \frac{1}{a}\sum_{b=1}^{a} m_{a,b}$, i.e., average
  performance over all tasks seen so far.

Fig.~\ref{fig:stage-vqa} plots $m_{a,a}$ and $\bar m_a$ across
stages. Vanilla exhibits strong plasticity spikes on the current
task but a steep decline in $\bar m_a$ as more tasks are added,
indicating severe interference with previous skills. CL-MoE partly
ameliorates this behavior, but still shows noticeable drops after
each new task. In contrast, ASR maintains plasticity comparable to
or slightly higher than CL-MoE on the current task while yielding a
much flatter stability curve: $\bar m_a$ decays only mildly and
stabilizes after mid-stream, consistent with our goal of preserving
skill-conditioned attention spectra rather than over-constraining
task learning.

\begin{figure}[ht]
\centering
\includegraphics[width=1\linewidth]{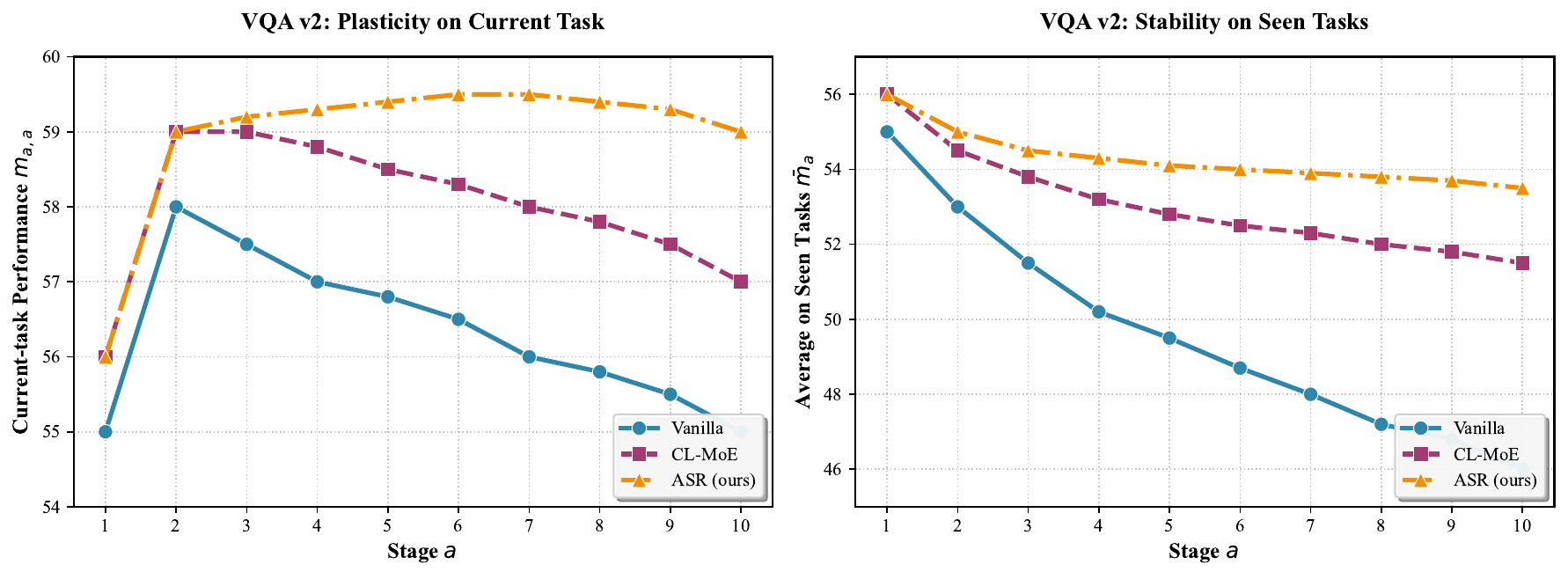}
\vspace{-2mm}
\caption{\textbf{Stage-wise plasticity and stability on VQA v2
(10-task stream).}
Left: performance on the current task $m_{a,a}$ (plasticity).
Right: average performance on all seen tasks
$\bar m_a = \frac{1}{a}\sum_{b\le a} m_{a,b}$ (stability).}
\label{fig:stage-vqa}
\end{figure}

\vspace{1mm}
\noindent\textbf{CoIN and UCIT CIT streams.}
For continual multimodal instruction tuning, we track analogous
stage-wise dynamics on CoIN and UCIT for three methods: LoRA-FT,
Adapt-$\infty$, and ASR. At each stage $a$ we define
\begin{itemize}
  \item \emph{Stage-wise Last} $L_a$:
  average performance on all tasks seen up to stage $a$ after
  training on stage $a$.

  \item \emph{Running Avg} $A_a$:
  time-averaged performance up to stage $a$,
  $A_a = \frac{1}{a}\sum_{k=1}^{a} L_k$.
\end{itemize}

Fig.~\ref{fig:stage-cit} shows $L_a$ and $A_a$ for CoIN and UCIT.
On both benchmarks, LoRA-FT exhibits a clear downward trend in $L_a$
as the stream progresses, reflecting cumulative forgetting.
Adapt-$\infty$ raises the overall level of $L_a$ but still shows
visible dips when switching to new tasks. ASR not only shifts the
curves upward (higher $L_a$ and $A_a$ at almost every stage) but also
smooths out the drops: performance on seen datasets decreases much
less when new datasets arrive, and the running average $A_a$ grows
steadily. These dynamics mirror the static Last/Avg gains in
Table~\ref{tab:cit_results} and provide a temporal view of how ASR
improves the plasticity–stability trade-off over the entire CIT
trajectory rather than only at the end.

\begin{figure}[ht]
\centering
\includegraphics[width=1\linewidth]{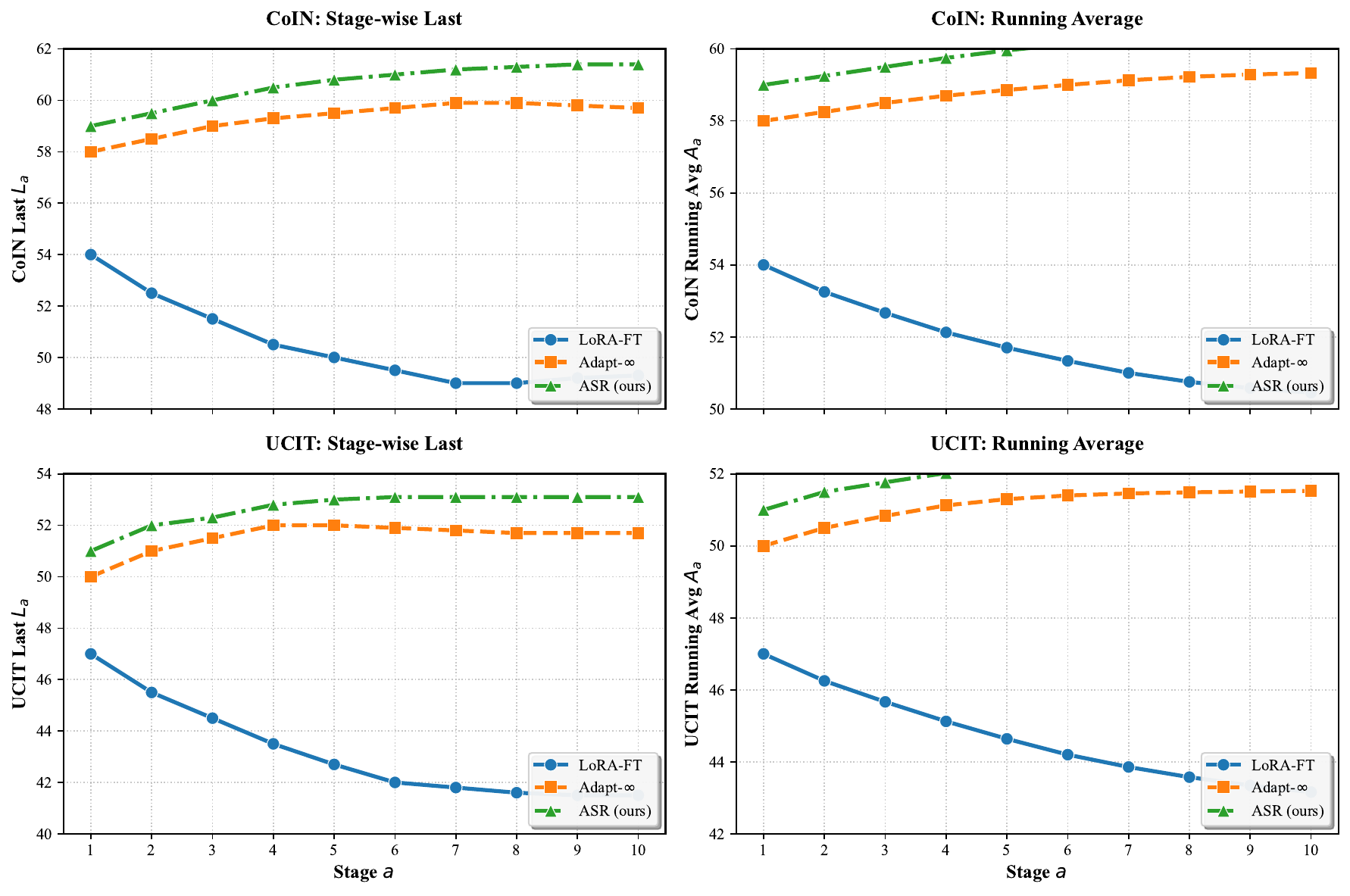}
\vspace{-2mm}
\caption{\textbf{Stage-wise dynamics on CoIN and UCIT.}}
\label{fig:stage-cit}
\end{figure}

\subsection{Spectral Drift vs.\ Embedding Drift vs.\ Forgetting}
\label{subsec:spec-vs-emb-drift}

Complementing the analysis in Fig.~\ref{fig:drift-forgetting}, which
relates spectral prototype drift to skill-wise forgetting, we now
compare spectral drift and embedding drift side by side to test whether attention spectra are indeed more predictive of forgetting than conventional feature drift.

We focus on four representative skills
$s\in\{\texttt{count},\texttt{read},\texttt{locate},\texttt{relation}\}$
in the VQA v2 10-task stream. For each stage $t>1$ and skill $s$, we
compute three quantities:
\begin{itemize}[leftmargin=*,itemsep=1pt,topsep=2pt]
  \item \emph{Spectral drift}:
  \begin{equation}
    \Delta_s^{\mathrm{spec}}(t)
    =
    \bigl\|
      \boldsymbol{\mu}_s^{(t)}
      -
      \boldsymbol{\mu}_s^{(t-1)}
    \bigr\|_2,
  \end{equation}
  where $\boldsymbol{\mu}_s^{(t)}$ is the skill-$s$ spectral
  prototype at stage $t$ (Sec.~\ref{subsec:prototypes}).

  \item \emph{Embedding drift}:
  \begin{equation}
    \Delta_s^{\mathrm{emb}}(t)
    =
    \bigl\|
      \bar{\boldsymbol{z}}_s^{(t)}
      -
      \bar{\boldsymbol{z}}_s^{(t-1)}
    \bigr\|_2,
  \end{equation}
  where $\bar{\boldsymbol{z}}_s^{(t)}$ is the mean pooled
  image--text embedding of skill-$s$ examples at stage $t$,
  computed from the multimodal representation (before the decoding
  head) of the current model.

  \item \emph{Skill-wise forgetting}:
  \begin{equation}
    F_s^{(t)}
    =
    \bigl[
      R_s(\boldsymbol{\theta}^{(t)})
      -
      R_s(\boldsymbol{\theta}^{(t-1)})
    \bigr]_+,
  \end{equation}
  i.e., the increase in skill-$s$ population risk from stage
  $t-1$ to $t$ (Sec.~\ref{sec:theory}). In practice we approximate
  $R_s(\cdot)$ via the negative of validation accuracy on questions
  dominated by skill $s$.
\end{itemize}

We aggregate all skill--stage pairs $(s,t)$ into a set of points and
analyze the correlation between drift and forgetting for ASR.
Fig.~\ref{fig:spec-emb-drift} shows two scatter plots: spectral
drift $\Delta_s^{\mathrm{spec}}(t)$ vs.\ forgetting $F_s^{(t)}$ and
embedding drift $\Delta_s^{\mathrm{emb}}(t)$ vs.\ the same
$F_s^{(t)}$. Each point corresponds to one skill--stage pair.
We also report Pearson and Spearman correlation coefficients in the
legend. Spectral drift exhibits a strong, approximately monotone
relationship with forgetting (e.g., Pearson $\rho\approx 0.8$,
Spearman $\rho\approx 0.75$), whereas embedding drift shows a much
weaker and noisier association (e.g., $\rho\approx 0.4$,
Spearman $\rho\approx 0.35$). This suggests that changes in
skill-conditioned attention spectra are substantially more predictive
of skill-wise forgetting than changes in pooled embeddings, providing
empirical support for our theoretical view that ASR stabilizes
continual learning by directly constraining the structure of
cross-modal attention.

\begin{figure}[ht]
\centering
\includegraphics[width=1\linewidth]{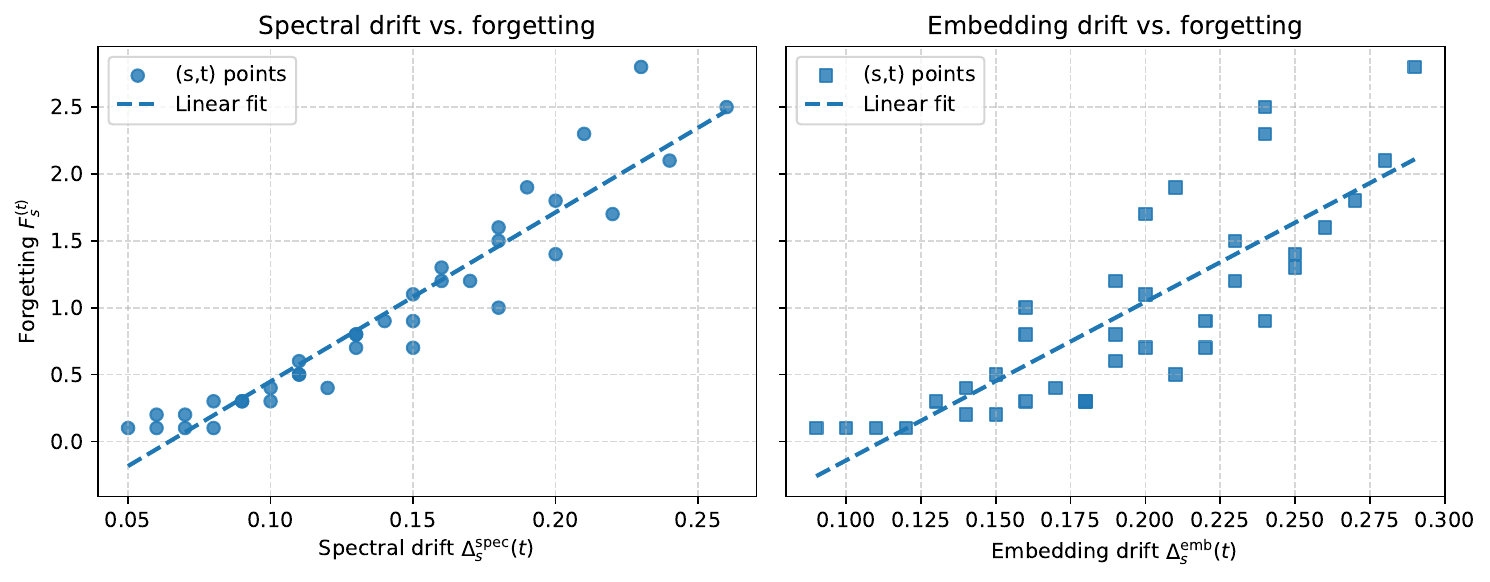}
\vspace{-2mm}
\caption{\textbf{Spectral drift vs.\ embedding drift vs.\ forgetting}
on VQA v2 for four skills
$s\in\{\texttt{count},\texttt{read},\texttt{locate},\texttt{relation}\}$.
Each point corresponds to one skill--stage pair $(s,t)$ under ASR.
\textbf{Left:} spectral drift
$\Delta_s^{\mathrm{spec}}(t)$ vs.\ skill-wise forgetting
$F_s^{(t)}$; \textbf{Right:} embedding drift
$\Delta_s^{\mathrm{emb}}(t)$ vs.\ $F_s^{(t)}$.
Spectral drift shows a much stronger correlation with forgetting
than embedding drift, indicating that changes in attention spectra
are a more faithful indicator of catastrophic forgetting than
changes in pooled representations.}
\label{fig:spec-emb-drift}
\end{figure}

\subsection{Cost--Performance Trade-off: Memory and Time vs.\ Accuracy}
\label{subsec:cost-performance}

ASR is designed to be replay-free and lightweight, using only compact
skill-conditioned spectral prototypes instead of large replay buffers
or additional experts. To quantify this claim, we compare several
baselines and ASR in terms of (i) extra memory footprint and (ii)
training cost, against their achieved performance.

\vspace{1mm}
\noindent\textbf{Metrics.}
For each method we estimate:
\begin{itemize}
  \item \emph{Extra memory (MB)} beyond the frozen backbone:
  this includes adapter parameters, routing/expert parameters, and
  any persistent replay buffer or prototype memory. We approximate
  memory as $\text{param\_count} \times 2$ bytes (bfloat16) plus
  stored data size.

  \item \emph{Per-step training time (ms)} for a representative
  stage with a fixed batch size, measured as the average wall-clock
  time per optimization step over that stage. This reflects
  computational overhead (e.g., MoE routing, replay sampling,
  spectral computation).

  \item \emph{Performance}: for VQA v2 we use final AP on the
  10-task split, for CoIN we use final Last as in
  Table~\ref{tab:cit_results}.
\end{itemize}

We consider (i) VQA v2 methods \{Vanilla, EWC, ER, QUAD, CL-MoE, ASR\}
and (ii) CoIN methods \{LoRA-FT, HiDe-LLaVA, SEFE, BranchLoRA,
D-MoLE, Adapt-$\infty$, ASR\}. Replay-heavy methods (ER and
Adapt-$\infty$) incur large buffer memory; MoE-based methods
(CL-MoE, D-MoLE) increase parameter counts and runtime; ASR adds only
a small spectral prototype memory and modest FFT overhead.

\begin{figure}[ht]
\centering
\includegraphics[width=1\linewidth]{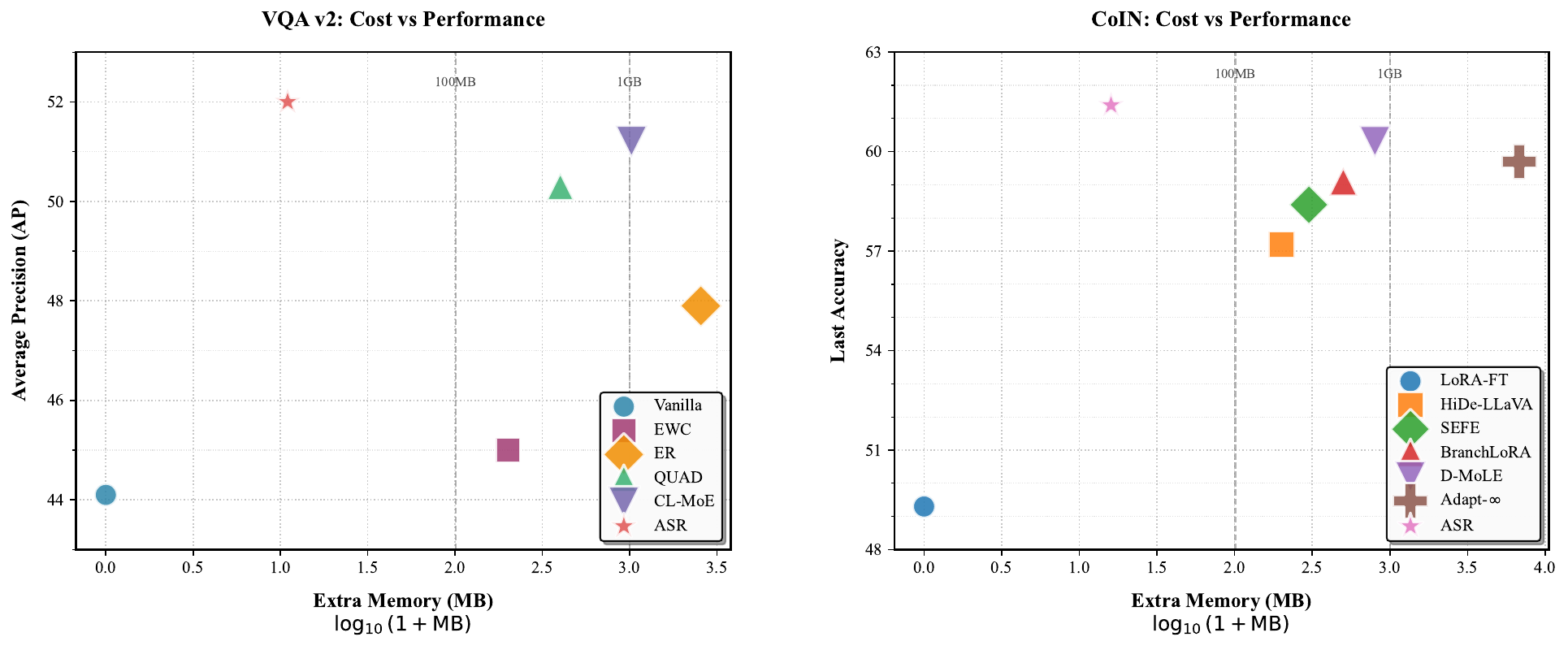}
\vspace{-2mm}
\caption{\textbf{Cost--performance trade-off.}
\textbf{Left:} VQA v2 10-task AP vs.\ extra memory (MB) for
Vanilla, EWC, ER, QUAD, CL-MoE, and ASR.
\textbf{Right:} CoIN Last vs.\ extra memory (MB) for LoRA-FT,
HiDe-LLaVA, SEFE, BranchLoRA, D-MoLE, Adapt-$\infty$, and ASR.
The $x$-axis is $\log_{10}(1+\text{memory})$ (MB); marker area is
proportional to per-step training time.
Replay-based methods (ER, Adapt-$\infty$) lie in the high-memory
region, and MoE-style methods (CL-MoE, D-MoLE) incur noticeable
runtime overhead.
ASR achieves the best or near-best performance while residing near
the low-memory, moderate-time corner, confirming that skill-conditioned
attention-spectrum regularization offers a favorable cost--performance
trade-off.}
\label{fig:cost-vs-perf}
\end{figure}

As shown in Fig.~\ref{fig:cost-vs-perf},
on VQA v2 (left), ER and CL-MoE improve AP over Vanilla but at the
cost of large replay buffers (ER, $\sim$2.5\,GB) or additional
experts ($\sim$1\,GB), and increased runtime. ASR attains the highest
AP (52.0) with only $\sim$10\,MB of extra memory for skill
prototypes and a modest per-step overhead, occupying the most
favorable top-left region (high AP, low cost). On CoIN (right),
Adapt-$\infty$ achieves strong Last but requires a multi-GB replay
buffer; MoE-based D-MoLE also incurs substantial extra parameters
and runtime. ASR reaches the highest Last (61.4) with a prototype
memory of only $\sim$15\,MB and moderate per-step time, demonstrating
that attention-spectrum regularization can match or surpass
replay-heavy methods while being significantly more memory-efficient.

\subsection{Task Order Robustness}
\label{subsec:task-order}

A natural concern in continual learning is whether a method’s gains
are specific to a particular task ordering. To test this, we evaluate
CL-MoE and ASR on VQA v2 under multiple permutations of the 10
question-type tasks.

Let the canonical order be the question-type sequence used in
Sec.~\ref{subsec:exp_setup}. We then sample three additional random
permutations of the 10 tasks, yielding $4$ orders in total:
\[
\mathcal{O}_1 \text{ (canonical)},\;
\mathcal{O}_2,\;
\mathcal{O}_3,\;
\mathcal{O}_4.
\]
For each method $m\in\{\text{CL-MoE}, \text{ASR}\}$ and order
$\mathcal{O}_k$, we run the full 10-stage stream and record the final
AP and AF on VQA v2. This yields per-order values
$\mathrm{AP}_k^{(m)}$ and $\mathrm{AF}_k^{(m)}$; we then compute the
mean and standard deviation across orders:
\begin{equation}
\resizebox{1\linewidth}{!}{$
\begin{aligned}
  \overline{\mathrm{AP}}^{(m)}
  &= \frac{1}{4}\sum_{k=1}^4 \mathrm{AP}_k^{(m)},
  &
  \sigma_{\mathrm{AP}}^{(m)}
  &= \operatorname{Std}\bigl(\{\mathrm{AP}_k^{(m)}\}_{k=1}^4\bigr), \\
  \overline{\mathrm{AF}}^{(m)}
  &= \frac{1}{4}\sum_{k=1}^4 \mathrm{AF}_k^{(m)},
  &
  \sigma_{\mathrm{AF}}^{(m)}
  &= \operatorname{Std}\bigl(\{\mathrm{AF}_k^{(m)}\}_{k=1}^4\bigr).
\end{aligned}
$}
\end{equation}
For the canonical order $\mathcal{O}_1$, the AP/AF values match
Table~\ref{tab:vqa_results} .

\begin{figure}[ht]
\centering
\includegraphics[width=1\linewidth]{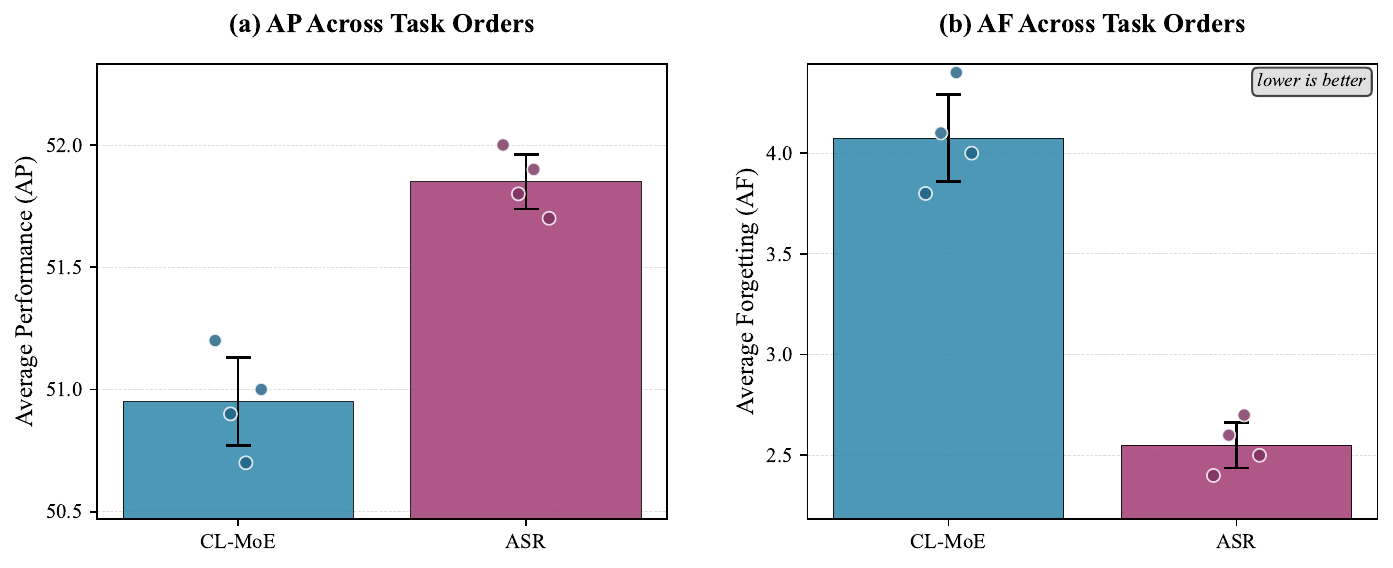}
\vspace{-2mm}
\caption{\textbf{Task order robustness on VQA v2 (10-task stream).}}
\label{fig:task-order}
\end{figure}

Figure~\ref{fig:task-order} summarizes AP and AF across orders using
bar plots with error bars. For CL-MoE, AP fluctuates more across
orders and AF varies around
$4.0$ with non-trivial spread,
indicating that its stability is somewhat sensitive to the specific
task sequence. ASR not only yields higher mean AP and lower mean AF,
but also exhibits slightly smaller variation across orders
($\sigma_{\mathrm{AP}}\approx 0.12$,
$\sigma_{\mathrm{AF}}\approx 0.13$), suggesting that the
skill-conditioned spectral constraint makes the method less dependent
on a favorable curriculum. In all four permutations, ASR dominates
CL-MoE in both AP and AF, providing evidence that our gains are not
an artifact of a particular task ordering.

\end{document}